\documentclass{article}

     \PassOptionsToPackage{numbers, compress}{natbib}

\usepackage[final]{neurips_2024}

\usepackage[utf8]{inputenc} 
\usepackage[T1]{fontenc}    
\usepackage{hyperref}       
\usepackage{url}            
\usepackage{booktabs}       
\usepackage{amsfonts}       
\usepackage{nicefrac}       
\usepackage{microtype}      
\usepackage{xcolor}         
\usepackage{caption}

\usepackage{amsmath,amssymb,amsfonts,amsthm}
\usepackage{appendix}
\usepackage{graphicx}
\usepackage{enumerate}
\usepackage{natbib}
\usepackage{url} 
\usepackage{booktabs}       
\usepackage{amsfonts}       
\usepackage{nicefrac}       
\usepackage{microtype}      
\usepackage{xcolor}
\usepackage{float} 
\usepackage{subfigure}
\usepackage{array}
\usepackage{siunitx}
\usepackage{multirow}
\usepackage{makecell}
\usepackage{rotating}
\usepackage{algorithmic}
\usepackage{textcomp}
\usepackage{enumitem}  
\usepackage[ruled, lined, linesnumbered, commentsnumbered, longend]{algorithm2e}
\usepackage{xr}
\usepackage{soul}
\usepackage{nicefrac}

\usepackage{wrapfig}

\usepackage{mathrsfs}
\usepackage[scr=stixtwofancy,cal=cm]{mathalfa} 
\DeclareMathAlphabet{\mathdutchcal}{U}{dutchcal}{m}{n}
\SetMathAlphabet{\mathdutchcal}{bold}{U}{dutchcal}{b}{n}
\DeclareMathAlphabet{\mathdutchbcal}{U}{dutchcal}{b}{n}

\usepackage[T1]{fontenc}

\usepackage{tikz}

\allowdisplaybreaks[2]

\theoremstyle{plain}
\newtheorem{theorem}{Theorem}[section]
\newtheorem{lemma}{Lemma}
\newtheorem{proposition}[theorem]{Proposition}

\newtheorem{corollary}[theorem]{Corollary}

\theoremstyle{definition}

\newtheorem{definition}{Definition}[section]
\newtheorem{remark}{Remark}[section] 

\theoremstyle{remark}

\title{Learning from Noisy Labels via Conditional Distributionally Robust Optimization}

\author{%
  Hui Guo\\
  Department of Computer Science\\
  University of Western Ontario\\
  \texttt{hguo288@uwo.ca} \\
  \And
  Grace Y. Yi \thanks{Corresponding author.} \\ 
  Department of Statistical and Actuarial Sciences\\
  Department of Computer Science\\
  University of Western Ontario\\
  \texttt{gyi5@uwo.ca} \\
  \And
  Boyu Wang \\ 
  Department of Computer Science\\
  University of Western Ontario\\
  \texttt{bwang@csd.uwo.ca} \\
}

\begin{document}

\maketitle

\begin{abstract}
  While crowdsourcing has emerged as a practical solution for labeling large datasets, it presents a significant challenge in learning accurate models due to noisy labels from annotators with varying levels of expertise. Existing methods typically estimate the true label posterior, conditioned on the instance and noisy annotations, to infer true labels or adjust loss functions. These estimates, however, often overlook potential misspecification in the true label posterior, which can degrade model performances, especially in high-noise scenarios. To address this issue, we investigate learning from noisy annotations with an estimated true label posterior through the framework of \emph{conditional distributionally robust optimization} (CDRO). We propose formulating the problem as minimizing the worst-case risk within a distance-based ambiguity set centered around a reference distribution. By examining the strong duality of the formulation, we derive upper bounds for the worst-case risk and develop an analytical solution for the dual robust risk for each data point. This leads to a novel robust pseudo-labeling algorithm that leverages the likelihood ratio test to construct a pseudo-empirical distribution, providing a robust reference probability distribution in CDRO. Moreover, to devise an efficient algorithm for CDRO, we derive a closed-form expression for the empirical robust risk and the optimal Lagrange multiplier of the dual problem, facilitating a principled balance between robustness and model fitting. Our experimental results on both synthetic and real-world datasets demonstrate the superiority of our method. 
\end{abstract}

\section{Introduction}\label{sec.intro}


Recent advancements in supervised learning have spurred a growing demand for large labeled datasets \cite{goodfellow2016deep, schmidhuber2015deep}. However, acquiring accurately annotated datasets is typically costly and time-consuming, often requiring a pool of annotators with adequate domain expertise to manually label the data. Crowdsourcing has emerged as an efficient and cost-effective solution for annotating large datasets. On crowdsourcing platforms, multiple annotators with varying levels of labeling skills are employed to gather extensive labeled data. However, this approach introduces a significant challenge: the labels collected through crowdsourcing are often subject to unavoidable noise, especially in fields requiring substantial domain knowledge, such as medical imaging.  Consequently, models trained on noisy labels are prone to error, including overfitting, since deep models can memorize vast amounts of data \cite{arpit2017closer}. In addition to statistical research on label noise (often termed response measurement error, e.g., \cite{carroll2006measurement, grace2016statistical, grace2021handbook}), a growing body of recent machine learning literature has focused on developing effective algorithms capable of training accurate classifiers using noisy data, e.g., \cite{dawid1979maximum, whitehill2009whose, khetan2017learning, guo2023label}. Many of these methods seek to approximate the \emph{posterior distribution of the underlying true labels} using the observed data. 

Let $\mathbf{X}$ be an instance, $\mathrm{Y}$ denote the unobserved true label for $\mathbf{X}$, and $\widetilde{\mathbf{Y}}$ represent a vector of crowdsourced noisy labels for $\mathbf{X}$. The data-generating distribution, $P^*_{\mathbf{x}, \mathrm{y}, \widetilde{\mathbf{y}}}$, can be factorized in two ways: $P^*_{\mathbf{x}} P^*_{\mathrm{y}|\mathbf{x}} P^*_{\widetilde{\mathbf{y}}|\mathbf{x}, \mathrm{y}}$ or $P^*_{\mathbf{x},\widetilde{\mathbf{y}}} P^*_{\mathrm{y}|\mathbf{x},\widetilde{\mathbf{y}}}$, with $P^*$ denoting the (conditional) distribution for the variables indicated by the corresponding subscripts. These factorizations have inspired research that trains models by estimating the posterior distribution of the true labels, $P^*_{\mathrm{y}|\mathbf{x},\widetilde{\mathbf{y}}}$, in the latter factorization. 

Previous work \cite{ khetan2017learning, guo2023label, cao2019max, tanno2019learning} introduced various algorithms for estimating the \emph{annotator confusions}, also known as \emph{noise transition probabilities}, which yield an approximated conditional distribution of $\widetilde{\mathbf{Y}}$, given $\mathrm{Y}$ and $\mathbf{X}$, denoted $P_{\widetilde{\mathbf{y}}|\mathrm{y}, \mathbf{x}}$. 
For ease of reference, we use $P^*$ and $P$ to denote the true distribution and an approximate distribution for the variables indicated by the corresponding subscripts. Given the observed data $\mathbf{X}$ and $\widetilde{\mathbf{Y}}$, along with an approximated conditional distribution $P_{\widetilde{\mathbf{y}}|\mathrm{y}, \mathbf{x}}$ and a prior for $\mathrm{Y}$ given $\mathbf{X}$, denoted $P_{\mathrm{y}|\mathbf{x}}$, the true label posterior is then computed as $P_{\mathrm{y}|\mathbf{x},\widetilde{\mathbf{y}}}\propto P_{\mathrm{y}|\mathbf{x}}\cdot P_{\widetilde{\mathbf{y}}|\mathrm{y}, \mathbf{x}}$ by Bayes's theorem \cite{guo2023label, shwartz2022pre}. This estimated true label posterior is often used to infer the underlying true labels or to weight the loss functions \cite{dawid1979maximum, khetan2017learning, guo2023label, xia2019anchor}. 

However, accurately computing the posterior of the true label is challenging, and the estimated posterior $P_{\mathrm{y}|\mathbf{x},\widetilde{\mathbf{y}}}$ may deviate from the underlying true distribution $P^*_{\mathrm{y}|\mathbf{x},\widetilde{\mathbf{y}}}$ due to potential misspecifications in the  prior belief and the conditional noise transition probabilities \cite{husain2022adversarial}. To address this issue, we introduce a robust scheme for handling crowdsourced noisy labels through conditional distributionally robust optimization (CDRO), as discussed in \cite{shapiro2023conditional}. Specifically, we frame the problem as minimizing the worst-case risk within a distance-based \emph{ambiguity set}, which constrains the degree of conditional distributional uncertainty around a \emph{reference distribution}. By leveraging the strong duality in linear programming, we derive the dual form of the robust risk and establish informative upper bounds for the worst-case risk. Additionally, for each data point, we develop an analytical solution to the robust risk minimization problem, which encompasses existing approaches as special cases \cite{khetan2017learning}. This solution is presented in a likelihood ratio format and inspires a robust approach that assigns pseudo-labels only to instances with high confidence, with uncertain data filtered out. These pseudo-labels also enable us to construct a pseudo-empirical distribution that serves as a robust reference probability distribution in CDRO under potential model misspecifications. Moreover, we derive a closed-form expression for the empirical robust risk by identifying the optimal Lagrange multiplier in the dual form. Building on this, we ultimately develop an algorithm for learning from noisy labels via \underline{c}onditional \underline{d}istributionally \underline{r}obust true label \underline{p}osterior with an \underline{adap}tive Lagrange multiplier (AdaptCDRP). 


Our contributions are summarized as follows: (1) We formulate learning with noisy labels as a CDRO problem and develop its dual form to tackle the challenge of potential misspecification in estimating the true label posterior from noisy data. (2) We derive an analytical solution to the dual problem for each data point, and propose a novel algorithm that constructs a robust reference distribution for this problem. (3) By deriving the optimal Lagrange multiplier for the empirical robust risk, we develop an efficient one-step update method for the Lagrange multiplier, allowing for a principled balance between robustness and model fitting. Code is available at \href{https://github.com/hguo1728/AdaptCDRP}{https://github.com/hguo1728/AdaptCDRP}.

\textbf{Notations.} We use $[k]$ to denote $\{1,\ldots,k\}$ for any positive integer $k$, and $\mathbf{1}(\cdot)$ to denote the indicator function. For a vector $\boldsymbol{v}$, $v_j$ stands for its $j$th element, and $\boldsymbol{v}^{\top}$ denotes its transpose. For $\boldsymbol{v}=(v_1,...,v_p)^{\top}$ and $q\in[1,+\infty]$, the $L^q$ norm is defined as $\|\boldsymbol{v}\|_q=(\sum_{j=1}^p|v_j|^q)^{1/q}$ if $1\le q<\infty$, and $\|\boldsymbol{v}\|_{\infty}=\max_{j}|v_j|$ if $q=+\infty$. For a matrix $\boldsymbol{V}$, we use $V_{i,j}$ to represent its $(i,j)$ element. Furthermore, let $(\Omega, \mathcal{G}, \mu)$ denote the measure space under consideration, where $\Omega$ is a set, $\mathcal{G}$ is the $\sigma$-field of subsets of $\Omega$, and $\mu$ is the associated measure. For $q>0$, let $L^q(\mu)$ represent the collection of Borel-measurable functions $f:\Omega\rightarrow \mathbb{R}$ such that $\int|f|^qd\mu<\infty$. Let $\mathsf{d}(\cdot,\cdot)$ denote a metric on $\Omega$. We call $f$ $L$-Lipschitz with respect to $\mathsf{d}(\cdot,\cdot)$ if $|f(u_1)-f(u_2)|\le L \cdot \mathsf{d}(u_1,u_2)$ for all $u_1,u_2\in\Omega$, where $L$ is a positive constant. 





\section{Proposed Framework}

\subsection{Problem Formulation}\label{sec.setup}

Consider a classification task with feature space $\mathcal{X}\subset\mathbb{R}^d$ and label space $\mathcal{Y}$, where $d$ is the feature dimension. Here $\mathcal{Y}$ is taken as $\{0,1\}$ for binary classification and $[K]$ for multi-class classification with $K > 2$. Let $\mathbf{X}\in\mathcal{X}$ denote an instance and $\mathrm{Y}\in\mathcal{Y}$ denote its true label. Let $\Psi$ denote the considered hypothesis class consisting of functions $\psi$ defined over $\mathcal{X}$, which, for example, can be neural networks that output predicted label probabilities for each $\mathbf{x}\in\mathcal{X}$. Specifically, for binary classification, $\psi:\mathcal{X}\rightarrow[0,1]$, with $\psi(\mathbf{x})$ representing $P(\mathrm{Y} = 1|\mathbf{X} = \mathbf{x})$, and the classified value is given by $\mathbf{1}(\psi(\mathbf{x}) > 0.5)$. For multi-class classification, $\psi:\mathcal{X}\rightarrow\Delta^{K-1}$ with $K > 2$ and $\Delta^{K-1}$ denoting the $K$-simplex, where {the $j$th component of $\psi(\mathbf{x})$, denoted $\psi(\mathbf{x})_j$,} represents the conditional probability $P(\mathrm{Y} = j|\mathbf{X} = \mathbf{x})$ for $j\in[K]$, with the classified value defined as $\arg\max_{j\in[K]}\psi(\mathbf{x})_j$. 

In applications, the true label $\mathrm{Y}$ is often unobserved, and instead, a set of crowdsourced noisy labels $\widetilde{\mathbf{Y}}\triangleq \{\widetilde{\mathrm{Y}}^{(r)}\}_{r=1}^{R}$ is collected, where $\widetilde{\mathrm{Y}}^{(r)}\in\mathcal{Y}$, {denoting} the label provided by annotator $r$ out of $R$ annotators. Let $\mathcal{D}\triangleq\{\mathbf{X}_i,\widetilde{\mathbf{Y}}_i\}_{i=1}^{n}$ denote the observed data of size $n$, where $\widetilde{\mathbf{Y}}_i$ contains noisy labels provided by $R$ annotators for instance $\mathbf{X}_i$, which may differ from the true label $\mathrm{Y}_i$ for each $i\in[n]$. Our goal is to train a classifier using $\mathcal{D}$ to accurately predict the true label for a future instance.


A common assumption in supervised learning is that the data points $\{\mathbf{X}_i,\mathrm{Y}_i,\widetilde{\mathbf{Y}}_i\}$ for $i\in[n]$ are independently drawn from a probability measure $P^*_{\mathbf{x},\mathrm{y},\widetilde{\mathbf{y}}}$ for $\{\mathbf{X},\mathrm{Y},\widetilde{\mathbf{Y}}\}$, defined over the space $\mathcal{Z}\triangleq\mathcal{X}\times\mathcal{Y}\times\mathcal{Y}^R$. Under this assumption, many existing methods aim to approximate the posterior distribution of the underlying true label $\mathrm{Y}$, given the observed data $\mathbf{X}$ and $\widetilde{\mathbf{Y}}$ \cite{dawid1979maximum, khetan2017learning, guo2023label, cao2019max}. The estimated true label posterior, denoted $P_{\mathrm{y}|\mathbf{x},\widetilde{\mathbf{y}}}$, is then applied to either infer the true labels or to weight the loss functions. For example, \cite{khetan2017learning} utilized $P_{\mathrm{y}|\mathbf{x},\widetilde{\mathbf{y}}}$ as a weight in the loss functions, without considering potential misspecification of the associated model. However, such strategies typically ignore the variability induced in estimating the true label posterior.

To mitigate the effects of potential misspecifications, we propose a conditional distributionally robust risk optimization problem:
\begin{align}\label{eq.minimax}
    \inf_{\psi\in\Psi}\mathsf{R}_{\epsilon}(\psi;P_{\mathrm{y}|\mathbf{x},\widetilde{\mathbf{y}}}),\ \text{with}\ \mathsf{R}_{\epsilon}(\psi;P_{\mathrm{y}|\mathbf{x},\widetilde{\mathbf{y}}})\triangleq\mathbb{E}_{\mathbf{x},\widetilde{\mathbf{y}}}\Big[\sup_{{Q_{\mathrm{y}|\mathbf{x},\widetilde{\mathbf{y}}}}\in\Gamma_\epsilon(P_{\mathrm{y}|\mathbf{x},\widetilde{\mathbf{y}}})}\mathbb{E}_{Q_{\mathrm{y}|\mathbf{x},\widetilde{\mathbf{y}}}}\left\{\ell(\psi(\mathbf{X}), \mathrm{Y})\right\}\Big],
\end{align}
where $\ell(\cdot,\cdot)$ is a loss function, the expectation $\mathbb{E}_{\mathbf{x},\widetilde{\mathbf{y}}}$ is taken with respect to the joint distribution of the observed data $\mathbf{X}$ and $\widetilde{\mathbf{Y}}$, and the expectation $\mathbb{E}_{Q_{\mathrm{y}|\mathbf{x},\widetilde{\mathbf{y}}}}$ is evaluated under the conditional distribution model, denoted $Q_{\mathrm{y}|\mathbf{x},\widetilde{\mathbf{y}}}$, of the true label $\mathrm{Y}$, given $\mathbf{X}$ and $\widetilde{\mathbf{Y}}$. Here, $\Gamma_\epsilon(P_{\mathrm{y}|\mathbf{x},\widetilde{\mathbf{y}}})$ is an \emph{ambiguity set} of probability measures centered around the \emph{reference probability distribution} $P_{\mathrm{y}|\mathbf{x},\widetilde{\mathbf{y}}}$, indexed by $\epsilon>0$ {\color{black}\cite{shapiro2023conditional, blanchet2019quantifying, gao2023distributionally}}. For instance, $\Gamma_\epsilon(P_{\mathrm{y}|\mathbf{x},\widetilde{\mathbf{y}}})$ can be conceptualized as a ``ball'' with $P_{\mathrm{y}|\mathbf{x},\widetilde{\mathbf{y}}}$ at its center and $\epsilon$ as the radius, where elements in the ball represent possible distribution models for $P^*_{\mathrm{y}|\mathbf{x},\widetilde{\mathbf{y}}}$, and the distance between two points is measured using a standard metric for distributions. Specifically,
\begin{align}\label{eq.set}
    \Gamma_\epsilon(P_{\mathrm{y}|\mathbf{x},\widetilde{\mathbf{y}}})=\left\{Q_{\mathrm{y}|\mathbf{x},\widetilde{\mathbf{y}}}\in\mathcal{P}(\mathcal{Y}): \mathscr{d}(Q_{\mathrm{y}|\mathbf{x},\widetilde{\mathbf{y}}},P_{\mathrm{y}|\mathbf{x},\widetilde{\mathbf{y}}})\le\epsilon\right\},
\end{align}
where $\mathcal{P}(\mathcal{Y})$ denotes all Borel probability measures on $\mathcal{Y}$, and $\mathscr{d}$ is a discrepancy metric of probability measures. In this paper, we employ the  Wasserstein distance in Definition \ref{def.Wasserstein} to define the ambiguity set. By taking the supremum in (\ref{eq.minimax}) over the ambiguity set (\ref{eq.set}), we aim to minimize the worst-case risk around the reference distribution, thereby mitigating the impact of potential model misspecifications.

One main obstacle in solving (\ref{eq.minimax}) is constructing a reliable reference distribution $P_{\mathrm{y}|\mathbf{x},\widetilde{\mathbf{y}}}$, which typically depends on an empirical distribution that requires true labels in conventional distributionally robust optimization (DRO). We address this issue by investigating the dual form of the robust risk presented in (\ref{eq.minimax}), which enables us to create a robust pseudo-empirical distribution using a likelihood ratio test, as detailed in Section \ref{sec.optimal_action}. An additional advantage of our approach is that it provides informative upper bounds for the worst-case risk in Section \ref{sec.upper_bound} via the dual formulation.

\begin{remark}\label{remark.2.1}
    For simplicity in theoretical presentation, we assume access to all annotations from all $R$ annotators. However, the theoretical framework presented in this paper is applicable to both single-annotator ($R=1$) and multiple-annotator ($R>1$) scenarios. In our experiments in Section \ref{sec.experiment}, we also consider the scenario of sparse labeling, where we generate a total of $R$ annotators and then randomly select one annotation per instance from these $R$ annotators. We also conduct experiments with varying numbers of annotators for a comprehensive analysis. 
\end{remark}

\subsection{Duality Result and Relaxed Problem}\label{sec.dual}

To derive the pseudo-label generation algorithm and establish a reliable reference distribution, we analyze the dual form of (\ref{eq.minimax}). We first define the Wasserstein distance of order $p$ for $p\in[1,+\infty)$.

\begin{definition}[$p$-Wasserstein distance, \cite{blanchet2019quantifying}]\label{def.Wasserstein}
    For a Polish space $\mathcal{S}$ (i.e., a complete separable metric space) endowed with a metric $c:\mathcal{S}\times\mathcal{S}\rightarrow\mathbb{R}_{\ge 0}$, also called a cost function, let $\mathcal{P}(\mathcal{S})$ represent the set of all Borel probability measures on $\mathcal{S}$, where $\mathbb{R}_{\ge 0}$ represents the set of all nonnegative real values. For $p\ge 1$, let $\mathcal{P}_{p}(\mathcal{S})$ stand for the subset of $\mathcal{P}(\mathcal{S})$ with finite $p$th moments.  Then, for $P_1,P_2\in \mathcal{P}_p(\mathcal{S})$, the Wasserstein distance of order $p$ is defined as
    \begin{align*}
        W_p(P_1,P_2)\triangleq\inf_{\Pi\in\text{Cpl}(P_1,P_2)}{\big[\mathbb{E}_{(S_1,S_2)\sim\Pi}\left\{c^p(S_1,S_2)\right\}\big]^{1/p}},
    \end{align*}
    where $\text{Cpl}(P_1,P_2)$ comprises all probability measures on the product space $\mathcal{S}\times \mathcal{S}$ such that their marginal measures are $P_1(\cdot)$ and $P_2(\cdot)$. {Here, $c^{p}(\cdot,\cdot)$ represents $\{c(\cdot,\cdot)\}^p$.}
\end{definition}

In (\ref{eq.set}), we set $\mathscr{d}(\cdot,\cdot)$ as the $p$-Wasserstein distance and incorporate the constraint $\mathscr{d}(Q_{\mathrm{y}|\mathbf{x},\widetilde{\mathbf{y}}},P_{\mathrm{y}|\mathbf{x},\widetilde{\mathbf{y}}})\le\epsilon$ using the Lagrange formulation, and then establish the strong duality result for (\ref{eq.minimax}) as follows.

\begin{proposition}[dual problem]\label{thm.dual}
    Assume that for every given $\mathbf{x}\in\mathcal{X}$, $\widetilde{\mathbf{y}}\in\mathcal{Y}^R$ and $\psi\in\Psi$, $\ell(\psi(\mathbf{x}),\cdot)\in L^{1}(P_{\mathrm{y}|\mathbf{x},\widetilde{\mathbf{y}}})$, where $L^{1}(\cdot)$ is defined in Section \ref{sec.intro}. Consider $\mathscr{d}(\cdot,\cdot)$ in (\ref{eq.set}) as the Wasserstein distance of order $p$. Then, for any $\epsilon>0$, $\mathsf{R}_{\epsilon}(\psi;P_{\mathrm{y}|\mathbf{x},\widetilde{\mathbf{y}}})$ in (\ref{eq.minimax}) becomes:
    \begin{align}\label{eq.rob_risk_dual}
        \mathsf{R}_{\epsilon}(\psi;P_{\mathrm{y}|\mathbf{x},\widetilde{\mathbf{y}}})=\mathbb{E}_{\mathbf{x},\widetilde{\mathbf{y}}}\left\{\inf_{\gamma\ge 0}\bigg(\gamma\epsilon^p +\mathbb{E}_{P_{\mathrm{y}|\mathbf{x},\widetilde{\mathbf{y}}}}\Big[\sup_{y'\in\mathcal{Y}}\left\{\ell(\psi(\mathbf{X}),y')-\gamma c^p(y',\mathrm{Y})\right\}\Big]\bigg)\right\}.
    \end{align}
\end{proposition}

To avoid solving nested optimization problems, we consider an alternative formulation by swapping the infimum and the first expectation operations: 
\begin{align}\label{eq.re_rob_risk}
    \mathfrak{R}_{\epsilon}(\psi;P_{\mathrm{y}|\mathbf{x},\widetilde{\mathbf{y}}})\triangleq\inf_{\gamma\ge 0}\mathbb{E}_{\mathbf{x},\widetilde{\mathbf{y}}}\left(\gamma\epsilon^p +\mathbb{E}_{P_{\mathrm{y}|\mathbf{x},\widetilde{\mathbf{y}}}}\Big[\sup_{y'\in\mathcal{Y}}\left\{\ell(\psi(\mathbf{X}),y')-\gamma c^p(y',\mathrm{Y})\right\}\Big]\right),
\end{align}
which is an upper bound of $\mathsf{R}_{\epsilon}(\psi;P_{\mathrm{y}|\mathbf{x},\widetilde{\mathbf{y}}})$ according to Proposition \ref{thm.dual}, and hence, (\ref{eq.re_rob_risk}) can be regarded as an relaxation of (\ref{eq.rob_risk_dual}). The empirical counterpart of $\mathfrak{R}_{\epsilon}(\psi;P_{\mathrm{y}|\mathbf{x},\widetilde{\mathbf{y}}})$ is given by
\begin{align}\label{eq.re_rob_risk_emp}
    \widehat{\mathfrak{R}}_{\epsilon}(\psi;P_{\mathrm{y}|\mathbf{x},\widetilde{\mathbf{y}}})\triangleq\inf_{\gamma\ge 0}\mathbb{E}_{P^{(n)}_{\mathbf{x},\widetilde{\mathbf{y}}}}\left(\gamma\epsilon^p +\mathbb{E}_{P_{\mathrm{y}|\mathbf{x},\widetilde{\mathbf{y}}}}\Big[\sup_{y'\in\mathcal{Y}}\left\{\ell(\psi(\mathbf{X}),y')-\gamma c^p(y',\mathrm{Y})\right\}\Big]\right),
\end{align}
where $P^{(n)}_{\mathbf{x},\widetilde{\mathbf{y}}}\triangleq\frac{1}{n}\sum_{i=1}^{n}\delta_{\mathbf{X}_i,\widetilde{\mathbf{Y}}_i}$ is the empirical distribution of $(\mathbf{X},\widetilde{\mathbf{Y}})$ based on the dataset $\mathcal{D}$ defined in Section \ref{sec.setup}. Here, for any $\mathbf{v}\in \mathcal{X}\times\mathcal{Y}^R$, $\delta_{\mathbf{v}}$ represents the Dirac measure on $\mathcal{X}\times\mathcal{Y}^R$, defined as $\delta_{\mathbf{v}}(A)\triangleq\mathbf{1}\{\mathbf{v}\in A\}$ for any $A\subset \mathcal{X}\times\mathcal{Y}^R$. 

\begin{remark}\label{remark.lagrangian}
The Lagrange multiplier $\gamma$ in (\ref{eq.re_rob_risk}) and (\ref{eq.re_rob_risk_emp}) captures the trade-off between robustness and model fitting in the presence of label noise and potential model misspecifications. When {the solution in} $\gamma$ is large, the inner supremum tends to favor $\mathrm{y}'=\mathrm{Y}$, thus encouraging the minimization of the natural risk using the reference distribution directly. In contrast, a small {solution in} $\gamma$ introduces perturbations to the data, pushing the classifier away from the sample instances weighted by the reference distribution. 
\end{remark}

\begin{remark}\label{remark.dual_emp}
    When $p=1$, (\ref{eq.re_rob_risk_emp}) represents the dual form of the following problem:
    \begin{align}\label{eq.re_rob_risk_emp_prim}
        \sup_{Q_{\mathrm{y}|\mathbf{x},\widetilde{\mathbf{y}}}\in\overline{\Gamma}_\epsilon(P_{\mathrm{y}|\mathbf{x},\widetilde{\mathbf{y}}})}\mathbb{E}_{P^{(n)}_{\mathbf{x},\widetilde{\mathbf{y}}}}\left[\mathbb{E}_{Q_{\mathrm{y}|\mathbf{x},\widetilde{\mathbf{y}}}}\left\{\ell(\psi(\mathbf{X}), \mathrm{Y})\right\}\right],
    \end{align}
    where $\overline{\Gamma}_\epsilon(P_{\mathrm{y}|\mathbf{x},\widetilde{\mathbf{y}}})=\Big\{Q_{\mathrm{y}|\mathbf{x},\widetilde{\mathbf{y}}}\in\mathcal{P}(\mathcal{Y}): \mathbb{E}_{P^{(n)}_{\mathbf{x},\widetilde{\mathbf{y}}}}\mathscr{d}(Q_{\mathrm{y}|\mathbf{x},\widetilde{\mathbf{y}}},P_{\mathrm{y}|\mathbf{x},\widetilde{\mathbf{y}}})\le\epsilon\Big\}$. The proof of this statement is deferred to Appendix \ref{appdx.pf_remark_dual_emp}. This result indicates that the empirical robust risk in the relaxed problem (\ref{eq.re_rob_risk_emp}) corresponds to the worst-case risk within an ambiguity set that constrains the size of the average conditional distributional uncertainty.
\end{remark}


\subsection{Generalization Bounds}\label{sec.upper_bound}

With the duality result in Proposition \ref{thm.dual} and the derivations of the alternative formulations (\ref{eq.re_rob_risk}) and (\ref{eq.re_rob_risk_emp}), we now characterize the difference between $\widehat{\mathfrak{R}}_{\epsilon}(\psi;P_{\mathrm{y}|\mathbf{x},\widetilde{\mathbf{y}}})$ and its population counterpart $\mathfrak{R}_{\epsilon}(\psi;P_{\mathrm{y}|\mathbf{x},\widetilde{\mathbf{y}}})$.

\begin{theorem}\label{thm.bias}
Consider the loss function $\ell(\cdot,\cdot)$ in Proposition \ref{thm.dual}, and let the cost function $c(y,y')=\kappa\mathbf{1}(y\ne y')$ for $y,y'\in\mathcal{Y}$, where $\kappa$ is a positive constant.
Assume that there exists a positive constant $M$ such that $\ell(\psi(\mathbf{x}),\mathrm{y})\in[0,M]$ for all $\mathbf{x}\in\mathcal{X}$, $\mathrm{y}\in\mathcal{Y}$, and $\psi\in\Psi$, and that $\ell(\psi(\mathbf{x}),\mathrm{y})$ is $L$-Lipschitz in the second argument with respect to the cost function $c(\cdot,\cdot)$. Then, there exists a positive constant $C_1$ such that for any given $\epsilon>0$, $\psi\in\Psi$, and $0<\eta<1$, with probability at least $1-\eta$: 
\begin{align*}
    \left|\mathfrak{R}_{\epsilon}(\psi;P_{\mathrm{y}|\mathbf{x},\widetilde{\mathbf{y}}})-\widehat{\mathfrak{R}}_{\epsilon}(\psi;P_{\mathrm{y}|\mathbf{x},\widetilde{\mathbf{y}}})\right|\le \frac{C_1 L\kappa^p}{\epsilon^{p-1}\sqrt{n}}+M\sqrt{\frac{\log (1/\eta)}{2n}}.
\end{align*}
\end{theorem}

Theorem \ref{thm.bias} suggests that the {empirical counterpart, $\widehat{\mathfrak{R}}_{\epsilon}(\psi;P_{\mathrm{y}|\mathbf{x},\widetilde{\mathbf{y}}})$,} is a useful approximation for the risk function $\mathfrak{R}_{\epsilon}(\psi;P_{\mathrm{y}|\mathbf{x},\widetilde{\mathbf{y}}})$, as their disparity is bounded and cannot grow indefinitely large. For a finite sample size $n$, this disparity is upper bounded by a finite value depending on the characteristics of the cost and loss functions, as reflected by $\kappa$, $L$, and $M$. As the sample size $n\rightarrow\infty$, the difference tends to zero with high probability, and specifically, the difference is of order $O(n^{-1/2})$.

Next, we establish an informative bound for the empirical robust risk minimizer.
For $\psi_1,\psi_2\in\Psi$ and for any given norm $\|\cdot\|$, let $\|\psi_1-\psi_2\|_\infty\triangleq\sup_{\mathbf{x}\in\mathcal{X}}\|\psi_1(\mathbf{x})-\psi_2(\mathbf{x})\|$. 
Here, $\|\cdot\|$ can be taken as any specific norms, including the $L^q$ norm with $q\ge 1$ that is defined in Section \ref{sec.intro}.

\begin{corollary}[Empirical Robust Minimizer]\label{thm.regret_empirical}
Let $\widehat{\psi}_{\epsilon,n}\in\inf_{\psi\in\Psi}\widehat{\mathfrak{R}}_{\epsilon}(\psi;P_{\mathrm{y}|\mathbf{x},\widetilde{\mathbf{y}}})$. Under the assumptions in Theorem \ref{thm.bias}, if we further assume that the loss function $\ell(\cdot,\cdot)$ is $L'$-Lipschitz in terms of the first argument with respect to the supremum metric $\|\cdot\|_\infty$, then there exists a positive constant $C_2$ such that for any $\epsilon>0$ and $0<\eta<1$, with probability at least $1-\eta$, {\color{black}the empirical robust risk minimizer $\widehat{\psi}_{\epsilon,n}$} satisfies:
\begin{align*}
    &\mathsf{R}_{\epsilon}(\widehat{\psi}_{\epsilon,n};P_{\mathrm{y}|\mathbf{x},\widetilde{\mathbf{y}}})\le\mathfrak{R}_{\epsilon}(\widehat{\psi}_{\epsilon,n};P_{\mathrm{y}|\mathbf{x},\widetilde{\mathbf{y}}})\\
    \le&\inf_{\psi\in\Psi}\mathfrak{R}_{\epsilon}(\psi;P_{\mathrm{y}|\mathbf{x},\widetilde{\mathbf{y}}})+C_2\left\{\frac{L\kappa^p}{\epsilon^{p-1}}+L'\int_{0}^{\infty}\sqrt{\log N(s; \Psi, \|\cdot\|_\infty)}ds\right\}\cdot\frac{1}{\sqrt{n}}+2M\sqrt{\frac{\log (1/\eta)}{2n}},
\end{align*}
where $N(s; \Psi, \|\cdot\|_\infty)$ denotes the $s$-covering number of $\Psi$ with respect to the supremum metric.
\end{corollary}



\section{Implementation Algorithm}\label{sec.theory}

In Section \ref{sec.optimal_action}, we {derive} the analytical solution to the dual robust risk minimization problem {in (\ref{eq.re_rob_risk_emp})}, which leads to the development of a novel approach for assigning pseudo-labels using the likelihood ratio test. These pseudo-labels facilitate the construction of a \emph{pseudo-empirical distribution}, {serving} as a robust reference distribution in using (\ref{eq.re_rob_risk_emp}). In Section \ref{sec.close_em_risk}, we derive the optimal value {in $\gamma$} for the empirical robust risk (\ref{eq.re_rob_risk_emp}) and establish its closed-form expression. This analysis provides a principled framework for balancing the trade-off between robustness and model fitting and also motivates an efficient one-step update technique in solving the robust empirical risk minimization problem.

\subsection{Optimal Solution for Single Data Point}\label{sec.optimal_action}

In this subsection, we determine the optimal value of $\psi(\mathbf{x})$ in (\ref{eq.re_rob_risk_emp}) for a single data point $(\mathbf{x}, \widetilde{\mathbf{y}})$. To simplify the analysis, we first focus on the binary classification problem with $\mathcal{Y}=\{0,1\}$ and consider a broad family of loss functions of the form: 
\begin{align}\label{eq.loss_bianry}
    \ell(\psi(\mathbf{x}),\mathrm{y})=(1-\mathrm{y})\mathcal{T}(1-\psi(\mathbf{x}))+\mathrm{y}\mathcal{T}(\psi(\mathbf{x})), 
\end{align}
where $\psi(\mathbf{x})$ represents the conditional distribution $P(\mathrm{Y}=1|\mathbf{X}=\mathbf{x})$ as described in Section \ref{sec.setup}, $\mathcal{T}:[0,1]\rightarrow \mathcal{I}$ is a bounded, decreasing, and twice differentiable function, and $\mathcal{I}$ is a compact subset of $\mathbb{R}$. 

For any given $\mathbf{x}\in\mathcal{X}$ and $\widetilde{\mathbf{y}}\in\mathcal{Y}^{R}$, let $P_j(\mathbf{x},\widetilde{\mathbf{y}})\triangleq P(\mathrm{Y}=j|\mathbf{X}=\mathbf{x},\widetilde{\mathbf{Y}}=\widetilde{\mathbf{y}})$ for $j=0,1$. With the loss function in (\ref{eq.loss_bianry}) and the metric $c(\cdot,\cdot)$ considered in Theorem \ref{thm.bias}, minimizing (\ref{eq.re_rob_risk_emp}) with respect to $\psi\in\Psi$ becomes: 
\begin{align}\label{eq.binary_problem}
    &\inf_{\psi\in\Psi}\inf_{\gamma\ge 0}\big[\gamma\epsilon^p+P_0(\mathbf{x},\widetilde{\mathbf{y}})\max\{\mathcal{T}(1-\psi(\mathbf{x})), \mathcal{T}(\psi(\mathbf{x}))-\gamma\kappa^p\}\notag\\
    &\hspace{2.7cm}+P_1(\mathbf{x},\widetilde{\mathbf{y}})\max\{\mathcal{T}(1-\psi(\mathbf{x}))-\gamma\kappa^p, \mathcal{T}(\psi(\mathbf{x}))\}\big],
\end{align}
and let $\psi^\star$ denote the solution of (\ref{eq.binary_problem}). For $\epsilon$ and $\kappa$ described in Theorem \ref{thm.bias}, let $\varrho(\epsilon)\triangleq\epsilon^p/\kappa^p$. The following theorem shows that (\ref{eq.binary_problem}) has a closed-form solution, with its form varying based on whether $\mathcal{T}$ is concave or convex.

\begin{theorem}[Optimal Action for Single Data Point: Binary Case]\label{thm.optimal_action}
Let $\mathbf{x}\in\mathcal{X}$ and $\widetilde{\mathbf{y}}\in\mathcal{Y}^R$ be given. Then, for a concave function $\mathcal{T}$, the optimal solution for (\ref{eq.binary_problem}) is given by:
\begin{align*}
    \psi^\star(\mathbf{x})
    =& \left\{
    \begin{aligned}
        &\ j,\ \text{if}\ P_j(\mathbf{x},\widetilde{\mathbf{y}})\ge\varrho(\epsilon)+\varpi_1\ \ \text{for}\ j=0,1;\\
        &\ 1/2, \ \text{otherwise},
    \end{aligned}
    \right.
\end{align*}
with $\varpi_1=\{\mathcal{T}(0)-\mathcal{T}(1/2)\}/\{\mathcal{T}(0)-\mathcal{T}(1)\}\in(0,1/2]$; and for a convex function $\mathcal{T}$, the optimal solution of (\ref{eq.binary_problem}) is given by:
\begin{align*}
    \psi^\star(\mathbf{x})
    =& \left\{
    \begin{aligned}
        &\ j,\ \text{if}\ P_j(\mathbf{x},\widetilde{\mathbf{y}})\ge \varrho(\epsilon)+\varpi_2\ \text{for}\ \ j=0,1,\\
        &\ \mathdutchcal{t}^*_j,\ \text{if}\ \varrho(\epsilon)+1/2< P_j(\mathbf{x},\widetilde{\mathbf{y}})< \varrho(\epsilon)+\varpi_2\ \text{for}\ j=0,1,\\
        &\ 1/2, \ \text{otherwise},
    \end{aligned}\right.
\end{align*}
where $\varpi_2=\{\mathcal{T}'(0)\}/\{\mathcal{T}'(0)+\mathcal{T}'(1)\}\in[1/2, 1)$, $\mathdutchcal{t}^*_0$ is the unique solution of $\{P_0(\mathbf{x},\widetilde{\mathbf{y}})-\varrho(\epsilon)\}\mathcal{T}'(1-\mathdutchcal{t})=\{P_1(\mathbf{x},\widetilde{\mathbf{y}})+\varrho(\epsilon)\}\mathcal{T}'(\mathdutchcal{t})$ for $\mathdutchcal{t}\in(0,\frac{1}{2})$, and $\mathdutchcal{t}^*_1$ is the unique solution of $\{P_0(\mathbf{x},\widetilde{\mathbf{y}})+\varrho(\epsilon)\}\mathcal{T}'(1-\mathdutchcal{t})=\{P_1(\mathbf{x},\widetilde{\mathbf{y}})-\varrho(\epsilon)\}\mathcal{T}'(\mathdutchcal{t})$ for $\mathdutchcal{t}\in(\frac{1}{2}, 1)$.
\end{theorem}

\begin{remark} \label{remark.LRT_binary}
The optimal solution $\psi^\star(\mathbf{x})$ in Theorem \ref{thm.optimal_action} can also be expressed in a likelihood ratio format, which naturally leads to a novel algorithm for assigning \emph{robust pseudo-labels}. Specifically, when $\mathcal{T}$ is concave, the optimal solution can be expressed as: $\psi^\star(\mathbf{x})=0$ if $P_0(\mathbf{x},\widetilde{\mathbf{y}})/P_1(\mathbf{x},\widetilde{\mathbf{y}})\ge\mathdutchcal{C}_1$; $\psi^\star(\mathbf{x})=1$ if $P_1(\mathbf{x},\widetilde{\mathbf{y}})/P_0(\mathbf{x},\widetilde{\mathbf{y}})\ge\mathdutchcal{C}_1$; and $\psi^\star(\mathbf{x})=1/2$ otherwise, where $\mathdutchcal{C}_1\triangleq (\varrho(\epsilon)+\varpi_1)/\{1-(\varrho(\epsilon)+\varpi_1)\}>1$ serves as a threshold for the likelihood ratio test. Consequently, for a data point $(\mathbf{x},\widetilde{\mathbf{y}})$, if $P_1(\mathbf{x},\widetilde{\mathbf{y}})/P_0(\mathbf{x},\widetilde{\mathbf{y}})\ge\mathdutchcal{C}_1$, we assign a robust pseudo-label $\mathrm{y}^\star=1$; if $P_0(\mathbf{x},\widetilde{\mathbf{y}})/P_1(\mathbf{x},\widetilde{\mathbf{y}})\ge\mathdutchcal{C}_1$, we assign $\mathrm{y}^\star=0$. Leveraging the likelihood ratio format also facilitates extending the robust pseudo-label selection method to the multi-class case by considering pairwise comparisons. Specifically, if $P_{k^\star}(\mathbf{x},\widetilde{\mathbf{y}})/\max_{j\ne k^\star}P_j(\mathbf{x},\widetilde{\mathbf{y}})\ge\mathdutchcal{C}_1$, we assign the pseudo-label $\mathrm{y}^\star=k^\star$ to the instance.
\end{remark}

\begin{remark}
Existing pseudo-labeling methods \cite{khetan2017learning, tanaka2018joint} typically identify the underlying true label as the one with the highest probability in the approximated true label posterior. In contrast, the proposed approach in Remark \ref{remark.LRT_binary} considers both the highest and second-highest predicted probabilities. A pseudo-label is assigned only if the ratio of these probabilities exceeds a specified threshold. This strategy ensures that pseudo-labels are assigned to instances with high confidence, effectively filtering out uncertain data.
\end{remark}

\begin{remark}\label{remark.LRT_theory}
In the special case where $P_j(\mathbf{x},\widetilde{\mathbf{y}})\propto \tau_j(\widetilde{\mathbf{y}};\mathbf{x})P_j(\mathbf{x})$, with $\tau_j(\widetilde{\mathbf{y}};\mathbf{x})=P^*(\widetilde{\mathbf{Y}}=\widetilde{\mathbf{y}}|\mathrm{Y}=j,\mathbf{X}=\mathbf{x})$ denoting the noisy label transition probability and $P_j(\mathbf{x})$ representing a proper prior for $\mathrm{Y}=j$ conditional on $\mathbf{x}$ for $j=0,1$, previous studies have indicated the existence of a Chernoff information-type bound on the probability of error for robust pseudo-label selection, as described in Remark \ref{remark.LRT_binary} \cite{guo2023label, cover1999elements}. Specifically, for a fixed instance $\mathbf{x}$, let a pseudo-label $\mathrm{Y}^\star$ be generated as described in Remark \ref{remark.LRT_binary}, which depends on $\mathbf{x}$ and the corresponding noisy label vector $\widetilde{\mathbf{Y}}$. Consider the Bayes error, defined as $\Re_{\text{Bayes}}\triangleq \sum_{j=0,1}P_j(\mathbf{x})P^*(\mathrm{Y}^\star\ne j|\mathrm{Y}=j,\mathbf{x})$. According to Section 11.9 of \cite{cover1999elements}, $\Re_{\text{Bayes}}\le \exp\{-C(\tau_0(\cdot;\mathbf{x}), \tau_1(\cdot;\mathbf{x}))\}$, where $C(\cdot, \cdot)$ represents the Chernoff information between two distributions.



\end{remark}

\begin{remark}
In practice, one can use either uninformative priors, such as a uniform prior for each class, or informative priors derived from pre-trained or concurrently trained models for $P_j(\mathbf{x})$ as discussed in Remark \ref{remark.LRT_theory} \cite{guo2023label, shwartz2022pre}. Moreover, the estimation of $P_j(\mathbf{x},\widetilde{\mathbf{y}})$ is not limited to Bayes's rule. For example, \cite{cao2019max} proposed aggregating data and noisy label information by maximizing the $f$-mutual information gain.
\end{remark}

\begin{remark}
Theorem \ref{thm.optimal_action} is developed based on the assumption that the function $\mathcal{T}$ is convex or concave. In our experiments, we use the cross-entropy loss for $\ell$, meaning $\mathcal{T}(\mathdutchcal{t})=-\log\mathdutchcal{t}$ for $\mathdutchcal{t}>0$. To meet the required conditions, we clip its input to $[0.01, 1 - 0.01]$ to ensure $\mathcal{T}(\cdot)$ remains bounded.
\end{remark}

Next, we extend the preceding development for binary classification to multi-class scenarios with $K>2$. Letting $\mathcal{T}(\cdot)$ in (\ref{eq.loss_bianry}) be specified as $\mathcal{T}(t)=1-t$, we extend loss function form (\ref{eq.loss_bianry}) to facilitate the worst-case misclassification probability in multi-class scenarios: $\ell(\psi(\mathbf{x}),\mathrm{y})=\sum_{j=1}^{K}\mathbf{1}(\mathrm{y}=j)\{1-\psi(\mathbf{x})_j\}$. For ease of presentation, we sometimes omit the dependence on $\mathbf{x}$ and $\widetilde{\mathbf{y}}$ in the notation. Specifically, for $j\in[K]$, we let $P_j\triangleq P_j(\mathbf{x}, \widetilde{\mathbf{y}})\triangleq P(\mathrm{Y}=j|\mathbf{x}, \widetilde{\mathbf{y}})$ and $\psi_j\triangleq\psi(\mathbf{x})_j$. In a manner similar to deriving (\ref{eq.binary_problem}), given $\mathbf{x}$, minimizing (\ref{eq.re_rob_risk_emp}) with respect to $\psi(\mathbf{x})$ can be expressed as:
\begin{align}\label{eq.multi_problem}
    \inf_{\psi\in\Psi}\inf_{\gamma\ge 0}\Big[\gamma\epsilon^p+\sum_{j=1}^{K}P_j&\max\{1-\psi_1-\gamma\kappa^p,\ldots,1-\psi_{j-1}-\gamma\kappa^p,\notag\\
    &1-\psi_j,1-\psi_{j+1}-\gamma\kappa^p,\ldots,1-\psi_K-\gamma\kappa^p\}\Big].
\end{align}

\begin{theorem}[Optimal Action for Single Data Point: Multi-class Case]\label{thm.optimal_action_multi}
Let $\{P_1,\ldots,P_K\}$ be arranged in decreasing order, denoted $P^{(1)}\ge\ldots\ge P^{(K)}$, with the associated indexes denoted $\chi(1),\ldots,\chi(K)$. Let $\psi^\star$ denote the solution of the outer optimization problem in (\ref{eq.multi_problem}). {For $j\in[K]$, let $\psi^{\star(j)}$ denote the $\chi(j)$-th component of $\psi(\mathbf{x})$ corresponding to $P^{(j)}$. Then, the elements of $\psi^\star$ are given as follows:}
\begin{itemize}
    \item[(a).] If $\frac{1}{K}\ge \frac{1}{k}\sum_{j=1}^{k}P^{(j)}-\frac{1}{k}\varrho(\epsilon)$ for all ${k}\in[K-1]$, then ${\psi^{\star(j)}}=\frac{1}{K}$ for all $j\in[K]$.
    \item[(b).] If there exists some $k_0\in[K-1]$ such that $\frac{1}{k_0}\sum_{j=1}^{k_0}P^{(j)}-\frac{1}{k_0}\varrho(\epsilon)>\frac{1}{K}$, and $\frac{1}{k_0}\sum_{j=1}^{k_0}P^{(j)}-\frac{1}{k_0}\varrho(\epsilon)\ge\frac{1}{k}\sum_{j=1}^{k}P^{(j)}-\frac{1}{k}\varrho(\epsilon)$ for all ${k}\in[K-1]$, then $\psi^{\star(j)}=\frac{1}{k_0}$ for $j\in[k_0]$ and $\psi^{\star(j)}=0$ for $j=k_0+1,\ldots,K$.  
\end{itemize}
\end{theorem}

\begin{remark}
The robust pseudo-labeling method described in Remark \ref{remark.LRT_binary} can also be extended from Theorem \ref{thm.optimal_action_multi}. Specifically, by Theorem \ref{thm.optimal_action_multi}, if $P^{(1)}\ge\max\{\frac{1}{K}+\varrho(\epsilon),P^{(2)}+\varrho(\epsilon)\}$, then the optimal solution is: $\psi^{\star(1)}=1$ and $\psi^{\star(j)}=0$ for $j=2,\ldots, K$, which can also be expressed in a likelihood ratio format and applied to assign robust pseudo-labels. 
\end{remark}

\subsection{Closed-Form Robust Risk }\label{sec.close_em_risk}

We investigate the empirical robust risk (\ref{eq.re_rob_risk_emp}) by examining its closed form expression. For $i\in[K]$ and  $j\in[K]$, we let $P_{i,j}\triangleq P_j(\mathbf{x}_i, \widetilde{\mathbf{y}}_i)\triangleq {P(\mathrm{Y}=j|\mathbf{X}=\mathbf{x}_i, \widetilde{\mathbf{Y}}=\widetilde{\mathbf{y}}_i)}$ and $\psi_{i,j}\triangleq\psi(\mathbf{x}_i)_j$. For simplicity, we denote the Wasserstein robust loss in (\ref{eq.re_rob_risk_emp}) and the nominal loss respectively as:
\begin{align}
    \widehat{\mathfrak{R}}_\epsilon&=\inf_{\gamma\ge 0}\Big[\gamma\epsilon^p+\frac{1}{n}\sum_{i=1}^{n}\sum_{j=1}^{K}P_{i,j}\max\big\{\mathcal{T}(\psi_{i,1})-\gamma\kappa^p,\ldots,\mathcal{T}(\psi_{i,j-1})-\gamma\kappa^p\notag\\
    &\hspace{1.5cm}\mathcal{T}(\psi_{i,j}),\mathcal{T}(\psi_{i,j+1})-\gamma\kappa^p,\ldots,
    \mathcal{T}(\psi_{i,K})-\gamma\kappa^p\big\}\Big]\ \text{for}\ \epsilon>0;\label{eq.w_robust_loss}\\
    \widehat{\mathfrak{R}}&=\frac{1}{n}\sum_{i=1}^{n}\sum_{j=1}^{K}P_{i,j}\mathcal{T}(\psi_{i,j}).\label{eq.nominal_loss}
\end{align}
For given $\mathbf{x}_i$, we sort $\{\psi_{i,1},\ldots,\psi_{i,K}\}$ in decreasing order, denoted as $\psi^{(1)}_i\ge\ldots\ge\psi^{(K)}_i$. Let $\alpha_{i,j}\triangleq\mathcal{T}(\psi_{i}^{(K)})-\mathcal{T}(\psi_{i,j})$ for $i\in[n]$ and $j\in[K]$, and sort $\{\alpha_{i,j}:i\in[n],j\in[K]\}$ in decreasing order, denoted as $\alpha^{(1)}\ge\ldots\ge\alpha^{(nK)}$. Correspondingly, the $P_{i,j}$ values with the associated indexes are denoted as $P^{(1)},\ldots,P^{(nK)}$. For any $\varrho(\epsilon)$, define an associated positive integer $s^*\in[nK+1]$ as follows: if $\frac{1}{n}P^{(1)}<\varrho(\epsilon)<\frac{1}{n}\sum_{t=1}^{nK}P^{(t)}$, then there exists $s^*\in\{2,\ldots,nK\}$ such that $\frac{1}{n}\sum_{t=1}^{s}P^{(t)}< \varrho(\epsilon)$ for $s< s^*$, and $\frac{1}{n}\sum_{t=1}^{s}P^{(t)}\ge \varrho(\epsilon)$ for $s\ge s^*$; if $\varrho(\epsilon)\le\frac{1}{n}P^{(1)}$, then $s^*$ is set as 1; if $\varrho(\epsilon)\ge\frac{1}{n}\sum_{t=1}^{nK}P^{(t)}$, then $s^*$ is set as $nK+1$.

Let $\gamma^\star_{\psi}$ denote the optimal value of $\gamma$ in $\widehat{\mathfrak{R}}_\epsilon$ in (\ref{eq.w_robust_loss}), where its dependence on $\epsilon$ is implicit, but its dependence on $\psi$ is explicit. The following theorem presents this value, based on which we demonstrate that the Wasserstein robust loss $\widehat{\mathfrak{R}}_\epsilon$ can be expressed as the nominal loss plus an additional term $\frac{1}{n}\sum_{t=1}^{s^*-1}P^{(t)}\alpha^{(t)}\mathbf{1}(s^*>1)$ that prevents the classifier from becoming {overly} certain on the data. 


\begin{theorem}[Closed-Form Robust Risk]\label{thm.robust_risk_form}
The optimal value of $\gamma$ in (\ref{eq.w_robust_loss}) is given by $\gamma^{\star}_{\psi}\triangleq\alpha^{(s^*)}/\kappa^p$, and the resulting robust risk is expressed as 
\begin{align*}
    \widehat{\mathfrak{R}}_\epsilon=\widehat{\mathfrak{R}}+\frac{1}{n}\sum_{t=1}^{s^*-1}P^{(t)}\alpha^{(t)}\mathbf{1}(s^*>1)+O\left(\frac{1}{n}\right)\alpha^{(s^*)}.
\end{align*}
\end{theorem}

\begin{remark}
    Theorem \ref{thm.robust_risk_form} shows that minimizing the Wasserstein robust loss $\widehat{\mathfrak{R}}_\epsilon$ in (\ref{eq.w_robust_loss}) effectively minimizes the nominal loss $\widehat{\mathfrak{R}}$ in (\ref{eq.nominal_loss}) while simultaneously penalizing terms associated with $|\alpha_i|$ values {exceeding} a certain threshold $|\alpha^{s^*}|$, weighted by the corresponding reference probability values. This minimization prevents the classifier from becoming overly confident in certain data points, particularly when there are potential misspecifications in the approximated true label posterior. 
\end{remark}

\begin{remark}
    As suggested by Remark \ref{remark.lagrangian}, Theorem \ref{thm.robust_risk_form} provides a guideline for balancing robustness and model fitting by deriving the optimal value for $\gamma$ in (\ref{eq.w_robust_loss}). In Section \ref{sec.algorithm}, we develop a one-step update method for determining $\gamma^\star_\psi$.
\end{remark}

\subsection{Training using Conditional Distributionally Robust True Label Posterior}\label{sec.algorithm}

In this subsection, we outline the steps for approximating the true label posterior, constructing the pseudo-empirical distribution as the reference distribution for solving the robust risk minimization problem (\ref{eq.re_rob_risk_emp}), and subsequently training classifiers robustly. The pseudo code for the training process is provided in Algorithm \ref{algorithm_AdatpGamma} in Appendix \ref{appdx.exp_setup}. Here we elaborate on the details.

\textbf{Approximating noise transitions probabilities.}  We begin by warming up the classifiers on the noisy training data, denoted $\check{\mathcal{D}}=\{\mathbf{x}_i,\check{\mathrm{y}}_i\}_{i=1}^{n}$, where $\check{\mathrm{y}}_i$ represents the majority vote label for instance $\mathbf{x}_i$, determined by the label that receives the highest number of votes from the annotators. After warming up the classifiers for 20-30 epochs, {we sort the dataset by the cross-entropy loss values and collect a subset of size $m$ with the smallest $m$ losses}, denoted as $\mathcal{D}^\star_0=\{\mathbf{x}_i,\check{\mathrm{y}}_i\}_{i=1}^{m}$, where $m\ll n$, {and the ratio of $m$ to $n$ is set to 1 minus the estimated noise rate}. Next, we estimate the noise transition probabilities by $\widehat{\tau}_j(\widetilde{\mathbf{y}})=\sum_{i=1}^{m}\mathbf{1}(\widetilde{\mathbf{y}}=\widetilde{\mathbf{y}}_i, \check{\mathrm{y}}_i=j)/\sum_{i=1}^{m}\mathbf{1}(\check{\mathrm{y}}_i=j)$ for $\widetilde{\mathbf{y}}\in[K]^R$ and $j\in[K]$ (Line 1 of Algorithm \ref{algorithm_AdatpGamma}). With $\widehat{\tau}_j(\widetilde{\mathbf{y}})$, we then iteratively update the approximated true label posterior, construct the pseudo-empirical distribution, and robustly train the classifiers (Lines 2-13 of Algorithm \ref{algorithm_AdatpGamma}). Here, we employ the straightforward frequency-counting method for noise transition estimation for simplicity. However, our approach is versatile and can be integrated with various methods for estimating the noise transition matrices or the true label posterior. Additional experimental results using more advanced transition matrix estimation methods are provided in Appendix \ref{appdx.experiment}.

\textbf{Constructing a pseudo-empirical distribution.} We train two classifiers, $\psi^{(1)}$ and $\psi^{(2)}$, in parallel each serving as an informative prior for the other. In the $t$th epoch, the approximated true label posterior with prior $\psi^{(\iota)}$ is updated as $\widehat{P}^{(\iota)}_j(\mathbf{x},\widetilde{\mathbf{y}})\triangleq\widehat{P}^{(\iota)}(\mathrm{Y}=j|\widetilde{\mathbf{Y}}=\widetilde{\mathbf{y}},\mathbf{X}=\mathbf{x})\propto \psi^{(\iota)}_j(\mathbf{x})\cdot \widehat{\tau}_j(\widetilde{\mathbf{y}})$ (Line 4 of Algorithm \ref{algorithm_AdatpGamma}), where $\psi^{(\iota)}_j(\mathbf{x})$ denotes the $j$th element of the vector-valued function $\psi^{(\iota)}(\mathbf{x})$ for $j\in[K]$ and $\iota=1,2$. As described in Remark \ref{remark.LRT_binary}, for $i\in[n]$, if $\widehat{P}^{(\iota)}_{k^\star}(\mathbf{x}_i,\widetilde{\mathbf{y}}_i)/\max_{j\ne k^\star}\widehat{P}^{(\iota)}_j(\mathbf{x},\widetilde{\mathbf{y}})\ge\mathdutchcal{C}$ for a pre-specified threshold $\mathdutchcal{C}>1$, we assign the robust pseudo-label $\mathrm{y}^\star_i=k^\star$ to the instance and collect it into $\mathcal{D}^{\star}_{t,\iota}$ (Lines 5-7 of Algorithm \ref{algorithm_AdatpGamma}). The pseudo-empirical distribution $P^\star_{t,\iota}$ is updated based on $\mathcal{D}^{\star}_{t,\iota}$ (Line 8 of Algorithm \ref{algorithm_AdatpGamma}).

\textbf{Robustly training the classifiers.} For $\iota=1,2$, let $\backslash\iota\triangleq 1$ if $\iota=2$; and $\backslash\iota\triangleq 2$ if $\iota=1$. With the updated pseudo-empirical distribution, the classifier $\psi^{(\iota)}$ is then trained by minimizing the empirical robust risk (\ref{eq.re_rob_risk_emp}) with the reference distribution $P^\star_{t,\backslash\iota}$ (Line 9 of Algorithm \ref{algorithm_AdatpGamma}). After updating the classifier $\psi^{(\iota)}$ with $\gamma^{(\iota)}_{t-1}$ from the previous iteration, we take one step to update the $\gamma$ value $\gamma^{(\iota)}_{t}$. In particular, as suggested by Theorem \ref{thm.robust_risk_form}, we use $\gamma_{0,t}=|\alpha^{(s^*)}|/\kappa^p$ as a reference value for $\gamma$ (Lines 11-12 f Algorithm \ref{algorithm_AdatpGamma}). We then update $\gamma^{(\iota)}_{t}$ by minimizing $\left[\gamma\{\epsilon^p -\mathbb{E}_{P^\star_{t,\backslash\iota}}c^p(y',\mathrm{Y})\}+\frac{\lambda}{2}(\gamma-\gamma_0)^2\right]$ (Line 13 of Algorithm \ref{algorithm_AdatpGamma}) with respect to $\gamma$, where $y'$ is determined by (\ref{eq.re_rob_risk_emp}) after updating $\psi^{(\iota)}$, and $\lambda>0$ is a positive constant that determines the learning rate of $\gamma$.

\section{Experimental Results}\label{sec.experiment}

\textbf{Datasets and model architectures.} We evaluate the performance of the proposed AdaptCDRP on two datasets, CIFAR-10 and CIFAR-100 \cite{krizhevsky2009learning}, by generating synthetic noisy labels (details provided below), as well as four datasets, CIFAR-10N \cite{wei2022learning}, CIFAR-100N \cite{wei2022learning}, LabelMe \cite{rodrigues2017learning, torralba2010labelme}, and Animal-10N \cite{song2019selfie}, which contain human annotations. For all datasets except LabelMe, we set aside 10\% of the original data, together with the corresponding synthetic or human annotated noisy labels, to validate the model selection procedure. We use the ResNet-18 architecture \cite{he2016deep} for CIFAR-10 and CIFAR-10N, and the ResNet-34 architecture \cite{he2016deep} for CIFAR-100 and CIFAR-100N. Following \cite{rodrigues2018deep}, we employ a pretrained VGG-16 model with a 50\% dropout rate for the LabelMe dataset. In line with \cite{song2019selfie}, the VGG19-BN architecture \cite{simonyan2014very} is used for the Animal-10N dataset. {Further details on the datasets and experimental setup are provided in Appendix \ref{appdx.exp_setup}.} 

\textbf{Noise generation.} We generate synthetic annotations on the CIFAR-10 and CIFAR-100 datasets using Algorithm 2 from \cite{xia2020part}. Three groups of annotators, labeled as IDN-LOW, IDN-MID, and IDN-HIGH, are considered, with average labeling error rates of approximately 20\%, 35\%, and 50\%, respectively, representing low, intermediate, and high error rates. Each group consists of $R=5$ annotators. To assess the algorithms in an incomplete labeling setting, we randomly select only one annotation per instance from the $R$ annotators for the training dataset rather than using all available annotations \cite{guo2023label}. {Further details on noise generation are provided in Appendix \ref{appdx.exp_setup}.}


\textbf{Comparison with SOTA methods.} We compare our method with a comprehensive set of state-of-the-art approaches, including: (1) CE (Clean) with clean labels; (2) CE (MV) with majority vote labels; (3) CE (EM) \cite{dawid1979maximum}; (4) Co-teaching \cite{han2018co}; (5) Co-teaching+ \cite{yu2019does}; (6) CoDis \cite{xia2023combating}; (7) LogitClip \cite{wei2023mitigating}; (8) DoctorNet \cite{guan2018said}; (9) MBEM \cite{khetan2017learning}; (10) CrowdLayer \cite{rodrigues2018deep}; (11) TraceReg \cite{tanno2019learning}; (12) Max-MIG \cite{cao2019max}; (13) CoNAL \cite{chu2021learning}; and (14) CCC \cite{zhang2024coupled}. We report the average test accuracy over five repeated experiments, each with a different random seed, on synthetic datasets, CIFAR-10 and CIFAR-100, with instance-dependent label noise introduced at low, intermediate, and high error rates. Standard errors are shown following the plus/minus sign ($\pm$), and the two highest accuraries are highlighted in bold. Table \ref{table_real} presents evaluation results on four real-world datasets. As shown, our AdaptCDRP consistently outperforms competing methods across all scenarios. To further explore the impact of annotation sparsity, we conduct additional experiments with the number of annotators ranging from 5 to 100, with each instance labeled only once. Figure \ref{Fig.ACC_cifar10} illustrates the average accuracy across different numbers of annotators on CIFAR-10, highlighting the advantages of the proposed method under diverse settings. The results for the CIFAR-100 dataset are shown in Figure \ref{Fig.ACC_cifar100} in Appendix \ref{appdx.exp_result}.
 
\textbf{Hyper-parameter analysis.} We investigate the impact of the hyperparameter $\epsilon$ in the empirical robust risk (\ref{eq.re_rob_risk_emp}). Under our experiment setup, $\epsilon$ should be chosen within $(0,1/K)$ for a $K$-class classification problem with $K\ge 2$ as demonstrated in the proof of Theorem \ref{thm.robust_risk_form} in Appendix \ref{appdx.pf_thm_robust_risk_form}. Hence, we take $\epsilon\in(0, 0.1)$ for CIFAR-10 and $\epsilon\in(0, 0.01)$ for CIFAR-100, with the results presented in Figure \ref{figure_gamma}. The results suggest that setting $\epsilon$ near zero leads to relatively low test accuracies, highlighting the importance of CDRO under model specification when handling noisy labels. Furthermore, continually increasing $\epsilon$ eventually results in a drop in accuracy due to excessive noise injection into the data.

\textbf{Additional experimental results.} To further evaluate the performance of the proposed method across various scenarios, we conducted additional experiments, detailed in Appendix \ref{appdx.exp_result}. Specifically, we compare different annotation aggregation methods, present average test accuracies and robust pseudo-label accuracies during training, assess sensitivity to the number of warm-up epochs, explore different noise transition estimation methods, and examine the impact of sparse annotation.

\begin{table}\footnotesize\fontsize{8}{0.0}
  \renewcommand\arraystretch{0.83}
  \setlength\tabcolsep{2.0pt}
  \caption{Average accuracies (with associated standard errors expressed after the $\pm$ signs) for learning the CIFAR-10 and CIFAR-100 datasets ($R=5$).}
  \label{table_cifar}
  \centering
  \begin{tabular}{ccccccc}
    \toprule 
    \multirow{2}{*}{Method} & \multicolumn{3}{c}{CIFAR-10} & \multicolumn{3}{c}{CIFAR-100} \\
    \cmidrule(lr){2-4}\cmidrule(lr){5-7}
    & IDN-LOW & IDN-MID & IDN-HIGH  & IDN-LOW & IDN-MID & IDN-HIGH  \\
    \midrule
    CE (Clean) & \multicolumn{3}{c}{$88.60_{\pm 0.79}$} & \multicolumn{3}{c}{$58.75_{\pm 0.55}$} \\
    \cmidrule(lr){1-7}
    CE (MV) & $80.90_{\pm 0.88}$ & $76.05_{\pm 0.70}$ & $69.65_{\pm 1.73}$ & $50.96_{\pm 0.49}$ & $44.80_{\pm0.99}$ & $38.51_{\pm0.66}$ \\
    CE (EM) \cite{dawid1979maximum} & $81.15_{\pm 0.74}$ & $75.84_{\pm 0.97}$ & $69.85_{\pm 1.43}$ & $51.29_{\pm 1.00}$ & $45.24_{\pm 0.41}$ & $38.01_{\pm 0.90}$ \\
    {Co-teaching} \cite{han2018co} & $83.08_{\pm 0.52}$ & $80.58_{\pm 0.36}$ & $\mathbf{81.30}_{\pm 0.82}$ & $53.10_{\pm 0.98}$ & $47.25_{\pm 0.82}$ & ${44.11}_{\pm 0.31}$ \\
    {Co-teaching+} \cite{yu2019does} & $81.17_{\pm 0.55}$ & $78.23_{\pm 0.43}$ & ${71.84}_{\pm 1.13}$ & $53.10_{\pm 0.64}$ & $47.92_{\pm 0.76}$ & $41.33_{\pm 0.81}$ \\
    {CoDis \cite{xia2023combating}}& $85.33_{\pm 0.39}$ & $\mathbf{82.02}_{\pm 0.41}$ & ${78.67}_{\pm 0.46}$ & $\mathbf{58.66}_{\pm 0.44}$ &$ 52.27_{\pm 0.64}$ & $46.12_{\pm 0.66}$\\
    {LogitClip \cite{wei2023mitigating}} & $\mathbf{85.39}_{\pm 0.35}$ & $80.87_{\pm 0.42}$ & $75.36_{\pm 0.79}$ & $57.79_{\pm 0.77}$ & $\mathbf{53.14}_{\pm 0.37}$ & $\mathbf{49.00}_{\pm 0.35}$\\
    DoctorNet \cite{guan2018said} & $81.85_{\pm 0.41}$ & $78.69_{\pm 0.75}$ & $76.26_{\pm 1.28}$ & $52.61_{\pm 0.70}$ & $47.80_{\pm 0.86}$ & $43.50_{\pm 0.53}$ \\
    MBEM \cite{khetan2017learning}  & $82.37_{\pm 0.77}$ & $78.05_{\pm 0.83}$ & $71.43_{\pm 2.43}$ & $52.20_{\pm 0.07}$ & $45.26_{\pm 0.50}$ & $38.92_{\pm 0.69}$  \\
    CrowdLayer \cite{rodrigues2018deep} & $83.98_{\pm 0.35}$ & $77.76_{\pm 1.06}$ & $67.77_{\pm 1.69}$ & $51.28_{\pm 0.64}$ & $45.28_{\pm 0.64}$ & $38.93_{\pm 0.76}$  \\
    {TraceReg} \cite{tanno2019learning} & $80.72_{\pm 0.79}$ & $77.71_{\pm 1.36}$ & $67.86_{\pm 1.77}$ & $51.43_{\pm 0.61}$ & $45.08_{\pm 0.57}$ & $38.69_{\pm 1.01}$  \\
    Max-MIG \cite{cao2019max} & $81.00_{\pm 0.72}$ & $75.90_{\pm 0.52}$ & $70.96_{\pm 0.96}$ & $51.76_{\pm 1.11}$ & $44.93_{\pm 0.71}$ & $38.70_{\pm 0.49}$ \\
    CoNAL \cite{chu2021learning} & $81.60_{\pm 0.82}$ & $76.02_{\pm 0.79}$ & $69.50_{\pm 1.89}$ & $51.61_{\pm 1.14}$ & $44.19_{\pm 0.62}$ & $38.24_{\pm 0.29}$  \\
    {CCC \cite{zhang2024coupled}} & $84.81_{\pm 0.89}$ & $81.29_{\pm 0.66}$ & $77.28_{\pm 1.05}$ & $56.65_{\pm 0.55}$ & $50.68_{\pm 0.40}$ & $43.94_{\pm 0.95}$ \\
    \cmidrule(lr){1-7}
    Ours (AdaptCDRP) & $\mathbf{88.09}_{\pm 0.37}$ & $\mathbf{87.37}_{\pm 0.29}$ & $\mathbf{86.62}_{\pm 0.45}$ & $\mathbf{60.20}_{\pm 0.15}$ & $\mathbf{56.65}_{\pm 1.03}$ & $\mathbf{54.24}_{\pm 0.99}$  \\
    \bottomrule
  \end{tabular}
\end{table}

\begin{figure}[H]
\centering 
\subfigure[IDN-LOW]{
\label{Fig.ACC_cifar10_sub.1}
\includegraphics[width=0.323\textwidth]{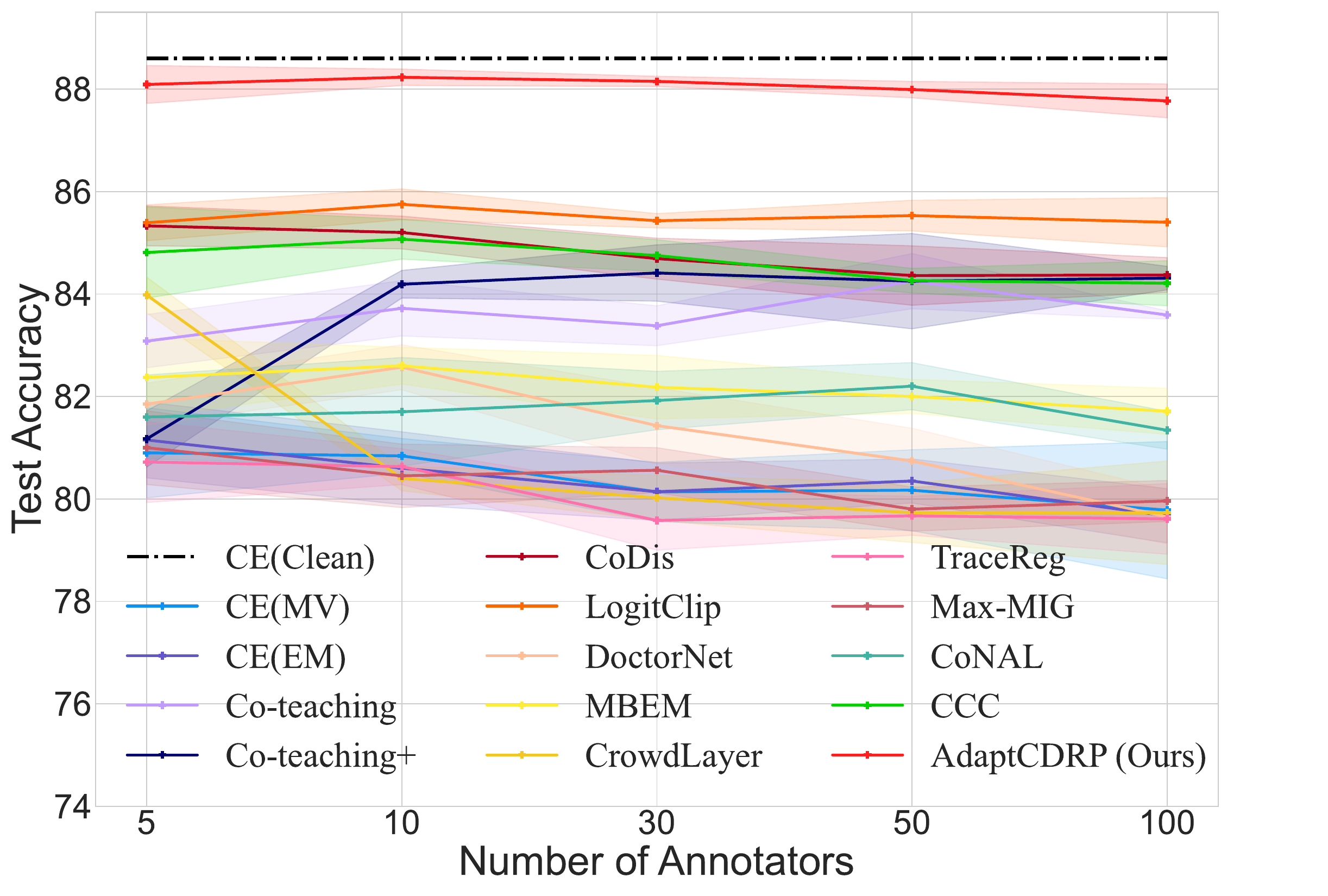}}
\subfigure[IDN-MID]{
\label{Fig.ACC_cifar10_sub.2}
\includegraphics[width=0.323\textwidth]{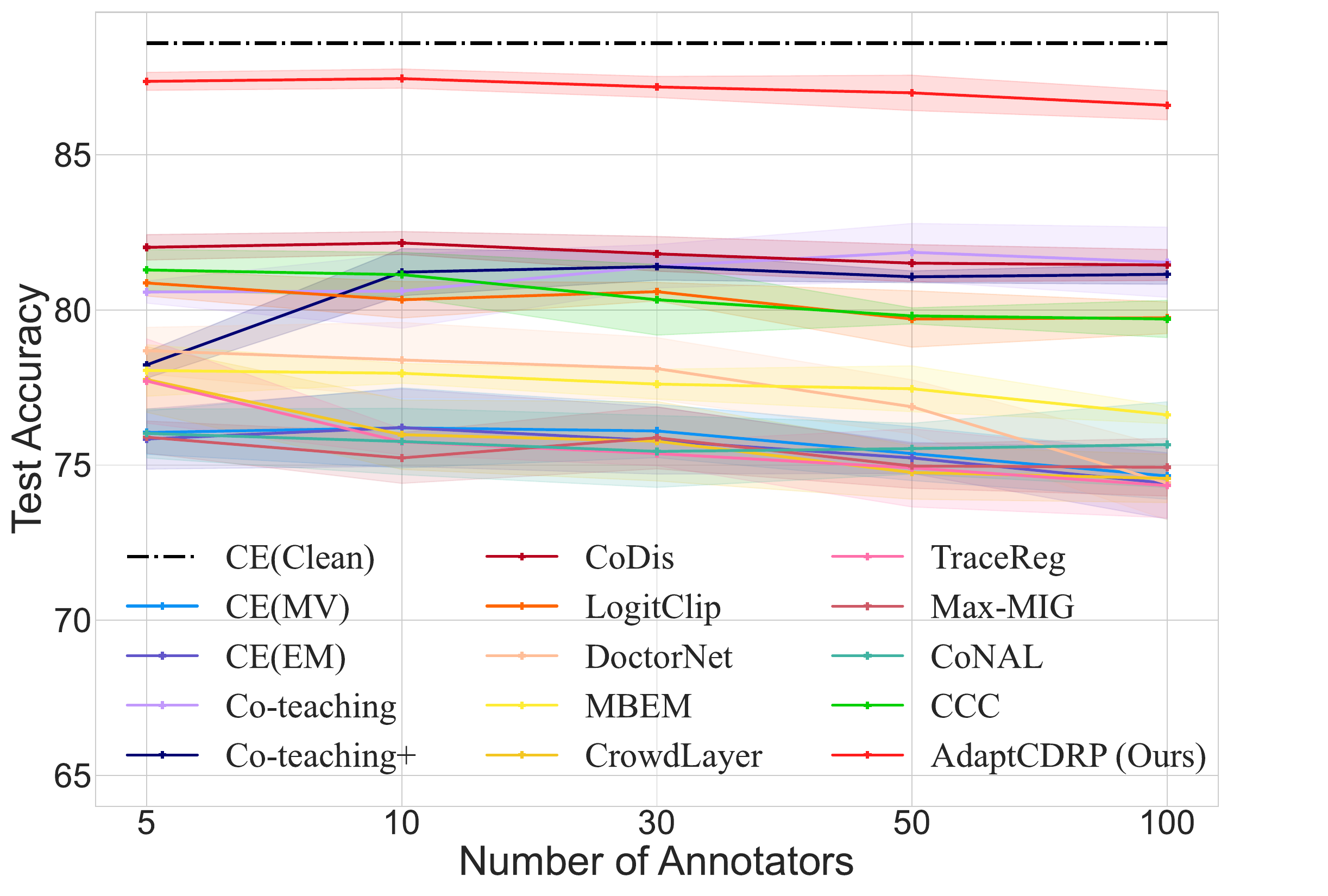}}
\subfigure[IDN-HIGH]{
\label{Fig.ACC_cifar10_sub.3}
\includegraphics[width=0.323\textwidth]{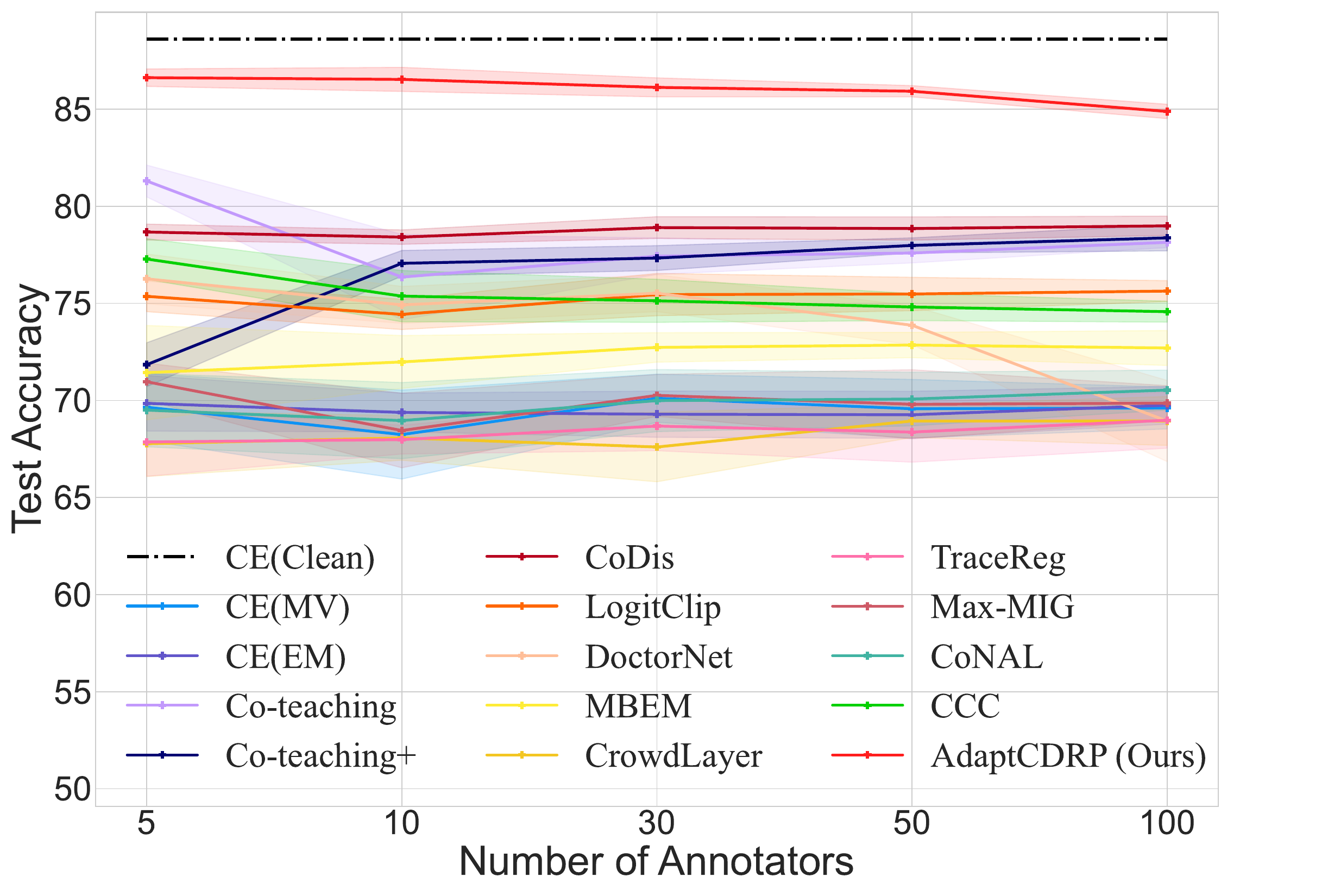}}
\caption{Average test accuracy on the CIFAR-10 dataset with varying numbers of annotators. The shaded areas are constructed using the associated standard deviations.}
\label{Fig.ACC_cifar10}
\end{figure}

\begin{table}\footnotesize\fontsize{8.0}{0.0}
\begin{minipage}{0.63\linewidth}
\renewcommand\arraystretch{0.9}
  \setlength\tabcolsep{1.0pt}
  \caption{Average accuracies (with associated standard errors expressed after the $\pm$ signs) for learning the CIFAR-10N, CIFAR-100N, LabelMe, and Animal-10N datasets.}
  \label{table_real}
  \centering
  \begin{tabular}{ccccc}
    \toprule 
    Method & CIFAR-10N & CIFAR-100N & LabelMe & Animal-10N\\
    \midrule
    CE (MV) & $82.82_{\pm 0.05}$ & $46.26_{\pm 0.81}$ & $79.49_{\pm 0.48}$ & $79.88_{\pm 0.38}$ \\
    CE (EM) \cite{dawid1979maximum} &  $83.14_{\pm 0.80}$ & $46.14_{\pm 0.69}$ & $80.64_{\pm 0.55}$  & $80.18_{\pm 0.34}$ \\
    {Co-teaching} \cite{han2018co} & $85.66_{\pm 0.54}$ &  $52.34_{\pm 0.31}$ & $79.71_{\pm 0.55}$ &  $\mathbf{81.96}_{\pm 0.48}$\\
    {Co-teaching+} \cite{yu2019does} & $82.25_{\pm 0.21}$ & $50.52_{\pm 0.40}$ & $81.55_{\pm 0.92}$ & $81.24_{\pm 0.22}$\\
    {CoDis \cite{xia2023combating}}& $\mathbf{87.23}_{\pm 0.45}$ & $\mathbf{52.66}_{\pm 0.44}$ & $81.85_{\pm 0.49}$ & $73.08_{\pm 0.35}$ \\
    {LogitClip \cite{wei2023mitigating}} & $86.37_{\pm 0.43}$ & $51.50_{\pm 0.58}$ & $81.75_{\pm 0.90}$ & $70.89_{\pm 0.65}$ \\
    DoctorNet \cite{guan2018said} & $84.52_{\pm 0.69}$ & $46.21_{\pm 0.81}$ & $79.09_{\pm 0.40}$ & $79.96_{\pm 0.55}$ \\
    MBEM \cite{khetan2017learning}  & $85.49_{\pm 0.43}$ & $46.74_{\pm 0.69}$ & $80.10_{\pm 1.09}$ & $76.96_{\pm 3.17}$ \\
    CrowdLayer \cite{rodrigues2018deep} & $82.84_{\pm 0.24}$ & $47.43_{\pm 0.59}$ & $82.95_{\pm 0.21}$ & $79.70_{\pm 0.35}$ \\
    {TraceReg} \cite{tanno2019learning} & $82.94_{\pm 0.27}$ & $47.71_{\pm 0.70}$ & $83.10_{\pm 0.15}$ & $80.34_{\pm 0.66}$ \\
    Max-MIG \cite{cao2019max} & $85.12_{\pm 0.36}$ & $46.56_{\pm 0.64}$ & $\mathbf{83.25}_{\pm 0.26}$ & $79.78_{\pm 0.80}$ \\
    CoNAL \cite{chu2021learning} & $83.01_{\pm 0.21}$ & $49.37_{\pm 0.48}$ & $82.96_{\pm 0.30}$ & $80.45_{\pm 0.49}$ \\
    {CCC \cite{zhang2024coupled}} & $86.45_{\pm 0.53}$ & $48.57_{\pm 0.58}$ & $83.18_{\pm 0.38}$ & $78.36_{\pm 0.35}$ \\
    \cmidrule(lr){1-5}
    Ours (AdaptCDRP) & $\mathbf{88.25}_{\pm 0.34}$ & $\mathbf{53.42}_{\pm 0.64}$ & $\mathbf{83.36}_{\pm 0.68}$ & $\mathbf{83.08}_{\pm 0.39}$ \\
    \bottomrule
  \end{tabular}
\end{minipage}\hfill
\begin{minipage}{0.32\linewidth}
    \vspace{0.5cm}
    \centering
    \includegraphics[width=0.96\textwidth]{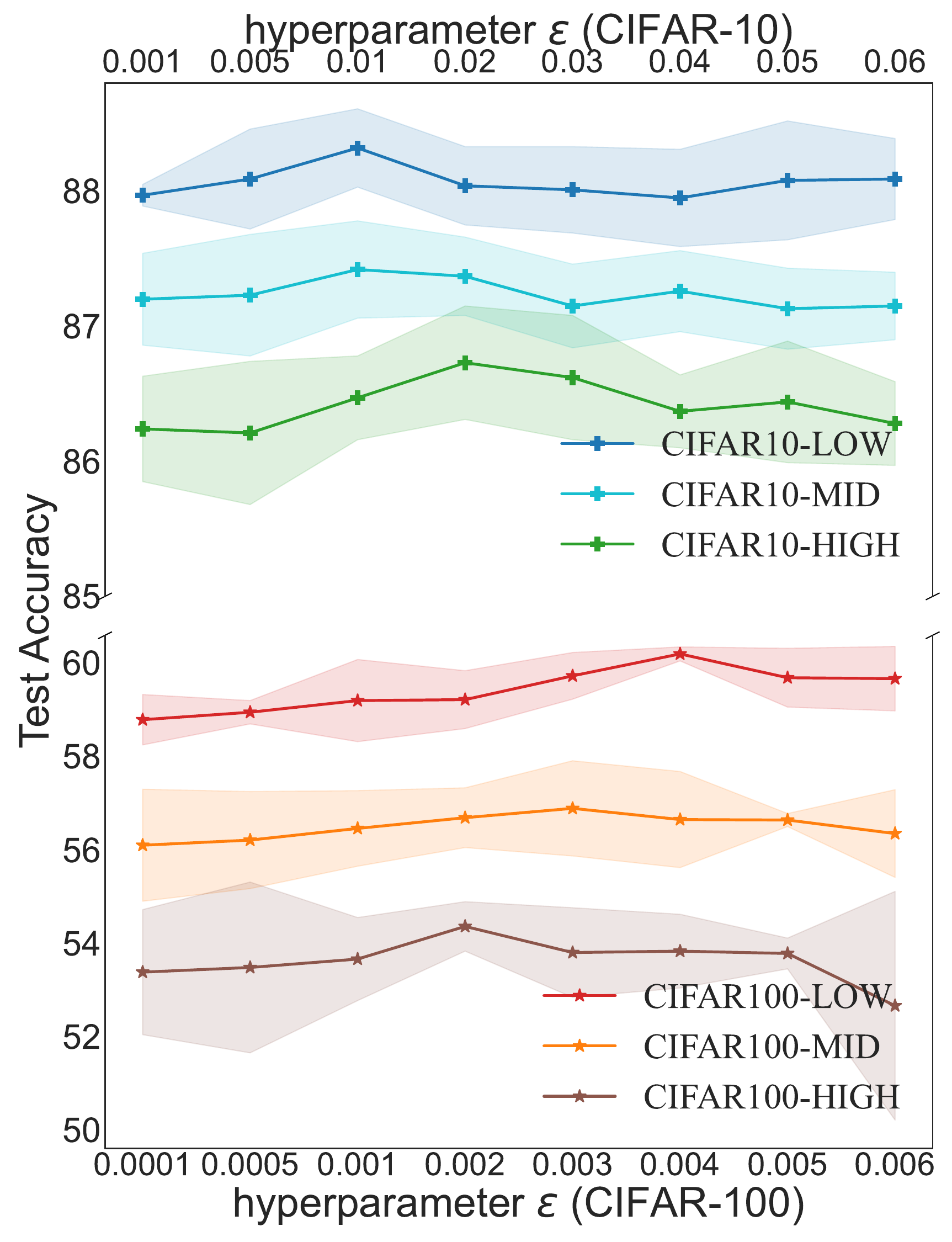}
    \captionof{figure}{\small Average accuracy on the CIFAR-10 and CIFAR-100 datasets ($R=5$) for different $\epsilon$ values.}
    \label{figure_gamma}
\end{minipage}
\end{table}

\section{Conclusion}

In this paper, we address the challenge of learning from noisy annotations by estimating true label posteriors using the CDRO framework. We formulate the problem as minimizing the worst-case risk within a distance-based ambiguity set, which constrains the conditional distributional uncertainty around a reference distribution. By deriving the dual form of the worst-case risk and finding the analytical solution to the robust risk minimization problem for each data point, we propose a novel approach for determining robust pseudo-labels using the likelihood ratio test. This approach further leads to the construction of a pseudo-empirical distribution that serves as a robust reference probability distribution in CDRO. We also derive a closed-form expression of the empirical robust risk and identify the optimal Lagrange multiplier for the dual problem. This leads to a guideline for balancing robustness and model fitting in a principled way and inspires an efficient one-step update method for the Lagrange multiplier. 

\paragraph{Limitations and Extensions.}
Our development here does not focus on precisely estimating the noise transition matrix or the true label posterior. Further research may be conducted to address the sparse annotation problem and improve estimates of the true label posterior. This can be accomplished through several approaches: (1) employing regularization techniques to mitigate the impact of small sample sizes by smoothing estimates and reducing sensitivity to outliers; (2) leveraging subgroup structures among annotators to capture additional nuances; and (3) directly modeling the true label posterior by integrating both data and noisy label information, moving beyond the limitation of purely applying Bayes’s rule.

\section*{Acknowledgements}
Yi is the Canada Research Chair in Data Science (Tier 1). Her research was supported by the Canada Research Chairs Program and the Natural Sciences and Engineering Research Council of Canada (NSERC).





{
\small
\bibliographystyle{unsrt}
\bibliography{main}

\begin{thebibliography}{10}

\bibitem{goodfellow2016deep}
Ian Goodfellow, Yoshua Bengio, and Aaron Courville.
\newblock {\em Deep Learning}.
\newblock MIT press, 2016.

\bibitem{schmidhuber2015deep}
J{\"u}rgen Schmidhuber.
\newblock Deep learning in neural networks: An overview.
\newblock {\em Neural Networks}, 61:85--117, 2015.

\bibitem{arpit2017closer}
Devansh Arpit, Stanis{\l}aw Jastrz{\k{e}}bski, Nicolas Ballas, David Krueger, Emmanuel Bengio, Maxinder~S Kanwal, Tegan Maharaj, Asja Fischer, Aaron Courville, Yoshua Bengio, et~al.
\newblock A closer look at memorization in deep networks.
\newblock In {\em Proceedings of the 34th International Conference on Machine Learning}, volume~70, pages 233--242, 2017.

\bibitem{carroll2006measurement}
Raymond~J Carroll, David Ruppert, Leonard~A Stefanski, and Ciprian~M Crainiceanu.
\newblock {\em Measurement Error in Nonlinear Models: A Modern Perspective}.
\newblock Chapman and Hall/CRC, 2006.

\bibitem{grace2016statistical}
Grace~Y Yi.
\newblock {\em Statistical Analysis with Measurement Error or Misclassification}.
\newblock Springer, 2017.

\bibitem{grace2021handbook}
Grace~Y Yi, Aurore Delaigle, and Paul Gustafson.
\newblock {\em Handbook of Measurement Error Models}.
\newblock CRC Press, 2021.

\bibitem{dawid1979maximum}
Alexander~Philip Dawid and Allan~M Skene.
\newblock Maximum likelihood estimation of observer error-rates using the em algorithm.
\newblock {\em Journal of the Royal Statistical Society: Series C}, 28(1):20--28, 1979.

\bibitem{whitehill2009whose}
Jacob Whitehill, Ting-fan Wu, Jacob Bergsma, Javier Movellan, and Paul Ruvolo.
\newblock Whose vote should count more: Optimal integration of labels from labelers of unknown expertise.
\newblock In {\em Advances in Neural Information Processing Systems}, volume~22, pages 2035--2043, 2009.

\bibitem{khetan2017learning}
Ashish Khetan, Zachary~C Lipton, and Anima Anandkumar.
\newblock Learning from noisy singly-labeled data.
\newblock {\em arXiv preprint arXiv:1712.04577}, 2017.

\bibitem{guo2023label}
Hui Guo, Boyu Wang, and Grace~Y Yi.
\newblock Label correction of crowdsourced noisy annotations with an instance-dependent noise transition model.
\newblock In {\em Advances in Neural Information Processing Systems}, volume~36, pages 347--386, 2023.

\bibitem{cao2019max}
Peng Cao, Yilun Xu, Yuqing Kong, and Yizhou Wang.
\newblock Max-mig: an information theoretic approach for joint learning from crowds.
\newblock {\em arXiv preprint arXiv:1905.13436}, 2019.

\bibitem{tanno2019learning}
Ryutaro Tanno, Ardavan Saeedi, Swami Sankaranarayanan, Daniel~C Alexander, and Nathan Silberman.
\newblock Learning from noisy labels by regularized estimation of annotator confusion.
\newblock In {\em Proceedings of the IEEE/CVF Conference on Computer Vision and Pattern Recognition}, pages 11244--11253, 2019.

\bibitem{shwartz2022pre}
Ravid Shwartz-Ziv, Micah Goldblum, Hossein Souri, Sanyam Kapoor, Chen Zhu, Yann LeCun, and Andrew~G Wilson.
\newblock Pre-train your loss: Easy bayesian transfer learning with informative priors.
\newblock In {\em Advances in Neural Information Processing Systems}, volume~35, pages 27706--27715, 2022.

\bibitem{xia2019anchor}
Xiaobo Xia, Tongliang Liu, Nannan Wang, Bo~Han, Chen Gong, Gang Niu, and Masashi Sugiyama.
\newblock Are anchor points really indispensable in label-noise learning?
\newblock In {\em Advances in Neural Information Processing Systems}, volume~32, pages 6838--6849, 2019.

\bibitem{husain2022adversarial}
Hisham Husain and Jeremias Knoblauch.
\newblock Adversarial interpretation of bayesian inference.
\newblock In {\em Proceedings of The 33rd International Conference on Algorithmic Learning Theory}, volume 167, pages 553--572. Proceedings of Machine Learning Research, 2022.

\bibitem{shapiro2023conditional}
Alexander Shapiro and Alois Pichler.
\newblock Conditional distributionally robust functionals.
\newblock {\em Operations Research}, 2023.

\bibitem{blanchet2019quantifying}
Jose Blanchet and Karthyek Murthy.
\newblock Quantifying distributional model risk via optimal transport.
\newblock {\em Mathematics of Operations Research}, 44(2):565--600, 2019.

\bibitem{gao2023distributionally}
Rui Gao and Anton Kleywegt.
\newblock Distributionally robust stochastic optimization with {W}asserstein distance.
\newblock {\em Mathematics of Operations Research}, 48(2):603--655, 2023.

\bibitem{tanaka2018joint}
Daiki Tanaka, Daiki Ikami, Toshihiko Yamasaki, and Kiyoharu Aizawa.
\newblock Joint optimization framework for learning with noisy labels.
\newblock In {\em Proceedings of the IEEE Conference on Computer Vision and Pattern Recognition}, pages 5552--5560, 2018.

\bibitem{cover1999elements}
Thomas~M. Cover and Joy~A. Thomas.
\newblock {\em Elements of Information Theory}.
\newblock Wiley-Interscience, 2006.

\bibitem{krizhevsky2009learning}
Alex Krizhevsky.
\newblock Learning multiple layers of features from tiny images.
\newblock Technical report, University of Toronto, 2009.

\bibitem{wei2022learning}
Jiaheng Wei, Zhaowei Zhu, Hao Cheng, Tongliang Liu, Gang Niu, and Yang Liu.
\newblock Learning with noisy labels revisited: A study using real-world human annotations.
\newblock In {\em International Conference on Learning Representations}, 2022.

\bibitem{rodrigues2017learning}
Filipe Rodrigues, Mariana Lourenco, Bernardete Ribeiro, and Francisco~C Pereira.
\newblock Learning supervised topic models for classification and regression from crowds.
\newblock {\em IEEE Transactions on Pattern Analysis and Machine Intelligence}, 39(12):2409--2422, 2017.

\bibitem{torralba2010labelme}
Antonio Torralba, Bryan~C. Russell, and Jenny Yuen.
\newblock Labelme: Online image annotation and applications.
\newblock {\em Proceedings of the IEEE}, 98(8):1467--1484, 2010.

\bibitem{song2019selfie}
Hwanjun Song, Minseok Kim, and Jae-Gil Lee.
\newblock {SELFIE}: Refurbishing unclean samples for robust deep learning.
\newblock In {\em Proceedings of the 36th International Conference on Machine Learning}, volume~97, pages 5907--5915, 2019.

\bibitem{he2016deep}
Kaiming He, Xiangyu Zhang, Shaoqing Ren, and Jian Sun.
\newblock Deep residual learning for image recognition.
\newblock In {\em Proceedings of the IEEE Conference on Computer Vision and Pattern Recognition}, pages 770--778, 2016.

\bibitem{rodrigues2018deep}
Filipe Rodrigues and Francisco Pereira.
\newblock Deep learning from crowds.
\newblock In {\em Proceedings of the AAAI Conference on Artificial Intelligence}, volume~32, pages 1611--1618, 2018.

\bibitem{simonyan2014very}
Karen Simonyan and Andrew Zisserman.
\newblock Very deep convolutional networks for large-scale image recognition.
\newblock In {\em International Conference on Learning Representations}, 2015.

\bibitem{xia2020part}
Xiaobo Xia, Tongliang Liu, Bo~Han, Nannan Wang, Mingming Gong, Haifeng Liu, Gang Niu, Dacheng Tao, and Masashi Sugiyama.
\newblock Part-dependent label noise: Towards instance-dependent label noise.
\newblock In {\em Advances in Neural Information Processing Systems}, volume~33, pages 7597--7610, 2020.

\bibitem{han2018co}
Bo~Han, Quanming Yao, Xingrui Yu, Gang Niu, Miao Xu, Weihua Hu, Ivor Tsang, and Masashi Sugiyama.
\newblock Co-teaching: Robust training of deep neural networks with extremely noisy labels.
\newblock In {\em Advances in Neural Information Processing Systems}, volume~31, pages 8527--8537, 2018.

\bibitem{yu2019does}
Xingrui Yu, Bo~Han, Jiangchao Yao, Gang Niu, Ivor Tsang, and Masashi Sugiyama.
\newblock How does disagreement help generalization against label corruption?
\newblock In {\em Proceedings of the 36th International Conference on Machine Learning}, volume~97, pages 7164--7173, 2019.

\bibitem{xia2023combating}
Xiaobo Xia, Bo~Han, Yibing Zhan, Jun Yu, Mingming Gong, Chen Gong, and Tongliang Liu.
\newblock Combating noisy labels with sample selection by mining high-discrepancy examples.
\newblock In {\em Proceedings of the IEEE/CVF International Conference on Computer Vision}, pages 1833--1843, 2023.

\bibitem{wei2023mitigating}
Hongxin Wei, Huiping Zhuang, Renchunzi Xie, Lei Feng, Gang Niu, Bo~An, and Yixuan Li.
\newblock Mitigating memorization of noisy labels by clipping the model prediction.
\newblock In {\em Proceedings of the 40th International Conference on Machine Learning}, volume 202, pages 36868--36886, 2023.

\bibitem{guan2018said}
Melody Guan, Varun Gulshan, Andrew Dai, and Geoffrey Hinton.
\newblock Who said what: Modeling individual labelers improves classification.
\newblock In {\em Proceedings of the AAAI conference on artificial intelligence}, volume~32, pages 3109--3118, 2018.

\bibitem{chu2021learning}
Zhendong Chu, Jing Ma, and Hongning Wang.
\newblock Learning from crowds by modeling common confusions.
\newblock In {\em Proceedings of the AAAI Conference on Artificial Intelligence}, volume~35, pages 5832--5840, 2021.

\bibitem{zhang2024coupled}
Hansong Zhang, Shikun Li, Dan Zeng, Chenggang Yan, and Shiming Ge.
\newblock Coupled confusion correction: Learning from crowds with sparse annotations.
\newblock In {\em Proceedings of the AAAI Conference on Artificial Intelligence}, volume~38, pages 16732--16740, 2024.

\bibitem{luenberger1984linear}
David~G Luenberger and Yinyu Ye.
\newblock {\em Linear and Nonlinear Programming}, volume~2.
\newblock Springer, 1984.

\bibitem{bradley1977applied}
Stephen~P Bradley, Arnoldo~C Hax, and Thomas~L Magnanti.
\newblock {\em Applied Mathematical Programming}.
\newblock Addison-Wesley Publishing Company, 1977.

\bibitem{wainwright2019high}
Martin~J Wainwright.
\newblock {\em High-Dimensional Statistics: A Non-Asymptotic Viewpoint}.
\newblock Cambridge University Press, 2019.

\bibitem{lee2018minimax}
Jaeho Lee and Maxim Raginsky.
\newblock Minimax statistical learning with wasserstein distances.
\newblock In {\em Advances in Neural Information Processing Systems}, volume~31, pages 2687--2696, 2018.

\bibitem{huntsupremum}
John Hunter.
\newblock The supremum and infimum.
\newblock \url{https://www.math.ucdavis.edu/~hunter/m125b/ch2.pdf}.

\bibitem{mohri2018foundations}
Mehryar Mohri, Afshin Rostamizadeh, and Ameet Talwalkar.
\newblock {\em Foundations of Machine Learning}.
\newblock MIT Press, 2018.

\bibitem{kingma2014adam}
Diederik~P Kingma and Jimmy Ba.
\newblock Adam: A method for stochastic optimization.
\newblock {\em arXiv preprint arXiv:1412.6980}, 2014.

\bibitem{li2014error}
Hongwei Li and Bin Yu.
\newblock Error rate bounds and iterative weighted majority voting for crowdsourcing.
\newblock {\em arXiv preprint arXiv:1411.4086}, 2014.

\bibitem{bucarelli2023leveraging}
Maria~Sofia Bucarelli, Lucas Cassano, Federico Siciliano, Amin Mantrach, and Fabrizio Silvestri.
\newblock Leveraging inter-rater agreement for classification in the presence of noisy labels.
\newblock In {\em Proceedings of the IEEE/CVF Conference on Computer Vision and Pattern Recognition}, pages 3439--3448, 2023.

\bibitem{zheng2020error}
Songzhu Zheng, Pengxiang Wu, Aman Goswami, Mayank Goswami, Dimitris Metaxas, and Chao Chen.
\newblock Error-bounded correction of noisy labels.
\newblock In {\em Proceedings of the 37th International Conference on Machine Learning}, volume 119, pages 11447--11457, 2020.

\bibitem{ibrahimdeep}
Shahana Ibrahim, Tri Nguyen, and Xiao Fu.
\newblock Deep learning from crowdsourced labels: Coupled cross-entropy minimization, identifiability, and regularization.
\newblock In {\em The Eleventh International Conference on Learning Representations}, 2023.

\end{thebibliography}
}

\newpage
\appendix



\begin{center}
{\large\bf SUPPLEMENTARY MATERIAL}
\end{center}

{
\renewcommand{\theequation}{\thesection\arabic{equation}}

\setcounter{equation}{0}



\section{Technical Details}

\subsection{Preliminaries about Linear Programming and Concentration Bounds}

A \emph{linear program} (LP) is an optimization problem of the form
\begin{equation}\label{appdx.LP_def_1}
\begin{split}
    \max_{\mathbf{x}\in\mathbb{R}^n}&\ \mathbf{c}^\top \mathbf{x}\\
    s.t.&\ \ \mathbf{A}\mathbf{x}\le\mathbf{b}\\
    &\ \ \mathbf{x}\ge 0,
\end{split}
\end{equation}
where $\mathbf{c}\in\mathbb{R}^{n}$ and $\mathbf{b}\in\mathbb{R}^{m}$ are given, and $\mathbf{A}$ is a specified $m\times n$ matrix. Here, ``$\le$'' represents elementwise inequality for vectors. The expression $\mathbf{c}^\top \mathbf{x}$ is called the \emph{objective function}, and the set $\{\mathbf{x}\in\mathbb{R}^{n}: \mathbf{A}\mathbf{x}\le\mathbf{b}, \mathbf{x}\ge 0\}$ defines the \emph{feasible region} of the linear program. By introducing slack variables, any linear program can be converted to the following \emph{standard form}:
\begin{equation}\label{appdx.LP_def_2}
\begin{split}
    \max_{\mathbf{x}\in\mathbb{R}^n}&\ \mathbf{c}^\top \mathbf{x}\\
    s.t.&\ \ \mathbf{A}\mathbf{x}=\mathbf{b}\\
    &\ \ \mathbf{x}\ge 0.
\end{split}
\end{equation}

We begin by introducing the concept of extreme point of related properties.  
\begin{definition}[{\cite[Chapter~2]{luenberger1984linear}}]\label{def.extreme}
   A point $\mathbf{z}$ in a convex set $\Theta$ is called an \emph{extreme point} of $\Theta$ if there do not exist two distinct points $\mathbf{z}', \mathbf{z}''\in\Theta$ and a scalar $\nu$ with $0<\nu<1$ such that $\mathbf{z}=\nu \mathbf{z}'+(1-\nu)\mathbf{z}''$. 
\end{definition}

The following lemma shows that optimal solutions of a linear program are located among the extreme points.
\begin{lemma}[{\cite[Chapter 2]{luenberger1984linear}}]\label{lemma.extreme}
If a linear programming problem has a finite optimal solution (i.e., a feasible solution that optimizes the objective function), then there is a finite optimal solution that is an extreme point of the constraint set.
\end{lemma}

The following lemma on strong duality in linear programming will be used in the proof of Proposition \ref{thm.dual} and Remark \ref{remark.dual_emp}.
\begin{lemma}[Duality Theorem of Linear Programming; {\cite[Chapter 4]{bradley1977applied}}]\label{lemma.dual_linear}
    Let $\mathbf{c}=(c_1,\ldots,c_n)^\top$ and $\mathbf{b}=(b_1,\ldots,b_m)^\top$ be given vectors, and let $\mathbf{A}=[a_{ij}]$ be a given $m\times n$ matrix with $a_{ij}$ being its $(i, j)$ element. Define the \textbf{primal} problem as:
    \begin{equation*}
        \left\{
        \begin{aligned}
            &\hspace{-0.4cm}\max_{x_1,\ldots,x_n\in\mathbb{R}}\ \mathsf{P},\ \text{with}\ \mathsf{P}\triangleq\sum_{j=1}^{n} c_j x_j,\\
            s.t.\ &\sum_{j=1}^{n}a_{ij}x_j\le b_i\ \ \text{for}\ \ i\in[m],\\
           &x_j\ge 0\ \ \text{for}\ \ j\in[n],\ 
        \end{aligned}
        \right.
    \end{equation*}
    or equivalently written in compact form:
    \begin{equation*}
        \left\{
        \begin{aligned}
            &\hspace{-0.1cm}\max_{\mathbf{x}\in\mathbb{R}^n}\ \mathsf{P},\ \text{with}\ \mathsf{P}\triangleq\mathbf{c}^\top\mathbf{x}\ \text{and}\ \mathbf{x}=(x_1,\ldots,x_n)^\top,\\
            s.t.\ &\mathbf{Ax}\ge\mathbf{b},\\
            &\mathbf{x}\ge0.
        \end{aligned}
        \right.
    \end{equation*}
    The corresponding \textbf{dual} linear problem is:
    \begin{equation*}
        \left\{
        \begin{aligned}
            &\hspace{-0.4cm}\min_{y_1,\ldots,y_m\in\mathbb{R}}\ \mathsf{D},\ \text{with}\ \mathsf{D}\triangleq\sum_{i=1}^{m} b_i y_i,\\
            s.t.\ &\sum_{i=1}^{m}a_{ij}y_i\ge c_j\ \ \text{for}\ \ j\in[n],\\
            &y_i\ge 0\ \ \text{for}\ \ i\in[m],
        \end{aligned}
        \right.
    \end{equation*}
    or equivalently,
    \begin{equation*}
        \left\{
        \begin{aligned}
            &\hspace{-0.15cm}\min_{\mathbf{y}\in\mathbb{R}^m}\ \mathsf{D},\ \text{with}\ \mathsf{D}\triangleq\mathbf{y}^\top\mathbf{b}\ \text{and}\ \mathbf{y}=(y_1,\ldots, y_m)^\top,\\
            s.t.\ &\mathbf{y}^\top\mathbf{A}\le\mathbf{c}^\top,\\
            &\mathbf{y}\ge0.
        \end{aligned}
        \right.
    \end{equation*}
    If the primal (or dual) problem has a finite optimal solution, then the dual (or primal) problem also has a finite solution, and the optimal values of the primal and dual problems are equal.
\end{lemma}

Next we introduce a concentration bound along with associated concepts, which will be used in the proof of Theorem \ref{thm.bias}.

Let $\Omega$ denote a subset of $\mathbb{R}$ and $f:\Omega^n\rightarrow\mathbb{R}$. We say that a function $f$ satisfies the \emph{bounded difference inequality} \cite{wainwright2019high} with parameters $\{L_1,\ldots,L_n\}$ if for any $k\in[n]$,
\begin{align}\label{appdx.bounded_diff_def}
    \sup_{s_1,\ldots,s_n,s'_k\in\Omega}|f(s_1,\ldots,s_k,\ldots,s_n)-f(s_1,\ldots,s'_k,\ldots,s_n)|\le L_k.
\end{align}
That is, for any $k\in[n]$, if we substitute $s_k$ with $s'_k$ while keeping other $s_j$ fixed for all $j\ne k$, the function $f$ changes by at most $L_k$.



\begin{lemma}[Bounded Differences Inequality; {\cite[Corollary 2.21]{wainwright2019high}}]\label{lemma.concentration}
    Let $\boldsymbol{S}=(S_1,\ldots,S_n)^\top$ represent a random vector with independent components defined on a sample space $\Omega^n$, and $f(\boldsymbol{S})\triangleq f(S_1,\ldots,S_n)$ for a function $f:\Omega^n\rightarrow\mathbb{R}$. Suppose that $f$ satisfies the bounded difference property with parameters $\{L_1,\ldots,L_n\}$. Then, for any $t\ge 0$,
    \begin{align*}
        &\mathbb{P}\left[f(\boldsymbol{S})-\mathbb{E}\left\{f(\boldsymbol{S})\right\}\ge t\right]\le \exp\left\{-\frac{2t^2}{\sum_{i=1}^{n}L^2_i}\right\};\\
        &\mathbb{P}\left[f(\boldsymbol{S})-\mathbb{E}\left\{f(\boldsymbol{S})\right\}\le -t\right]\le \exp\left\{-\frac{2t^2}{\sum_{i=1}^{n}L^2_i}\right\}.
    \end{align*}
\end{lemma}

We also introduce the following definitions and lemmas, which will be used to characterize the complexity of the function class.

\begin{definition}[Covering number; {\cite[Definition 5.1]{wainwright2019high}}]\label{def.covering}
Let $\Theta$ denote a set and $\rho$ a metric on $\Theta$. For $t>0$, a \emph{$t$-cover} of $\Theta$ with respect to $\rho$ is a set $\{\theta_i\in\Theta:i=1,\ldots,N\}$ such that for each $\theta\in\Theta$, there exists some $i\in[N]$ such that $\rho(\theta,\theta_i)\le t$. The $t$-\emph{covering number} $N(t;\Theta,\rho)$ is the cardinality of the smallest $t$-cover.
\end{definition}

\begin{lemma}\label{lemma.covering_cartesian}
Let $\Theta_j$ be a set equipped with a metric $\rho_j$ for $j=1,2$, and define $\Theta=\Theta_1\times \Theta_2$. Given $\alpha_1,\alpha_2>0$, define the metric $\rho$ on $\Theta$ as $\rho(\theta,\theta')\triangleq\alpha_1\rho_1(\theta^1,\theta^{1'})+\alpha_2\rho_2(\theta^2,\theta^{2'})$ for any $\theta\triangleq(\theta^1,\theta^2)\in\Theta$ and $\theta'\triangleq(\theta^{1'},\theta^{2'})\in\Theta$. Then for $t>0$,
\begin{align*}
    N(t;\Theta,\rho)\le N(t/(2\alpha_1);\Theta_1,\rho_1)\times N(t/(2\alpha_2);\Theta_2,\rho_2).
\end{align*}
\end{lemma}
\begin{proof}
For $j=1,2$, let $\Bar{\Theta}_{j}\triangleq\{\theta^j_1,\ldots,\theta^j_{N_j}\}$ denote the smallest $t/(2\alpha_j)$-cover of $\Theta_j$ with respect to $\rho_j$. Then, by Definition \ref{def.covering}, $N(t/(2\alpha_j);\Theta_j,\rho_j)=N_j$ for $j=1,2$. For any $\theta=(\theta^1,\theta^2)\in\Theta$, by definition \ref{def.covering}, there exists $i_1\in[N_1]$ and $i_2\in[N_2]$ such that $\rho_1(\theta^1,\theta^1_{i_1})\le t/(2\alpha_1)$ and $\rho_2(\theta^2,\theta^2_{i_2})\le t/(2\alpha_2)$. Then, $\theta_i\triangleq(\theta^1_{i_1}, \theta^2_{i_2})\in\Bar{\Theta}_1\times \Bar{\Theta}_2\subset\Theta$, and 
\begin{align*}
    \rho(\theta, \theta_i)=\alpha_1\rho_1(\theta^1,\theta^1_{i_1})+\alpha_2\rho_2(\theta^2,\theta^2_{i_2})\le t.
\end{align*}
Hence, by Definition \ref{def.covering}, $\Bar{\Theta}_1\times \Bar{\Theta}_2$ is a $t$-cover of $\Theta$ with respect to $\rho$ and $N(t;\Theta,\rho)\le|\Bar{\Theta}_1\times \Bar{\Theta}_2|=N_1\times N_2=N(t/(2\alpha_1);\Theta_1,\rho_1)\times N(t/(2\alpha_2);\Theta_2,\rho_2)$, where $|\cdot|$ represents the cardinality of a set.
\end{proof}

\begin{lemma}\label{lemma.covering_interval}
Let $\mathcal{I}\triangleq[a,b]\subset\mathbb{R}$ denote a closed interval on $\mathbb{R}$ with $a< b$. Define the metric $\rho$ on $\mathcal{I}$ as $\rho(x, x')=|x-x'|$ for any $x,x'\in\mathcal{I}$. Then, for any $t>0$, $N(t;\mathcal{I},\rho)\le \frac{b-a}{2t}+1$ if $t<\frac{b-a}{2}$, and  $N(t;\mathcal{I},\rho)=1$ if $t\ge\frac{b-a}{2}$.
\end{lemma}
\begin{proof}
Let $n_t=\lfloor\frac{b-a}{2}\rfloor$, where $\lfloor\cdot\rfloor$ represents the floor function. To prove the first result for $b-a\ge 2t$, we construct the following subset $\overline{\mathcal{I}}$ of $\mathcal{I}$:
\begin{align*}
    \overline{\mathcal{I}}\triangleq &\Big\{a+t, a+t+2t, a+t+4t,\ldots, a+t+(n_t-1)\cdot 2t,\min(b, a+t+n_t\cdot 2t)\Big\}.
\end{align*}
Clearly, $\mathcal{I}\subset\cup_{k\in[n_t]}[a+2t(k-1),a+2tk]\cup[a+2tn_t,b]$. Next, we verify that $\overline{\mathcal{I}}$ is a $t$-cover of $\mathcal{I}$ with respect to metric $\rho$. Specifically, for any $x\in\mathcal{I}$, there exists $k\in[n_t]$ such that $x\in[a+2t\cdot(k-1), a+2t\cdot k]$, or $x\in[a+2tn_t,b]$. For the former case,
\begin{align*}
    \rho(x, a+t+2t\cdot(K-1))\le t;
\end{align*}
and for the latter case,
\begin{align*}
    \rho(x,\min(b, a+t+n_t\cdot 2t))\le t.
\end{align*}
Hence, by Definition \ref{def.covering}, $\overline{\mathcal{I}}$ is a $t$-cover of $\mathcal{I}$, therefore $N(t;\mathcal{I},\rho)\le|\overline{\mathcal{I}}|=n_t+1\le \frac{b-a}{2t}+1$. The first result is then established.

The second result for $b-a\le 2t$ follows from the fact that $\{a+\frac{b-a}{2}\}$ is a $t$-cover of $\mathcal{I}$ since $\rho(x,a+\frac{b-a}{2})\le\frac{b-a}{2}\le t$ for any $x\in\mathcal{I}. $
\end{proof}

\begin{definition}[{\cite[Definition 5.16]{wainwright2019high}}]\label{def.sub_gaussian_pro}
    A collection of zero-mean random variables $\{S_\theta: \theta\in\Theta\}$ is a \emph{sub-Gaussian} process with respect to a metric $\rho$ on $\Theta$ if, for all $\theta_1,\theta_2\in\Theta$ and $t\in\mathbb{R}$,
    \begin{align*}
        \mathbb{E}\left[\exp\left\{t(S_{\theta_1}-S_{\theta_2})\right\}\right]\le\exp\left\{\frac{t^2\rho^2(\theta_1,\theta_2)}{2}\right\}.
    \end{align*}
\end{definition}

\begin{lemma}[Dudley’s Entropy Integral Bound; modified from Theorem 5.22 of \cite{wainwright2019high}]\label{lemma.entropy}
    Let $\{S_\theta: \theta\in\Theta\}$ be a zero-mean sub-Gaussian process with respect to a metric $\rho$ on $\Theta$. Then, 
    \begin{align*}
        \mathbb{E}\left(\sup_{\theta\in\Theta}S_\theta\right)\le 8\sqrt{2}\int_{0}^{\infty}\sqrt{\log N(t; \Theta, \rho)}dt,
    \end{align*}
    where $N(t; \Theta, \rho)$ represents the $t$-covering number of $\Theta$ with respect to $\rho$.
\end{lemma}

\subsection{Proof of Proposition \ref{thm.dual}}\label{appdx.pf_thm_dual}

Strong duality can be established using Theorem 1 in \cite{gao2023distributionally}, which applies to general cases. However, by capitalizing the discrete nature of the sample space $\mathcal{Y}$, we can present a more concise result. To see this, here we provide an alternative proof of strong duality using the duality principle in finite-dimensional linear programming, as detailed below.

For every fixed $\mathbf{x}\in\mathcal{X}$, $\widetilde{\mathbf{y}}\in\mathcal{Y}^R$, $\psi\in\Psi$, by re-writing the the constraint ${Q_{\mathrm{y}|\mathbf{x},\widetilde{\mathbf{y}}}}\in\Gamma_\epsilon(P_{\mathrm{y}|\mathbf{x},\widetilde{\mathbf{y}}})$ in (\ref{eq.minimax}) using Definition \ref{def.Wasserstein}, we re-express the primal problem (\ref{eq.minimax}) as:
\begin{align*}
    \overline{\mathsf{P}}_{\epsilon}\triangleq\sup_{\Pi\in\mathcal{P}(\mathcal{Y}^2)}\Big\{\int\ell(\psi(\mathbf{x}), \mathrm{y})d\Pi(\mathrm{y},\mathrm{y}'): \int c^p(\mathrm{y},\mathrm{y}')d\Pi(\mathrm{y},\mathrm{y}')\le\epsilon^p,\ \Pi(\mathcal{Y},\cdot)=P_{\mathrm{y}|\mathbf{x},\widetilde{\mathbf{y}}}(\cdot)\Big\}.
\end{align*}
Given the discrete nature of the sample space $\mathcal{Y}$, $\overline{\mathsf{P}}_{\epsilon}$ can be reformulated as the following finite-dimensional linear program: 
\begin{align*}
    \overline{\mathsf{P}}_{\epsilon}
    =& \left\{
    \begin{aligned}
        &\hspace{-0.7cm}\max_{\pi_{jk}\in\mathbb{R}\ \text{with}\ j,k\in[K]}\ \bigg\{\sum_{j,k\in[K]} \ell(\psi(\mathbf{x}),j)\pi_{jk}\bigg\},\\
        \ s.t.\ &\sum_{j,k\in[K]}c^{p}(j,k)\pi_{jk}\le \epsilon^p, \\
        \ &\hspace{0.2cm} \sum_{j=1}^{K}\pi_{jk}=P_{\mathrm{y}|\mathbf{x},\widetilde{\mathbf{y}}}(k)\ \text{for}\ k\in[K],\\
        \ &\hspace{0.2cm} \pi_{jk}\ge 0\ \text{for}\ j,k\in[K],
    \end{aligned}
    \right.
\end{align*}
where $\pi_{jk}=\Pi(\mathrm{Y}=j,\mathrm{Y}'=k)$ for $j,k\in[K]$. By Lemma \ref{lemma.dual_linear}, each primal constraint corresponds to a dual variable. Introducing the dual variables $\gamma$ and $\tau_k$ for $k\in[K]$, the dual linear program for $\overline{\mathsf{P}}_{\epsilon}$ is then expressed as 
\begin{align}\label{appdx.pf_prop_dual}
    \overline{\mathsf{D}}_{\epsilon}
    =& \left\{
    \begin{aligned}
        &\hspace{-0.8cm}\min_{\gamma,\tau_k\in\mathbb{R}\ \text{with}\ k\in[K]} \ \Big\{\gamma\epsilon^p+\sum_{k=1}^{K}P_{\mathrm{y}|\mathbf{x},\widetilde{\mathbf{y}}}(k)\tau_{k}\Big\}, \\
        \ s.t.\ &\gamma c^{p}(j,k)+\tau_k\ge \ell(\psi(\mathbf{x}), j)\ \text{for}\ j,k\in[K], \\
        \ & \gamma\ge 0,
    \end{aligned}
    \right.
\end{align}
where first constraint can be written as $\tau_k\ge\max_{j\in[K]}\big\{\ell(\psi(\mathbf{x}), j)-\gamma c^{p}(j,k)\big\}$. Therefore, to minimize the objective function in (\ref{appdx.pf_prop_dual}), the value of the dual variable $\tau_k$ should be taken as $\max_{j\in[K]}\big\{\ell(\psi(\mathbf{x}), j)-\gamma c^{p}(j,k)\big\}$ for each $\gamma\ge 0$ and  $k\in[K]$. Hence, the proof is established.

\subsection{Proof of Remark \ref{remark.dual_emp}}\label{appdx.pf_remark_dual_emp}
As in the proof of Proposition \ref{thm.dual}, the optimization problem in (\ref{eq.re_rob_risk_emp_prim}) can be written as
\begin{align*}
    \overline{\mathsf{P}}_{\epsilon}
    \triangleq&\sup\left\{\frac{1}{n}\sum_{i=1}^{n}\left[\mathbb{E}_{Q_{\mathrm{y}|\mathbf{x}_i,\widetilde{\mathbf{y}}_i}}\left\{\ell(\psi(\mathbf{x}_i), \mathrm{Y})\right\}\right]:\frac{1}{n}\sum_{i=1}^{n}\mathscr{d}(Q_{\mathrm{y}|\mathbf{x}_i,\widetilde{\mathbf{y}}_i},P_{\mathrm{y}|\mathbf{x}_i,\widetilde{\mathbf{y}}_i})\le\epsilon\right\}\\
    =& \sup_{\Pi^{(i)}\in\mathcal{P}(\mathcal{Y}^2),i\in[n]}\bigg\{\frac{1}{n}\sum_{i=1}^{n}\int\ell(\psi(\mathbf{x}_i), \mathrm{y})d\Pi^{(i)}(\mathrm{y},\mathrm{y}'): 
    \frac{1}{n}\sum_{i=1}^{n}\int c(\mathrm{y},\mathrm{y}')d\Pi^{(i)}(\mathrm{y},\mathrm{y}')\le\epsilon,\\
    \ &\ \ \ \ \ \ \ \ \ \ \ \ \ \ \ \ \ \ \ \ \Pi^{(i)}(\mathcal{Y},\cdot)=P_{\mathrm{y}|\mathbf{x}_i,\widetilde{\mathbf{y}}_i}(\cdot)\ \text{for}\ i\in[n]\bigg\},
\end{align*}
where the first step is due to the definition of the empirical distribution $P^{(n)}_{\mathbf{x},\widetilde{\mathbf{y}}}$, and the second step holds by re-writing the constraint $\frac{1}{n}\sum_{i=1}^{n}\mathscr{d}(Q_{\mathrm{y}|\mathbf{x}_i,\widetilde{\mathbf{y}}_i},P_{\mathrm{y}|\mathbf{x}_i,\widetilde{\mathbf{y}}_i})\le\epsilon$ using the definition of the the Wasserstein distance $\mathscr{d}(\cdot,\cdot)$ given in Definition \ref{def.Wasserstein}.

$\overline{\mathsf{P}}_{\epsilon}$ can be further expressed as the finite-
dimensional linear program:
\begin{align*}
    \overline{\mathsf{P}}_{\epsilon}
    =& \left\{
    \begin{aligned}
        &\hspace{-1.1cm}\max_{\pi^{(i)}_{jk}\in\mathbb{R}\ \text{with} \ i\in[n], j,k\in[K]}\ \bigg\{\frac{1}{n}\sum_{i=1}^{n}\sum_{j,k\in[K]} \ell(\psi(\mathbf{x}_i),j)\pi^{(i)}_{jk}\bigg\},\\
        \hspace{0.5cm} s.t.&\ \ \frac{1}{n}\sum_{i=1}^{n}\sum_{j,k\in[K]}c(j,k)\pi^{(i)}_{jk}\le \epsilon, \\
        \hspace{0.5cm} &\ \sum_{j=1}^{K}\pi^{(i)}_{jk}=P_{\mathrm{y}|\mathbf{x}_i,\widetilde{\mathbf{y}}_i}(k)\ \text{for}\ i\in[n]\ \text{and}\ k\in[K],\\
        \hspace{0.5cm} &\ \pi^{(i)}_{jk}\ge 0\ \text{for}\ i\in[n]\ \text{and}\ j,k\in[K],
    \end{aligned}
    \right.
\end{align*}
where $\pi^{(i)}_{jk}=\Pi^{(i)}(\mathrm{Y}=j,\mathrm{Y}'=k)$ for $i\in[n]$ and $j,k\in[K]$. 

By Lemma \ref{lemma.dual_linear} and introducing dual variables $\gamma$ and $\tau^{(i)}_k$ with $i\in[n]$ and $k\in[K]$, the dual linear program for $\overline{\mathsf{P}}_{\epsilon}$ is expressed as 
\begin{align*}
    \overline{\mathsf{D}}_{\epsilon}
    =& \left\{
    \begin{aligned}
        &\hspace{-1.2cm}\min_{\gamma,\tau^{(i)}_k\in\mathbb{R}\ \text{for}\ i\in[n]\ \text{and}\ k\in[K]}\ \bigg\{\gamma\epsilon+\sum_{i=1}^{n}\sum_{k=1}^{K}P_{\mathrm{y}|\mathbf{x}_i,\widetilde{\mathbf{y}}_i}(k)\tau^{(i)}_{k}\bigg\}, \\
        \hspace{0.6cm} s.t.&\ \ \frac{1}{n}\gamma c(j,k)+\tau^{(i)}_k\ge \frac{1}{n}\ell(\psi(\mathbf{x}_i), j)\ \text{for}\ i\in[n]\ \text{and}\ j,k\in[K], \\
        \hspace{0.6cm} &\ \ \gamma\ge 0,
    \end{aligned}
    \right.\\
    =& \left\{
    \begin{aligned}
        &\hspace{-1.2cm}\min_{\gamma,\tau^{(i)}_k\in\mathbb{R}\ \text{for}\ i\in[n]\ \text{and}\ k\in[K]}\ \bigg\{\gamma\epsilon+\frac{1}{n}\sum_{i=1}^{n}\sum_{k=1}^{K}P_{\mathrm{y}|\mathbf{x}_i,\widetilde{\mathbf{y}}_i}(k)\widetilde{\tau}^{(i)}_{k}\bigg\}, \\
        \hspace{0.6cm} s.t.&\ \ \gamma c(j,k)+\widetilde{\tau}^{(i)}_k\ge \ell(\psi(\mathbf{x}_i), j)\ \text{for}\ i\in[n]\ \text{and}\ j,k\in[K], \\
        \hspace{0.6cm} &\ \ \gamma\ge 0,
    \end{aligned}
    \right.
\end{align*}
where we let $\widetilde{\tau}^{(i)}_k=n\tau^{(i)}_k$ for $i\in[n]$ and $k\in[K]$ in the second step. Thus, the proof is completed.


\subsection{Proof of Theorem \ref{thm.bias}}\label{appdx.pf_thm_bias}

We begin by demonstrating that, for various choices of the reference distribution, the optimal value for $\gamma$ in the relaxed dual problem, as stated in (\ref{eq.re_rob_risk}), is constrained to a compact set. The proof techniques in \cite{lee2018minimax} are used.

Specifically, for a given $\epsilon>0$ and $\psi\in\Psi$, let 
\begin{align*}
    \gamma^*\in \arg\inf_{\gamma\ge 0}\mathbb{E}_{\mathbf{x},\widetilde{\mathbf{y}}}\bigg(\gamma\epsilon^p +\mathbb{E}_{P_{\mathrm{y}|\mathbf{x},\widetilde{\mathbf{y}}}}\Big[\sup_{y'\in\mathcal{Y}}\left\{\ell(\psi(\mathbf{X}),y')-\gamma c^p(y',\mathrm{Y})\Big]\right\}\bigg).
\end{align*}
Noting that for any $\mathbf{x}\in\mathcal{X}$ and $\mathrm{y}\in\mathcal{Y}$, $\sup_{y'\in\mathcal{Y}}\big\{\ell(\psi(\mathbf{x}),y')-\ell(\psi(\mathbf{x}),\mathrm{y})-\gamma^* c^p(y',\mathrm{y})\big\}\ge\big\{\ell(\psi(\mathbf{x}),y')-\ell(\psi(\mathbf{x}),\mathrm{y})-\gamma^* c^p(y',\mathrm{y})\big\}\big|_{y'=\mathrm{y}}=0$, we obtain that for any $\gamma\ge 0$,
\begin{align}\label{appdx.pf_bias_gamma*}
    \gamma^*\epsilon^p\le&\gamma^*\epsilon^p+\mathbb{E}_{\mathbf{x},\widetilde{\mathbf{y}}}\bigg(\mathbb{E}_{P_{\mathrm{y}|\mathbf{x},\widetilde{\mathbf{y}}}}\Big[\sup_{y'\in\mathcal{Y}}\big\{\ell(\psi(\mathbf{X}),y')-\ell(\psi(\mathbf{X}),\mathrm{Y})-\gamma^* c^p(y',\mathrm{Y})\big\}\Big]\bigg)\notag\\
    \le&\gamma\epsilon^p+\mathbb{E}_{\mathbf{x},\widetilde{\mathbf{y}}}\bigg(\mathbb{E}_{P_{\mathrm{y}|\mathbf{x},\widetilde{\mathbf{y}}}}\Big[\sup_{y'\in\mathcal{Y}}\big\{\ell(\psi({\mathbf{X}}),y')-\ell(\psi({\mathbf{X}}),\mathrm{Y})-\gamma c^p(y',\mathrm{Y})\big\}\Big]\bigg)\notag\\
    \le&\gamma\epsilon^p+\mathbb{E}_{\mathbf{x},\widetilde{\mathbf{y}}}\bigg(\mathbb{E}_{P_{\mathrm{y}|\mathbf{x},\widetilde{\mathbf{y}}}}\Big[\sup_{y'\in\mathcal{Y}}\big\{L\cdot c(y',\mathrm{Y})-\gamma c^p(y',\mathrm{Y})\big\}\Big]\bigg)\notag\\
    \le& \gamma\epsilon^p+\sup_{t\ge 0}\big\{Lt-\gamma t^p\big\},
\end{align}
where the second inequality is due to the definition of $\gamma^*$; the third inequality comes from the Lipschitz property in the assumption; and the last inequality holds because $c(y',\mathrm{Y})=\kappa\mathbf{1}(y'\ne\mathrm{Y})$ takes values in $\{0,\kappa\}$, leading to $\sup_{y'\in\mathcal{Y}}\big\{L\cdot c(y',\mathrm{Y})-\gamma \cdot c^p(y',\mathrm{Y})\big\}=\sup_{t\in\{0, \kappa\}}\big\{Lt-\gamma t^p\big\}\le \sup_{t\ge 0}\big\{Lt-\gamma t^p\big\}$, which is a constant.

Now, we show that 
\begin{align}\label{appdx.pf_gamma_upper}
    \gamma^*\le L\epsilon^{-(p-1)}\triangleq M^*.
\end{align}
Indeed, when $p=1$, then taking $\gamma=L$ in (\ref{appdx.pf_bias_gamma*}) shows (\ref{appdx.pf_gamma_upper}). When $p>1$, then $Lt-\gamma t^p$ in (\ref{appdx.pf_bias_gamma*}) takes its maximum value at $t^*=\{L/(p\gamma)\}^{1/(p-1)}$, and hence, (\ref{appdx.pf_bias_gamma*}) yields that for any $\gamma\ge 0$,
\begin{align*}
    \gamma^*\epsilon^p\le& \gamma\epsilon^p+ L\cdot \{L/(p\gamma)\}^{1/(p-1)}-\gamma \{L/(p\gamma)\}^{p/(p-1)}\\
    =&\gamma\epsilon^p + L^{p/(p-1)}\gamma^{-1/(p-1)}p^{-p/(p-1)}(p-1).
\end{align*}
Therefore, taking $\gamma=L/(p\epsilon^{p-1})$ leads to $\gamma^*\epsilon^p\le L\epsilon$, i.e., (\ref{appdx.pf_gamma_upper}) holds.

Next, let $\ell_{\gamma,\psi}(\mathbf{x},\widetilde{\mathbf{y}})\triangleq \mathbb{E}_{P_{\mathrm{y}|\mathbf{x},\widetilde{\mathbf{y}}}}\Big[\sup_{y'\in\mathcal{Y}}\left\{\ell(\psi(\mathbf{x}),y')-\gamma c^p(y',\mathrm{Y})\right\}\Big]$ for any $(\mathbf{x},\widetilde{\mathbf{y}})\in\mathcal{X}\times\mathcal{Y}^R$. For every $\psi\in\Psi$, we have that 
\begin{align}\label{appdx.pf_bias_upper}
    &\left|\mathfrak{R}_{\epsilon}(\psi;P_{\mathrm{y}|\mathbf{x},\widetilde{\mathbf{y}}})-\widehat{\mathfrak{R}}_{\epsilon}(\psi;P_{\mathrm{y}|\mathbf{x},\widetilde{\mathbf{y}}})\right|\notag\\
    =&\bigg|\inf_{\gamma\ge 0}\mathbb{E}_{\mathbf{x},\widetilde{\mathbf{y}}}\bigg(\gamma\epsilon^p +\mathbb{E}_{P_{\mathrm{y}|\mathbf{x},\widetilde{\mathbf{y}}}}\Big[\sup_{y'\in\mathcal{Y}}\left\{\ell(\psi(\mathbf{X}),y')-\gamma c^p(y',\mathrm{Y})\right\}\Big]\bigg)\notag\\
    &\ \ \ \ \ -\inf_{\gamma\ge 0}\mathbb{E}_{P^{(n)}_{\mathbf{x},\widetilde{\mathbf{y}}}}\bigg(\gamma\epsilon^p +\mathbb{E}_{P_{\mathrm{y}|\mathbf{x},\widetilde{\mathbf{y}}}}\Big[\sup_{y'\in\mathcal{Y}}\left\{\ell(\psi(\mathbf{X}),y')-\gamma c^p(y',\mathrm{Y})\right\}\Big]\bigg)\bigg|\notag\\
    =&\Big|\inf_{0\le \gamma\le M^*}\mathbb{E}_{\mathbf{x},\widetilde{\mathbf{y}}}\left\{\gamma\epsilon^p+\ell_{\gamma,\psi}(\mathbf{X},\widetilde{\mathbf{Y}})\right\}-\inf_{0\le \gamma\le M^*}\mathbb{E}_{P^{(n)}_{\mathbf{x},\widetilde{\mathbf{y}}}}\left\{\gamma\epsilon^p+\ell_{\gamma,\psi}(\mathbf{X},\widetilde{\mathbf{Y}})\right\}\Big|\notag\\
    \le&\sup_{0\le \gamma\le M^*}\Big|\mathbb{E}_{\mathbf{x},\widetilde{\mathbf{y}}}\left\{\ell_{\gamma,\psi}(\mathbf{X},\widetilde{\mathbf{Y}})\right\}-\mathbb{E}_{P^{(n)}_{\mathbf{x},\widetilde{\mathbf{y}}}}\left\{\ell_{\gamma,\psi}(\mathbf{X},\widetilde{\mathbf{Y}})\right\}\Big|\notag\\
    \triangleq& \Phi(\mathcal{D})
\end{align}
where the first equality comes from (\ref{eq.re_rob_risk}) and (\ref{eq.re_rob_risk_emp}), the second equality follows from (\ref{appdx.pf_gamma_upper}) and the definition of $\ell_{\gamma,\psi}$, and the third step is due to the fact that $|\inf_{v\in A} f(v) - \inf_{v\in A} g(v)|\le \sup_{v\in A} |f(v)-g(v)|$ for bounded functions $f,g: A\rightarrow \mathbb{R}$ \cite[Proposition 2.18]{huntsupremum}. 

In (\ref{appdx.pf_bias_upper}), we use $\mathcal{D}$ to stress the dependence of $\mathbb{E}_{P^{(n)}_{\mathbf{x},\widetilde{\mathbf{y}}}}\left\{\ell_{\gamma,\psi}(\mathbf{X},\widetilde{\mathbf{Y}})\right\}$ on the observed data of size $n$, as defined in Section \ref{sec.setup}, and let $\Phi$ represent the resulting function mapping from $\left(\mathcal{X}\times\mathcal{Y}^R\right)^n$ to $\mathbb{R}$, with $\Phi(\mathcal{D})$ being the value for data $\mathcal{D}$, where $\left(\mathcal{X}\times\mathcal{Y}^R\right)^n$ is the Cartesian product of multiplying $\mathcal{X}\times\mathcal{Y}^R$ $n$ times. The function $\Phi:\left(\mathcal{X}\times\mathcal{Y}^R\right)^n\rightarrow\mathbb{R}$ defined in (\ref{appdx.pf_bias_upper}) satisfies the bounded difference property (\ref{appdx.bounded_diff_def}) with parameters $\left\{\frac{M}{n},\ldots,\frac{M}{n}\right\}$, where $M$ represents the upper bound of the loss function $\ell$ in the assumption of Theorem \ref{thm.bias}. Indeed, for any $k\in[n]$,
\begin{align*}
    &\sup_{z_1,\ldots,z_n,z'_k\in\mathcal{X}\times\mathcal{Y}^R}\big|\Phi(z_1,\ldots,z_k,\ldots,z_n)-\Phi(z_1,\ldots,z'_k,\ldots,z_n)\big|\\
    =&\sup_{z_1,\ldots,z_n,z'_k\in\mathcal{X}\times\mathcal{Y}^R}\bigg|\sup_{0\le \gamma\le M^*}\Big|\mathbb{E}_{\mathbf{x},\widetilde{\mathbf{y}}}\left\{\ell_{\gamma,\psi}(\mathbf{X},\widetilde{\mathbf{Y}})\right\}-\frac{1}{n}\sum_{i=1}^{n}\ell_{\gamma,\psi}(z_i)\Big|\\
    &\hspace{2.2cm}-\sup_{0\le \gamma\le M^*}\Big|\mathbb{E}_{\mathbf{x},\widetilde{\mathbf{y}}}\left\{\ell_{\gamma,\psi}(\mathbf{X},\widetilde{\mathbf{Y}})\right\}-\frac{1}{n}\sum_{i=1}^{n}\ell_{\gamma,\psi}(z_i)-\frac{1}{n}\ell_{\gamma,\psi}(z'_k)+\frac{1}{n}\ell_{\gamma,\psi}(z_k)\Big|\bigg|\\
    \le& \sup_{\scriptstyle z_1,\ldots,z_n,z'_k\in\mathcal{X}\times\mathcal{Y}^R, 0\le \gamma\le M^*}\bigg|\Big|\mathbb{E}_{\mathbf{x},\widetilde{\mathbf{y}}}\left\{\ell_{\gamma,\psi}(\mathbf{X},\widetilde{\mathbf{Y}})\right\}-\frac{1}{n}\sum_{i=1}^{n}\ell_{\gamma,\psi}(z_i)\Big|\\ 
    &\hspace{3.5cm}-\Big|\mathbb{E}_{\mathbf{x},\widetilde{\mathbf{y}}}\left\{\ell_{\gamma,\psi}(\mathbf{X},\widetilde{\mathbf{Y}})\right\}-\frac{1}{n}\sum_{i=1}^{n}\ell_{\gamma,\psi}(z_i)-\frac{1}{n}\ell_{\gamma,\psi}(z'_k)+\frac{1}{n}\ell_{\gamma,\psi}(z_k)\Big|\bigg|\\
    \le& \sup_{\scriptstyle z_1,\ldots,z_n,z'_k\in\mathcal{X}\times\mathcal{Y}^R, 0\le \gamma\le M^*}\bigg|\frac{1}{n}\ell_{\gamma,\psi}(z'_k)-\frac{1}{n}\ell_{\gamma,\psi}(z_k)\bigg|\\
    \le& \frac{M}{n},
\end{align*}
where the second step holds since $|\sup_{v\in A} f(v) - \sup_{v\in A} g(v)|\le \sup_{v\in A} |f(v)-g(v)|$ for bounded functions $f,g: A\rightarrow \mathbb{R}$ \cite[Proposition 2.18]{huntsupremum}, the third step is due to the triangle inequality for absolute values, and the last step holds since $\ell_{\gamma,\psi}\in[0,M]$ by definition. 

Thus, by letting $t=M\sqrt{\frac{\log (1/\eta)}{2n}}$ in lemma \ref{lemma.concentration}, we have that, with probability at least $1-\eta$, 
\begin{align}\label{appdx.pf_bias_phi_D_diff}
    \Phi(\mathcal{D})\le& \mathbb{E}_{\mathbf{x},\widetilde{\mathbf{y}}}\left\{\Phi(\mathcal{D})\right\}+M\sqrt{\frac{\log (1/\eta)}{2n}}.
\end{align}
Similar to the derivations for (3.8)-(3.13) in the proof of Theorem 3.3 in \cite{mohri2018foundations}, we obtain that 
\begin{align}\label{appdx.pf_bias_rade}
    \mathbb{E}_{\mathbf{x},\widetilde{\mathbf{y}}}\left\{\Phi(\mathcal{D})\right\}\le 2 \mathbb{E}\left[\sup_{\gamma\in[0,M^*]}\frac{1}{n}\sum_{i=1}^{n}\sigma_i\ell_{\gamma,\psi}(\mathbf{X}_i,\widetilde{\mathbf{Y}}_i)\right],
\end{align}
where $\{\sigma_i\}_{i=1}^{n}$ are independent random variables chosen from $\{-1,+1\}$ with equal probability, and the expectation is taken with respect to all involved random variables. 

Applying (\ref{appdx.pf_bias_phi_D_diff}) and (\ref{appdx.pf_bias_rade}) to (\ref{appdx.pf_bias_upper}) gives that with probability at least $1-\eta$, 
\begin{align}\label{appdx.pf_bias_upper_all}
    \left|\mathfrak{R}_{\epsilon}(\psi;P_{\mathrm{y}|\mathbf{x},\widetilde{\mathbf{y}}})-\widehat{\mathfrak{R}}_{\epsilon}(\psi;P_{\mathrm{y}|\mathbf{x},\widetilde{\mathbf{y}}})\right|\le 
    2 \mathbb{E}\left[\sup_{\gamma\in[0,M^*]}\frac{1}{n}\sum_{i=1}^{n}\sigma_i\ell_{\gamma,\psi}(\mathbf{X}_i,\widetilde{\mathbf{Y}}_i)\right]+M\sqrt{\frac{\log (1/\eta)}{2n}}.
\end{align}

Now we identify an entropy based upper bound for the right-hand side of (\ref{appdx.pf_bias_rade}) using the proof techniques in \cite{lee2018minimax} and Example 5.24 of \cite{wainwright2019high}. Specifically, for a given $\psi$, we define the random process $\left\{S_{\gamma}\triangleq\frac{1}{\sqrt{n}}\sum_{i=1}^{n}\sigma_i\ell_{\gamma,\psi}(\mathbf{X}_i,\widetilde{\mathbf{Y}}_i):\gamma\in[0,M^*]\right\}$. By using $\mathbb{E}(\sigma_i)=0$ and the independence between $\sigma_i$ and $(\mathbf{X}_i,\widetilde{\mathbf{Y}}_i)$, we obtain that $\mathbb{E}(S_{\gamma})=\frac{1}{\sqrt{n}}\sum_{i=1}^{n}\mathbb{E}(\sigma_i)\mathbb{E}\big\{\ell_{\gamma,\psi}(\mathbf{X}_i,\widetilde{\mathbf{Y}}_i)\big\}=0$. For any $\gamma_1,\gamma_2\in[0,M^*]$, 
\begin{align}\label{appdx.pf_bias_lip}
    &\big|\ell_{\gamma_1,\psi}(\mathbf{x},\widetilde{\mathbf{y}})-\ell_{\gamma_2,\psi}(\mathbf{x},\widetilde{\mathbf{y}})\big|\notag\\
    =&\big|\mathbb{E}_{P_{\mathrm{y}|\mathbf{x},\widetilde{\mathbf{y}}}}\sup_{y'\in\mathcal{Y}}\left\{\ell(\psi(\mathbf{x}),y')-\gamma_1 c^p(y',\mathrm{Y})\right\}-\mathbb{E}_{P_{\mathrm{y}|\mathbf{x},\widetilde{\mathbf{y}}}}\sup_{y'\in\mathcal{Y}}\left\{\ell(\psi(\mathbf{x}),y')-\gamma_2 c^p(y',\mathrm{Y})\right\}\big|\notag\\
    \le&\mathbb{E}_{P_{\mathrm{y}|\mathbf{x},\widetilde{\mathbf{y}}}}\left|\sup_{y'\in\mathcal{Y}}\left\{\ell(\psi(\mathbf{x}),y')-\gamma_1 c^p(y',\mathrm{Y})\right\}-\sup_{y'\in\mathcal{Y}}\left\{\ell(\psi(\mathbf{x}),y')-\gamma_2 c^p(y',\mathrm{Y})\right\}\right|\notag\\
    \le& \mathbb{E}_{P_{\mathrm{y}|\mathbf{x},\widetilde{\mathbf{y}}}}\sup_{y'\in\mathcal{Y}}\big|\gamma_1 c^p(y',\mathrm{Y})-\gamma_2 c^p(y',\mathrm{Y})\big|\notag\\
    =&\kappa^p|\gamma_1-\gamma_2|,
\end{align}
where the second inequality is due to the fact that $|\sup_A f - \sup_A g|\le \sup_A |f-g|$ for bounded functions $f,g: A\rightarrow \mathbb{R}$ \cite[Proposition 2.18]{huntsupremum}, and the last step is due to the fact that $c(y',\mathrm{Y})$ can only take values in $\{0,\kappa\}$.

Hence, for $t\in\mathbb{R}$, we have that 
\begin{align}\label{appdx.pf_bias_sub_gaussian_verify}
    \mathbb{E}\left\{e^{t(S_{\gamma_1}-S_{\gamma_2})}\right\}=&\mathbb{E}\left\{\exp\left[\frac{t}{\sqrt{n}}\sum_{i=1}^{n}\sigma_i\left\{\ell_{\gamma_1,\psi}(\mathbf{X}_i,\widetilde{\mathbf{Y}}_i)-\ell_{\gamma_2,\psi}(\mathbf{X}_i,\widetilde{\mathbf{Y}}_i)\right\}\right]\right\}\notag\\
    =&\left\{\mathbb{E}\left(\exp\left[\frac{t}{\sqrt{n}}\sigma_1\left\{\ell_{\gamma_1,\psi}(\mathbf{X}_1,\widetilde{\mathbf{Y}}_1)-\ell_{\gamma_2,\psi}(\mathbf{X}_1,\widetilde{\mathbf{Y}}_1)\right\}\right]\right)\right\}^n\notag\\
    \le&\exp\left\{\frac{t^2(\kappa^p|\gamma_1-\gamma_2|)^2}{2}\right\},
\end{align}
where the inequality is due to Hoeffding's lemma and (\ref{appdx.pf_bias_lip}). Thus, $\{S_\gamma: \gamma\in[0,M^*]\}$ is a zero-mean sub-Gaussian process with respect to metric $\rho_\gamma$, defined as $\rho_\gamma(\gamma_1,\gamma_2)=\kappa^p |\gamma_1-\gamma_2|$ for any $\gamma_1,\gamma_2\in[0, M^*]$. 

Therefore, by Lemma \ref{lemma.entropy}, we obtain 
\begin{align}\label{appdx.pf_bias_upper_rade}
    &\mathbb{E}\left[\sup_{\gamma\in[0,M^*]}\frac{1}{n}\sum_{i=1}^{n}\sigma_i\ell_{\gamma,\psi}(\mathbf{X}_i,\widetilde{\mathbf{Y}}_i)\right]\notag\\
    =&\frac{1}{\sqrt{n}}\mathbb{E}\left(\sup_{\gamma\in[0,M^*]}S_\gamma\right)\notag\\
    \le& \frac{8\sqrt{2}}{\sqrt{n}}\int_{0}^{\infty}\sqrt{\log N(t; [0, M^*], \rho_\gamma)}dt\notag\\
    =& \frac{8\sqrt{2}}{\sqrt{n}}\int_{0}^{\infty}\sqrt{\log N(t/\kappa^p; [0, M^*], |\cdot|)}dt\notag\\
    =& \frac{8\sqrt{2}\kappa^p}{\sqrt{n}}\int_{0}^{\infty}\sqrt{\log N(s; [0, M^*], |\cdot|)}ds\notag\\
    \le& \frac{8\sqrt{2}\kappa^p}{\sqrt{n}}\int_{0}^{M^*/2}\sqrt{\log \left(\frac{M^*}{2s}+1\right)}ds\notag\\
    \le& \frac{8\sqrt{2}\kappa^p}{\sqrt{n}}\int_{0}^{M^*/2}\sqrt{\log \left(\frac{M^*}{s}\right)}ds\notag\\
    =&\frac{8\sqrt{2}}{\sqrt{n}} M^*\kappa^p \int_{0}^{1/2}\sqrt{\log(1/u)}du\notag\\
    =&\frac{4\sqrt{2}\{\sqrt{\log 2}+\sqrt{\pi}\text{erfc}(\sqrt{\log 2})\}}{\sqrt{n}} M^*\kappa^p,
\end{align}
where the third step comes from the fact that, by the definition of $\rho_\gamma$,  $\rho_\gamma(\gamma_1,\gamma_2)\le t$ if and only if $|\gamma_1-\gamma_2|\le t/\kappa^p$; the fourth step holds by letting $s=t/\kappa^p$; the fifth step follows from Lemma \ref{lemma.covering_interval}; the sixth step comes from the fact that $1<\frac{M^*}{2s}+1=\frac{M^*+2s}{2s}\le\frac{M^*}{s}$ for $s\in[0,M^*/2]$; the penultimate step holds by letting $u=s/M^*$; and the last step arises from the fact that 
\begin{align*}
    &\int_{0}^{1/2}\sqrt{\log(1/u)}du=\int_{0}^{1/2}\sqrt{-\log u}du=\int_{0}^{1/2}\int_{0}^{-\log (u)}\frac{1}{2\sqrt{s}}dsdu\\
    =&\int_{0}^{\log 2}\int_{0}^{1/2}\frac{1}{2\sqrt{s}}duds+\int_{\log 2}^{\infty}\int_{0}^{e^{-s}}\frac{1}{2\sqrt{s}}duds=\frac{\sqrt{\log 2}}{2}+\int_{\log 2}^{\infty}\frac{e^{-s}}{2\sqrt{s}}ds\\
    =&\frac{\sqrt{\log 2}}{2}+\int_{\sqrt{\log 2}}^{\infty}\frac{e^{-w^2}}{2w}2wdw=\frac{\sqrt{\log 2}}{2}+\frac{\sqrt{\pi}}{2}\text{erfc}(\sqrt{\log 2}),
\end{align*}
where $\text{erfc}(x)=\frac{2}{\sqrt{\pi}}\int_{x}^{\infty}e^{-w^2}dw$. 

Applying (\ref{appdx.pf_bias_upper_rade}) and (\ref{appdx.pf_gamma_upper}) to (\ref{appdx.pf_bias_upper_all}) gives that with probability at least $1-\eta$, 
\begin{align*}
    \left|\mathfrak{R}_{\epsilon}(\psi;P_{\mathrm{y}|\mathbf{x},\widetilde{\mathbf{y}}})-\widehat{\mathfrak{R}}_{\epsilon}(\psi;P_{\mathrm{y}|\mathbf{x},\widetilde{\mathbf{y}}})\right|\le& \frac{8\sqrt{2}(\sqrt{\log 2}+\sqrt{\pi}\text{erfc}(\sqrt{\log 2}))}{\sqrt{n}}\cdot \frac{L\kappa^p}{\epsilon^{p-1}}+M\sqrt{\frac{\log (1/\eta)}{2n}}\\
    <&\frac{15L\kappa^p}{\epsilon^{p-1}\sqrt{n}}+M\sqrt{\frac{\log (1/\eta)}{2n}},
\end{align*}
where the last step is due to the fact that $8\sqrt{2}(\sqrt{\log 2}+\sqrt{\pi}\text{erfc}(\sqrt{\log 2}))\approx 14.2<15$. Hence, the proof is completed.

\subsection{Proof of Corollary \ref{thm.regret_empirical}}\label{appdx.pf_thm_regret_empirical}

By the definition of $\inf_{\psi\in\Psi}\mathfrak{R}_{\epsilon}(\psi;P_{\mathrm{y}|\mathbf{x},\widetilde{\mathbf{y}}})$, for any $\zeta>0$, there exists $\psi_\zeta\in\Psi$ such that $\mathfrak{R}_{\epsilon}(\psi_\zeta;P_{\mathrm{y}|\mathbf{x},\widetilde{\mathbf{y}}})\le\inf_{\psi\in\Psi}\mathfrak{R}_{\epsilon}(\psi;P_{\mathrm{y}|\mathbf{x},\widetilde{\mathbf{y}}})+\zeta$. Therefore,
\begin{align*}
    &\mathfrak{R}_{\epsilon}(\widehat{\psi}_{\epsilon,n};P_{\mathrm{y}|\mathbf{x},\widetilde{\mathbf{y}}})- \inf_{\psi\in\Psi}\mathfrak{R}_{\epsilon}(\psi;P_{\mathrm{y}|\mathbf{x},\widetilde{\mathbf{y}}})\\
    \le&\mathfrak{R}_{\epsilon}(\widehat{\psi}_{\epsilon,n};P_{\mathrm{y}|\mathbf{x},\widetilde{\mathbf{y}}})- \mathfrak{R}_{\epsilon}(\psi_\zeta;P_{\mathrm{y}|\mathbf{x},\widetilde{\mathbf{y}}})+\zeta\\
    \le&\mathfrak{R}_{\epsilon}(\widehat{\psi}_{\epsilon,n};P_{\mathrm{y}|\mathbf{x},\widetilde{\mathbf{y}}})-\widehat{\mathfrak{R}}_{\epsilon}(\widehat{\psi}_{\epsilon,n};P_{\mathrm{y}|\mathbf{x},\widetilde{\mathbf{y}}})+\widehat{\mathfrak{R}}_{\epsilon}(\psi_\zeta;P_{\mathrm{y}|\mathbf{x},\widetilde{\mathbf{y}}})- \mathfrak{R}_{\epsilon}(\psi_\zeta;P_{\mathrm{y}|\mathbf{x},\widetilde{\mathbf{y}}})+\zeta\\
    \le& 2\sup_{\psi\in\Psi}\left|\mathfrak{R}_{\epsilon}(\psi;P_{\mathrm{y}|\mathbf{x},\widetilde{\mathbf{y}}})-\widehat{\mathfrak{R}}_{\epsilon}(\psi;P_{\mathrm{y}|\mathbf{x},\widetilde{\mathbf{y}}})\right|+\zeta.
\end{align*}
Since the inequality above is true for all $\zeta>0$, we have that 
\begin{align}\label{appdx.pf_regret_empirical_upper}
    &\mathfrak{R}_{\epsilon}(\widehat{\psi}_{\epsilon,n};P_{\mathrm{y}|\mathbf{x},\widetilde{\mathbf{y}}})- \inf_{\psi\in\Psi}\mathfrak{R}_{\epsilon}(\psi;P_{\mathrm{y}|\mathbf{x},\widetilde{\mathbf{y}}})\notag\\
    \le& 2\sup_{\psi\in\Psi}\left|\mathfrak{R}_{\epsilon}(\psi;P_{\mathrm{y}|\mathbf{x},\widetilde{\mathbf{y}}})-\widehat{\mathfrak{R}}_{\epsilon}(\psi;P_{\mathrm{y}|\mathbf{x},\widetilde{\mathbf{y}}})\right|\notag\\
    \le& 2\sup_{\psi\in\Psi,0\le \gamma\le M^*}\Big|\mathbb{E}_{\mathbf{x},\widetilde{\mathbf{y}}}\left\{\ell_{\gamma,\psi}(\mathbf{X},\widetilde{\mathbf{Y}})\right\}-\mathbb{E}_{P^{(n)}_{\mathbf{x},\widetilde{\mathbf{y}}}}\left\{\ell_{\gamma,\psi}(\mathbf{X},\widetilde{\mathbf{Y}})\right\}\Big|,
\end{align}
where the second inequality arises from (\ref{appdx.pf_bias_upper}).

Similar to the proof of Theorem \ref{thm.bias} in Appendix \ref{appdx.pf_thm_bias}, we can derive that 
\begin{align}\label{appdx.pf_regret_empirical_rade}
    &\sup_{\psi\in\Psi,0\le \gamma\le M^*}\Big|\mathbb{E}_{\mathbf{x},\widetilde{\mathbf{y}}}\left\{\ell_{\gamma,\psi}(\mathbf{X},\widetilde{\mathbf{Y}})\right\}-\mathbb{E}_{P^{(n)}_{\mathbf{x},\widetilde{\mathbf{y}}}}\left\{\ell_{\gamma,\psi}(\mathbf{X},\widetilde{\mathbf{Y}})\right\}\Big|\notag\\
    \le& 2 \mathbb{E}\left[\sup_{\gamma\in[0,M^*],\psi\in\Psi}\frac{1}{n}\sum_{i=1}^{n}\sigma_i\ell_{\gamma,\psi}(\mathbf{X}_i,\widetilde{\mathbf{Y}}_i)\right]+M\sqrt{\frac{\log (1/\eta)}{2n}}.
\end{align}

For any $\gamma_1,\gamma_2\in[0,M^*]$ and $\psi_1,\psi_2\in\Psi$, 
\begin{align*}
    &\big|\ell_{\gamma_1,\psi_1}(\mathbf{x},\widetilde{\mathbf{y}})-\ell_{\gamma_2,\psi_2}(\mathbf{x},\widetilde{\mathbf{y}})\big|\notag\\
    =&\big|\mathbb{E}_{P_{\mathrm{y}|\mathbf{x},\widetilde{\mathbf{y}}}}\sup_{y'\in\mathcal{Y}}\left\{\ell(\psi_1(\mathbf{x}),y')-\gamma_1 c^p(y',\mathrm{Y})\right\}-\mathbb{E}_{P_{\mathrm{y}|\mathbf{x},\widetilde{\mathbf{y}}}}\sup_{y'\in\mathcal{Y}}\left\{\ell(\psi_2(\mathbf{x}),y')-\gamma_2 c^p(y',\mathrm{Y})\right\}\big|\notag\\
    \le&\mathbb{E}_{P_{\mathrm{y}|\mathbf{x},\widetilde{\mathbf{y}}}}\left|\sup_{y'\in\mathcal{Y}}\left\{\ell(\psi_1(\mathbf{x}),y')-\gamma_1 c^p(y',\mathrm{Y})\right\}-\sup_{y'\in\mathcal{Y}}\left\{\ell(\psi_2(\mathbf{x}),y')-\gamma_2 c^p(y',\mathrm{Y})\right\}\right|\notag\\
    \le& \mathbb{E}_{P_{\mathrm{y}|\mathbf{x},\widetilde{\mathbf{y}}}}\sup_{y'\in\mathcal{Y}}\left\{\Big|\gamma_1 c^p(y',\mathrm{Y})-\gamma_2 c^p(y',\mathrm{Y})\Big|+\Big|\ell(\psi_1(\mathbf{x}),y')-\ell(\psi_2(\mathbf{x}),y')\Big|\right\}\notag\\
    \le&\kappa^p|\gamma_1-\gamma_2|+L'\|\psi_1-\psi_2\|_\infty,
\end{align*}
where the first step is due to the definition of $\ell_{\gamma,\psi}$ defined after (\ref{appdx.pf_gamma_upper}), the second step is due to Jensen's inequality, and the last step is due to the Lipschitz property with respect to the cost function $c(\cdot,\cdot)$ defined in Theorem \ref{thm.bias} in the assumption, and $\|\psi_1-\psi_2\|_\infty\triangleq\sup_{\mathbf{x}\in\mathcal{X}}\|\psi_1(\mathbf{x})-\psi_2(\mathbf{x})\|$ for some norm $\|\cdot\|$. 

Allowing $\psi$ to vary, we modify the discussion for the random process $\{S_{\gamma}:\gamma\in[0,M^*]\}$ in Appendix \ref{appdx.pf_thm_bias}, and consider the collection of random variables $\Big\{S_{\gamma,\psi}\triangleq\frac{1}{\sqrt{n}}\sum_{i=1}^{n}\sigma_i\ell_{\gamma,\psi}(\mathbf{X}_i,\widetilde{\mathbf{Y}}_i):\gamma\in[0,M^*],\psi\in\Psi\Big\}$. Clearly, $\mathbb{E}(S_{\gamma,\psi})=0$. Modifying the metric $\rho_{\gamma}(\gamma_1,\gamma_2)$ in Appendix \ref{appdx.pf_thm_bias}, we define the metric $\rho_{\gamma,\psi}((\gamma_1,\psi_1),(\gamma_2,\psi_2))\triangleq\kappa^p |\gamma_1-\gamma_2|+L'\|\psi_1-\psi_2\|_\infty$ for any $\gamma_1,\gamma_2\in[0, M^*]$ and $\psi_1,\psi_2\in\Psi$. Similar to deriving (\ref{appdx.pf_bias_sub_gaussian_verify}), we obtain that for $t\in\mathbb{R}$,
\begin{align*}
    \mathbb{E}\left\{e^{t\left(S_{\gamma_1,\psi_1}-S_{\gamma_2,\psi_2}\right)}\right\}\le\exp\left[\frac{t^2\left\{\rho_{\gamma,\psi}((\gamma_1,\psi_1),(\gamma_2,\psi_2))\right\}^2}{2}\right].
\end{align*}
Thus, $\Big\{S_{\gamma,\psi}:\gamma\in[0,M^*],\psi\in\Psi\Big\}$ is a zero-mean sub-Gaussian process with respect to metric $\rho_{\gamma,\psi}$. 

Let the Cartesian product $[0,M^*]\times\Psi$ denote the ``parameter'' space of $(\gamma,\psi)$. Then, by Lemma \ref{lemma.entropy}, we obtain
\begin{align}\label{appdx.pf_regret_empirical_upper_rade}
    &\mathbb{E}\left[\sup_{\gamma\in[0,M^*],\psi\in\Psi}\frac{1}{n}\sum_{i=1}^{n}\sigma_i\ell_{\gamma,\psi}(\mathbf{X}_i,\widetilde{\mathbf{Y}}_i)\right]\notag\\
    =&\frac{1}{\sqrt{n}}\mathbb{E}\left(\sup_{\gamma\in[0,M^*],\psi\in\Psi}S_{\gamma,\psi}\right)\notag\\
    \le& \frac{8\sqrt{2}}{\sqrt{n}}\int_{0}^{\infty}\sqrt{\log N(t; [0, M^*]\times\Psi, \rho_{\gamma,\psi})}dt\notag\\
    \le& \frac{8\sqrt{2}}{\sqrt{n}}\int_{0}^{\infty}\sqrt{\log \left\{N(t/(2\kappa^p); [0, M^*], |\cdot|)\times N(t/(2 L'); \Psi, \|\cdot\|_\infty)\right\}}dt\notag\\
    \le&\frac{8\sqrt{2}}{\sqrt{n}}\int_{0}^{\infty}\sqrt{\log N(t/(2\kappa^p); [0, M^*], |\cdot|)}dt+\frac{8\sqrt{2}}{\sqrt{n}}\int_{0}^{\infty}\sqrt{\log N(t/(2 L'); \Psi, \|\cdot\|_\infty)}dt\notag\\
    \le&\frac{16\sqrt{2}\kappa^p}{\sqrt{n}}\int_{0}^{\infty}\sqrt{\log N(s; [0, M^*], |\cdot|)}ds+\frac{16\sqrt{2}L'}{\sqrt{n}}\int_{0}^{\infty}\sqrt{\log N(s; \Psi, \|\cdot\|_\infty)}ds\notag\\
    \le&\frac{8\sqrt{2}(\sqrt{\log 2}+\sqrt{\pi}\text{erfc}(\sqrt{\log 2}))}{\sqrt{n}}\cdot M^*\kappa^p+\frac{16\sqrt{2}L'}{\sqrt{n}}\int_{0}^{\infty}\sqrt{\log N(s; \Psi, \|\cdot\|_\infty)}ds,
\end{align}
Here, for any set $\Omega$ and metric $\rho$ on $\Omega$, $N(t;\Omega,\rho)$ denotes the $t$-covering number for $t>0$ as defined in Definition \ref{def.covering}. In the derivation of (\ref{appdx.pf_regret_empirical_upper_rade}), the third step is due to Lemma \ref{lemma.covering_cartesian}; the fourth step is due to $\sqrt{\log(ab)}\le\sqrt{\log a}+\sqrt{\log b}$ for $a\ge 1$ and $b\ge 1$; the fifth step results from a change of variable; and the last line can be similarly proved as (\ref{appdx.pf_bias_upper_rade}).

By (\ref{appdx.pf_regret_empirical_upper})-(\ref{appdx.pf_regret_empirical_upper_rade}), we have that, with probability at least $1-\eta$,
\begin{align*}
    &\mathfrak{R}_{\epsilon}(\widehat{\psi}_{\epsilon,n};P_{\mathrm{y}|\mathbf{x},\widetilde{\mathbf{y}}})- \inf_{\psi\in\Psi}\mathfrak{R}_{\epsilon}(\psi;P_{\mathrm{y}|\mathbf{x},\widetilde{\mathbf{y}}})\\
    \le& \left\{57\frac{L\kappa^p}{\epsilon^{p-1}}+91L'\int_{0}^{\infty}\sqrt{\log N(s; \Psi, \|\cdot\|_\infty)}ds\right\}\cdot\frac{1}{\sqrt{n}}+2M\sqrt{\frac{\log (1/\eta)}{2n}}.
\end{align*}
Therefore, the proof is completed.


\subsection{Proof of Theorem \ref{thm.optimal_action}}\label{appdx.pf_thm_optimal_action}

For ease of presentation, in this proof we omit the dependence on $\mathbf{x}$ and $\widetilde{\mathbf{y}}$ in the notation. In particular, we let $P_0\triangleq P_0(\mathbf{x},\widetilde{\mathbf{y}})$, $P_1\triangleq P_1(\mathbf{x},\widetilde{\mathbf{y}})$, and $\psi\triangleq\psi(\mathbf{x})$. Let the objective function in (\ref{eq.binary_problem}) be denoted as 
\begin{align}\label{appdx.pf_binary_obj_fun}
    \mathdutchcal{g}(\gamma;\psi)\triangleq \gamma\epsilon^p+P_0\max\{\mathcal{T}(1-\psi), \mathcal{T}(\psi)-\gamma\kappa^p\}+P_1\max\{\mathcal{T}(1-\psi)-\gamma\kappa^p,\mathcal{T}(\psi)\}.
\end{align}

To complete the proof, we begin by investigating the \emph{inner optimization problem} in (\ref{eq.binary_problem}) by finding the optimal value of $\gamma$, $\gamma^{*}_{\psi}$, as defined in Section \ref{sec.close_em_risk}, that minimizes $\mathdutchcal{g}(\gamma;\psi)$ for each given $\psi$, and then address the \emph{outer optimization problem} in (\ref{eq.binary_problem}) by finding the optimal value $\psi^\star$ that minimizes $\mathdutchcal{g}(\gamma^*_{\psi};\psi)$. To this end, we eliminate the $\max$ operators in $\mathdutchcal{g}(\gamma;\psi)$ based on the values of $\psi$, and use the assumption that $\mathcal{T}$ is a decreasing function in (\ref{eq.loss_bianry}). As $\psi$ takes its value in $[0,1]$, we re-write the optimization problem (\ref{eq.binary_problem}) as 
\begin{align}\label{appdx.pf_binary_three_cases}
    &\inf_{\psi\in[0,\frac{1}{2}]\cup[\frac{1}{2},1]}\inf_{\gamma\ge 0}\mathdutchcal{g}(\gamma;\psi)\notag\\
    =&\min\left\{\min_{\psi\in[0,\frac{1}{2}]}\inf_{\gamma\ge 0}\mathdutchcal{g}(\gamma;\psi),\min_{\psi\in[\frac{1}{2},1]}\inf_{\gamma\ge 0}\mathdutchcal{g}(\gamma;\psi)\right\},
    \notag\\
    \triangleq&\min\left\{\mathdutchcal{g}(\gamma^*_{\psi^*_1};\psi^*_1), \mathdutchcal{g}(\gamma^*_{\psi^*_2};\psi^*_2)\right\},
\end{align}
where $(\gamma^*_{\psi^*_1},\psi^*_1)$ and $(\gamma^*_{\psi^*_2},\psi^*_2)$ are the arguments of $\min_{\psi_1\in[0,\frac{1}{2}]}\inf_{\gamma\ge 0}\mathdutchcal{g}(\gamma;\psi_1)$ and $\min_{\psi_2\in[\frac{1}{2},1]}\inf_{\gamma\ge 0}\mathdutchcal{g}(\gamma;\psi_2)$, respectively. We complete the proof by considering the following two cases.

\textbf{Case 1: $\psi_1\in[0,\frac{1}{2}]$.}

In this case, $\mathcal{T}(\psi_1)\ge\mathcal{T}(\frac{1}{2})\ge \mathcal{T}(1-\psi_1)\ge \mathcal{T}(1-\psi_1)-\gamma\kappa^p$. Let 
$$\gamma_0=\frac{\mathcal{T}(\psi_1)-\mathcal{T}(1-\psi_1)}{\kappa^p}.$$
Consequently, we have that 
\begin{align}\label{appdx.pf_binary_case2}
    \mathdutchcal{g}(\gamma;\psi_1)
    =& \left\{
    \begin{aligned}
        &\gamma\epsilon^p+P_0\{\mathcal{T}(\psi_1)-\gamma\kappa^p\}+P_1\mathcal{T}(\psi_1)=\mathcal{T}(\psi_1)+\gamma\kappa^p(\varrho(\epsilon)-P_0),\ \text{if}\ \gamma\le\gamma_0;\\
        &\gamma\epsilon^p+P_0\mathcal{T}(1-\psi_1)+P_1\mathcal{T}(\psi_1),\ \text{if}\ \gamma>\gamma_0.
    \end{aligned}
    \right.
\end{align}
For given $\psi_1$,
\begin{align*}
    \lim_{\gamma\rightarrow \gamma_0^+}\mathdutchcal{g}(\gamma;\psi_1)=\lim_{\gamma\rightarrow\gamma_0^-}\mathdutchcal{g}(\gamma;\psi_1)=\varrho(\epsilon)\left\{\mathcal{T}(\psi_1)-\mathcal{T}(1-\psi_1)\right\}+P_0\mathcal{T}(1-\psi_1)+P_1\mathcal{T}(\psi_1),
\end{align*}
showing that $\mathdutchcal{g}(\gamma;\psi_1)$ is continuous at $\gamma_0$. Therefore, for any given $\psi_1$, $\mathdutchcal{g}(\gamma;\psi_1)$ is continuous in $\gamma$ over $\mathbb{R}^+$. 

Then, for any given $\psi_1\in[0,\frac{1}{2}]$, corresponding to the first term in (\ref{appdx.pf_binary_three_cases}), we obtain that 
\begin{align}\label{appdx.pf_binary_case2_}
    \inf_{\gamma\ge 0}\mathdutchcal{g}(\gamma;\psi_1)=&\min\left\{\inf_{\gamma>\gamma_0}\mathdutchcal{g}(\gamma;\psi_1),\inf_{\gamma\in[0,\gamma_0]}\mathdutchcal{g}(\gamma;\psi_1)\right\}\notag\\
    =&\min\left\{\mathdutchcal{g}(\gamma_0;\psi_1),\min_{\gamma\in[0,\gamma_0]}\mathdutchcal{g}(\gamma;\psi_1)\right\}\notag\\
    =&\min_{\gamma\in[0,\gamma_0]}\mathdutchcal{g}(\gamma;\psi_1)\notag\\
    \triangleq&\mathdutchcal{g}(\gamma^*_{\psi_1};\psi_1),
\end{align}
where we use the continuity of $\mathdutchcal{g}(\gamma;\psi_1)$ in $\gamma$, the fact that $\mathdutchcal{g}(\gamma;\psi_1)$ is increasing in $\gamma$ when $\gamma>\gamma_0$, and the fact that a continuous function attains its infimum within any closed and bounded set in $\mathbb{R}$. Here, 
\begin{align}\label{appdx.pf_binary_gamma*_1}
    \gamma^*_{\psi_1}\triangleq\arg\min_{\gamma\in[0,\gamma_0]}\mathdutchcal{g}(\gamma;\psi_1)
\end{align}
for any $\psi_1\in[0,\frac{1}{2}]$.

We complete the proof by the following two steps to examine the range of $P_0$.

{\textbf{Step 1:}} If $P_0<\varrho(\epsilon)$, then, by (\ref{appdx.pf_binary_case2}), for any given $\psi_1\in[0,\frac{1}{2}]$, $\mathdutchcal{g}(\gamma;\psi_1)$ is increasing in $\gamma$ over $[0,\gamma_0]$, showing that the optimal value in (\ref{appdx.pf_binary_gamma*_1}) is $\gamma^{*}_{\psi_1}=0$. Furthermore, because $\mathdutchcal{g}(\gamma^{*}_{\psi_1};\psi_1)=\mathcal{T}(\psi_1)$ for any $\psi_1\in[0,\frac{1}{2}]$, we obtain that 
\begin{align*}
    &\min_{\psi_1\in[0,\frac{1}{2}]}\inf_{\gamma\ge 0}\mathdutchcal{g}(\gamma;\psi)\\
    =&\min_{\psi_1\in[0,\frac{1}{2}]}\mathdutchcal{g}(\gamma^*_{\psi_1};\psi_1)\\
    =&\min_{\psi_1\in[0,\frac{1}{2}]}\mathcal{T}(\psi_1)\\
    =&\mathcal{T}(\nicefrac{1}{2}).
\end{align*}
Consequently, $(0,\frac{1}{2})$ minimizes $\mathdutchcal{g}(\gamma;\psi)$ over $[0,\gamma_0]\times [0,\frac{1}{2}]$, i.e., $\psi_1^*=\frac{1}{2}$ and $\gamma^*_{\psi_1^*}=0$.

{\textbf{Step 2:}} If $P_0\ge\varrho(\epsilon)$, then, by (\ref{appdx.pf_binary_case2}), $\mathdutchcal{g}(\gamma;\psi_1)$ is decreasing in $\gamma$ when $\gamma\le\gamma_0$, showing that the optimal value in (\ref{appdx.pf_binary_gamma*_1}) is
\begin{align}\label{appdx.pf_binary_gamma*_step2}
    \gamma^{*}_{\psi_1}=\gamma_0,\ \text{with}\ \gamma_0=\frac{\mathcal{T}(\psi_1)-\mathcal{T}(1-\psi_1)}{\kappa^p},
\end{align}
and thus,
\begin{align*}
    \mathdutchcal{g}(\gamma^{*}_{\psi_1};\psi_1)=(P_1+\varrho(\epsilon))\mathcal{T}(\psi_1)+(P_0-\varrho(\epsilon))\mathcal{T}(1-\psi_1)\ \text{for any}\ \psi_1\in[0,\nicefrac{1}{2}].
\end{align*}
Consequently, the derivative of $\mathdutchcal{g}(\gamma^*_{\psi_1};\psi_1)$ with respect to $\psi_1$ is
\begin{align*}
    \mathdutchcal{g}'_{\psi_1}(\gamma^{*}_{\psi_1};\psi_1)=(P_1+\varrho(\epsilon))\mathcal{T}'(\psi_1)-(P_0-\varrho(\epsilon))\mathcal{T}'(1-\psi_1),
\end{align*}
leading to $\mathdutchcal{g}'_{\psi_1}(\gamma^{*}_{\psi_1};\psi_1)\big|_{\psi_1=0}=(1-P_0+\varrho(\epsilon))\mathcal{T}'(0)-(P_0-\varrho(\epsilon))\mathcal{T}'(1)$ and $\mathdutchcal{g}'_{\psi_1}(\gamma^{*}_{\psi_1};\psi_1)\big|_{\psi_1=1/2}=(1+2\varrho(\epsilon)-2P_0)\mathcal{T}'(\frac{1}{2})$. Solving 
\begin{align*}
    \mathdutchcal{g}'_{\psi_1}(\gamma^{*}_{\psi_1};\psi_1)\big|_{\psi_1=0}=0\ \text{and}\ \mathdutchcal{g}'_{\psi_1}(\gamma^{*}_{\psi_1};\psi_1)\big|_{\psi_1=1/2}=0
\end{align*}
for $P_0$ leads to solutions
\begin{align*}
    P^{(1)}_0\triangleq \varrho(\epsilon)+\frac{\mathcal{T}'(0)}{\mathcal{T}'(1)+\mathcal{T}'(0)}\ \text{and}\ P^{(2)}_{0}\triangleq \varrho(\epsilon)+\frac{1}{2},
\end{align*}
respectively. 

Next, we identify $\psi_1^*$ and $\gamma^*_{\psi_1^*}$ by examining $P_0$ relative to $P^{(1)}_0$ and $P^{(2)}_0$, in combination with the convexity or concavity of function $\mathcal{T}$ by the following two steps.

\textit{\textbf{Step 2.1:}} Assume $\mathcal{T}$ is concave. Then by twice differentiability of $\mathcal{T}$, $\mathcal{T}''(\psi_1)\le 0$ for $\psi_1\in[0,1]$, leading to $\mathcal{T}'(1)\le\mathcal{T}'(0)<0$, and hence $P^{(1)}_0\le P^{(2)}_0$. Additionally, $\mathdutchcal{g}''_{\psi_1}(\gamma^{*}_{\psi_1};\psi_1)=(P_1+\varrho(\epsilon))\mathcal{T}''(\psi_1)+(P_0-\varrho(\epsilon))\mathcal{T}''(1-\psi_1)\le0$, and thus, $\mathdutchcal{g}'_{\psi_1}(\gamma^{*}_{\psi_1};\psi_1)$ is non-increasing in $\psi_1$ for $\psi_1\in[0,\frac{1}{2}]$.
\begin{itemize}
    \item If $\varrho(\epsilon)\le P_0\le P^{(1)}_0$, then $\mathdutchcal{g}'_{\psi_1}(\gamma^{*}_{\psi_1};\psi_1)\le0$ for $\psi_1\in[0,\frac{1}{2}]$, and thus, $\mathdutchcal{g}(\gamma^{*}_{\psi_1};\psi_1)$ is non-increasing in $\psi_1$ for $\psi_1\in[0,\frac{1}{2}]$. Therefore, $\inf_{0\le\psi_1\le1/2}\mathdutchcal{g}(\gamma^{*}_{\psi_1};\psi_1)=\mathdutchcal{g}(\gamma^{*}_{1/2};\frac{1}{2})=\mathcal{T}(\frac{1}{2})$. Thus, $\psi_1^*=\frac{1}{2}$ and $\gamma^*_{\psi_1^*}=0$ {by (\ref{appdx.pf_binary_gamma*_step2})}.
    \item If $P_0\ge P^{(2)}_{0}$, then $\mathdutchcal{g}'_{\psi_1}(\gamma^{*}_{\psi_1};\psi_1)\ge0$, showing that $\mathdutchcal{g}(\gamma^{*}_{\psi_1};\psi_1)$ is non-decreasing in $\psi_1$ for $\psi_1\in[0,\frac{1}{2}]$. Therefore, $\inf_{0\le\psi_1\le1/2}\mathdutchcal{g}(\gamma^{*}_{\psi_1};\psi_1)=\mathdutchcal{g}(\gamma^{*}_{0};0)=(P_1+\varrho(\epsilon))\mathcal{T}(0)+(P_0-\varrho(\epsilon))\mathcal{T}(1)$. Thus, $\psi_1^*=0$ and $\gamma^*_{\psi_1^*}=\frac{\mathcal{T}(0)-\mathcal{T}(1)}{\kappa^p}$ {by (\ref{appdx.pf_binary_gamma*_step2})}.
    \item If $P^{(1)}_0<P_0<P^{(2)}_0$, then $\mathdutchcal{g}'_{\psi_1}(\gamma^{*}_{\psi_1};\psi_1)$ is non-increasing in $\psi_1$ for $\psi_1\in[0,\frac{1}{2}]$ with $\mathdutchcal{g}'_{\psi_1}(\gamma^{*}_{\psi_1};\psi_1)\big|_{\psi_1=0}>0$ and $\mathdutchcal{g}'_{\psi_1}(\gamma^{*}_{\psi_1};\psi_1)\big|_{\psi_1=1/2}<0$. Therefore, $\mathdutchcal{g}'_{\psi_1}(\gamma^{*}_{\psi_1};\psi_1)=0$ has a unique solution on $[0,\frac{1}{2}]$, denoted $\psi_1^{\scriptscriptstyle\diamond}$, and furthermore, $\mathdutchcal{g}(\gamma^{*}_{\psi_1};\psi_1)$ is increasing in $\psi_1$ for $\psi_1\in[0,\psi_1^{\scriptscriptstyle\diamond}]$ and decreasing on $[\psi_1^{\scriptscriptstyle\diamond},\frac{1}{2}]$. Therefore, the infimum of $\mathdutchcal{g}_{\psi_1}(\gamma^{*}_{\psi_1};\psi_1)$ on $\psi_1\in[0,\frac{1}{2}]$ is taken at $\psi_1=0$ or $\psi_1=\frac{1}{2}$. If $\mathdutchcal{g}(\gamma^{*}_{1/2};\frac{1}{2})\le \mathdutchcal{g}(\gamma^{*}_{0};0)$, i.e., $P_0\le\varrho(\epsilon)+\frac{\mathcal{T}(0)-\mathcal{T}(1/2)}{\mathcal{T}(0)-\mathcal{T}(1)}$, the optimal value for $\psi_1$ in Case 1 is $\psi_1^*=\frac{1}{2}$ with $\gamma^*_{\psi_1^*}=0$; otherwise, the optimal value is $\psi_1^*=0$ with $\gamma^*_{\psi_1^*}=\frac{\mathcal{T}(0)-\mathcal{T}(1)}{\kappa^p}$ {by (\ref{appdx.pf_binary_gamma*_step2})}.
\end{itemize}

Summarizing the discussion in Step 2.1, we obtain that when $\psi_1\in[0,\frac{1}{2}]$ and $\mathcal{T}$ is concave,
\begin{itemize}
    \item[(i)] if $P_0>\varrho(\epsilon)+\frac{\mathcal{T}(0)-\mathcal{T}(1/2)}{\mathcal{T}(0)-\mathcal{T}(1)}$, $(\psi^*_1,\gamma^*_{\psi^*_1})$ in (\ref{appdx.pf_binary_three_cases}) is given by $\psi^*_1=0$ and $\gamma^*_{\psi^*_1}=\frac{\mathcal{T}(0)-\mathcal{T}(1)}{\kappa^p}$, yielding $\mathdutchcal{g}(\gamma^*_{\psi^*_1};\psi^*_1)=(P_1+\varrho(\epsilon))\mathcal{T}(0)+(P_0-\varrho(\epsilon))\mathcal{T}(1)$;
    \item[(ii)] otherwise, $\psi^*_1=\frac{1}{2}$ and $\gamma^*_{\psi^*_1}=0$, yielding $\mathdutchcal{g}(\gamma^*_{\psi^*_1};\psi^*_1)=\mathcal{T}(\frac{1}{2})$.
\end{itemize}

\textit{\textbf{Step 2.2:}} Assume $\mathcal{T}$ is convex. Then by twice differentiability of $\mathcal{T}$, $\mathcal{T}''(\psi_1)\ge 0$ for $\psi_1\in[0,1]$, leading to $\mathcal{T}'(0)\le\mathcal{T}'(1)<0$, and hence $P^{(2)}_0\le P^{(1)}_0$. Additionally, $\mathdutchcal{g}''_{\psi_1}(\gamma^{*}_{\psi_1};\psi_1)=(P_1+\varrho(\epsilon))\mathcal{T}''(\psi_1)+(P_0-\varrho(\epsilon))\mathcal{T}''(1-\psi_1)\ge0$, and thus, $\mathdutchcal{g}'_{\psi_1}(\gamma^{*}_{\psi_1};\psi_1)$ is non-decreasing in $\psi_1$ for $\psi_1\in[0,\frac{1}{2}]$.
\begin{itemize}
    \item If $P_0\ge P^{(1)}_0$, then $\mathdutchcal{g}'_{\psi_1}(\gamma^{*}_{\psi_1};\psi_1)\ge0$ for $\psi_1\in[0,\frac{1}{2}]$, and thus, $\mathdutchcal{g}(\gamma^{*}_{\psi_1};\psi_1)$ is non-decreasing in $\psi_1$ for $\psi_1\in[0,\frac{1}{2}]$. Therefore, $\inf_{0\le\psi_1\le1/2}\mathdutchcal{g}(\gamma^{*}_{\psi_1};\psi_1)=\mathdutchcal{g}(\gamma^{*}_{0};0)=(P_1+\varrho(\epsilon))\mathcal{T}(0)+(P_0-\varrho(\epsilon))\mathcal{T}(1)$. Thus, $\psi_1^*=0$ and $\gamma^*_{\psi_1^*}=\frac{\mathcal{T}(0)-\mathcal{T}(1)}{\kappa^p}$ by (\ref{appdx.pf_binary_gamma*_step2}).
    \item If $\varrho(\epsilon)\le P_0\le P^{(2)}_0$, then $\mathdutchcal{g}'_{\psi_1}(\gamma^{*}_{\psi_1};\psi_1)\le0$ for $\psi_1\in[0,\frac{1}{2}]$, and thus, $\mathdutchcal{g}(\gamma^{*}_{\psi_1};\psi_1)$ is non-increasing in $\psi_1$ for $\psi_1\in[0,\frac{1}{2}]$. Therefore, $\inf_{0\le\psi_1\le\frac{1}{2}}\mathdutchcal{g}(\gamma^{*}_{\psi_1};\psi_1)=\mathdutchcal{g}(\gamma^{*}_{1/2};\frac{1}{2})=\mathcal{T}(\frac{1}{2})$. Thus, $\psi_1^*=\frac{1}{2}$ and $\gamma^*_{\psi_1^*}=0$ by (\ref{appdx.pf_binary_gamma*_step2}). 
    \item If $P^{(2)}_0<P_0<P^{(1)}_0$, then $\mathdutchcal{g}'_{\psi_1}(\gamma^{*}_{\psi_1};\psi_1)$ is non-decreasing in $\psi_1$ for $\psi_1\in[0,\frac{1}{2}]$ with $\mathdutchcal{g}'_{\psi_1}(\gamma^{*}_{\psi_1};\psi_1)\big|_{\psi_1=0}<0$ and $\mathdutchcal{g}'_{\psi_1}(\gamma^{*}_{\psi_1};\psi_1)\big|_{\psi_1=1/2}>0$. Therefore, $\mathdutchcal{g}'_{\psi_1}(\gamma^{*}_{\psi_1};\psi_1)=0$ has a unique solution on $[0,\frac{1}{2}]$, denoted $\psi_1^{\scriptscriptstyle\diamond}$, and furthermore, $\mathdutchcal{g}(\gamma^{*}_{\psi_1};\psi_1)$ is decreasing in $\psi_1$ for $\psi_1\in[0,\psi_1^{\scriptscriptstyle\diamond}]$ and increasing on $[\psi_1^{\scriptscriptstyle\diamond},1/2]$. Then, the infimum of $\mathdutchcal{g}_{\psi_1}(\gamma^{*}_{\psi_1};\psi_1)$ on $\psi_1\in[0,\frac{1}{2}]$ is taken at $\psi_1=\psi_1^{\scriptscriptstyle\diamond}$, that is, $\inf_{0\le\psi_1\le1/2}\mathdutchcal{g}(\gamma^{*}_{\psi_1};\psi_1)=\mathdutchcal{g}(\gamma^{*}_{\psi_1^{\scriptscriptstyle\diamond}};\psi_1^{\scriptscriptstyle\diamond})=(P_1+\varrho(\epsilon))\mathcal{T}(\psi_1^{\scriptscriptstyle\diamond})+(P_0-\varrho(\epsilon))\mathcal{T}(1-\psi_1^{\scriptscriptstyle\diamond})$. Thus, $\psi_1^*=\psi_1^{\scriptscriptstyle\diamond}$ with $\gamma^*_{\psi_1^*}=\frac{\mathcal{T}(\psi_1^{\scriptscriptstyle\diamond})-\mathcal{T}(1-\psi_1^{\scriptscriptstyle\diamond})}{\kappa^p}$ by (\ref{appdx.pf_binary_gamma*_step2}).
\end{itemize}

\textbf{Case 2: $\psi_2\in[\frac{1}{2},1]$.} 

In this case, we set $\overline{\psi}_2\triangleq1-\psi_2$, yielding $\overline{\psi}_2\in[0,\frac{1}{2}]$, and the objective function $\mathdutchcal{g}(\gamma;\psi_2)$ defined in (\ref{appdx.pf_binary_obj_fun}) can be written as 
\begin{align*}
    \mathdutchcal{g}(\gamma;\psi_2)=&\gamma\epsilon^p+P_0\max\{\mathcal{T}(1-\psi_2), \mathcal{T}(\psi_2)-\gamma\kappa^p\}+P_1\max\{\mathcal{T}(1-\psi_2)-\gamma\kappa^p,\mathcal{T}(\psi_2)\}\\
    =&\gamma\epsilon^p+P_1\max\{\mathcal{T}(1-\overline{\psi}_2),\mathcal{T}(\overline{\psi}_2)-\gamma\kappa^p\}+P_0\max\{\mathcal{T}(1-\overline{\psi}_2)-\gamma\kappa^p,\mathcal{T}(\overline{\psi}_2)\}.
\end{align*}
Hence, the derivation in Case 1 for any $\psi_1$ in $[0,\frac{1}{2}]$ can be applied to $\overline{\psi}_2$ by modifying the derivations based on the range of $P_0$ to be that for $P_1$, as outlined below.
\begin{itemize}
    \item {\textbf{Step 1:}} If $P_1<\varrho(\epsilon)$, then following the results for $\psi^*_1$ and $\gamma^*_{\psi^*_1}$ in Case 1, with only $\psi^*_1$, $P_1$, and $P_0$ there replaced by $\overline{\psi}_2^*$, $P_0$, and $P_1$, respectively, we obtain that $\overline{\psi}_2^*=\frac{1}{2}$ and $\gamma^*_{\overline{\psi}^*_2}=0$. Hence, $\psi_2^*$ is taken as $1-\overline{\psi}_2^*=\frac{1}{2}$ and $\gamma^*_{\psi^*_2}=0$, yielding $\mathdutchcal{g}(\gamma^*_{\psi^*_2};\psi^*_2)=\mathdutchcal{g}(0;\frac{1}{2})=\mathcal{T}(\frac{1}{2})$. 
    \item {\textbf{Step 2:}} If $P_1\ge\varrho(\epsilon)$, then, by (\ref{appdx.pf_binary_gamma*_step2}), $\gamma^*_{\psi^*_2}$ is set as $\gamma^*_{\overline{\psi}^*_2}=\frac{\mathcal{T}(\overline{\psi}_2^*)-\mathcal{T}(1-\overline{\psi}_2^*)}{\kappa^p}=\frac{\mathcal{T}(1-\psi_2^*)-\mathcal{T}(\psi_2^*)}{\kappa^p}$.
    \item \textit{\textbf{Step 2.1:}} Assume $\mathcal{T}$ is concave. We can directly derive the following result from the summary in Step 2.1 of Case 1.
    \begin{itemize}
        \item If $P_1\le\varrho(\epsilon)+\frac{\mathcal{T}(0)-\mathcal{T}(1/2)}{\mathcal{T}(0)-\mathcal{T}(1)}$, then $\overline{\psi}_2^*=\frac{1}{2}$ and $\gamma^*_{\psi_2^*}=0$. Hence, $\psi_2^*=1-\overline{\psi}_2^*=\frac{1}{2}$, and $ \mathdutchcal{g}(\gamma^*_{\psi^*_2};\psi^*_2)=\mathdutchcal{g}(0;\frac{1}{2})=\mathcal{T}(\frac{1}{2})$.
        \item If $P_1>\varrho(\epsilon)+\frac{\mathcal{T}(0)-\mathcal{T}(1/2)}{\mathcal{T}(0)-\mathcal{T}(1)}$, then $\overline{\psi}_2^*=0$ and $\gamma^*_{\psi^*_2}=\frac{\mathcal{T}(0)-\mathcal{T}(1)}{\kappa^p}$. Hence, $\psi_2^*=1-\overline{\psi}_2^*=1$, and $\mathdutchcal{g}(\gamma^*_{\psi^*_2};\psi^*_2)=\mathdutchcal{g}(\frac{\mathcal{T}(0)-\mathcal{T}(1)}{\kappa^p};1)=(P_0+\varrho(\epsilon))\mathcal{T}(0)+(P_1-\varrho(\epsilon))\mathcal{T}(1)$.
    \end{itemize}
    \item \textit{\textbf{Step 2.2:}} Assume $\mathcal{T}$ is convex. From the results on $\psi^*_1$ and $\gamma^*_{\psi^*_1}$ in Step 2.2 of Case 1, we obtain the following conclusion.
    \begin{itemize}
        \item If $P_1\ge \varrho(\epsilon)+\frac{\mathcal{T}'(0)}{\mathcal{T}'(1)+\mathcal{T}'(0)}$, then $\overline{\psi}_2^*=0$ and $\gamma^*_{\psi_2^*}=\frac{\mathcal{T}(0)-\mathcal{T}(1)}{\kappa^p}$. Hence, $\psi_2^*=1-\overline{\psi}_2^*=1$, and $ \mathdutchcal{g}(\gamma^*_{\psi^*_2};\psi^*_2)=\mathdutchcal{g}(\frac{\mathcal{T}(0)-\mathcal{T}(1)}{\kappa^p};1)=(P_0+\varrho(\epsilon))\mathcal{T}(0)+(P_1-\varrho(\epsilon))\mathcal{T}(1)$.
        \item If $\varrho(\epsilon)\le P_1\le \varrho(\epsilon)+\frac{1}{2}$, then $\overline{\psi}_2^*=\frac{1}{2}$ and $\gamma^*_{\psi_2^*}=0$. Hence, $\psi_2^*=1-\overline{\psi}_2^*=\frac{1}{2}$, and $\mathdutchcal{g}(\gamma^*_{\psi^*_2};\psi^*_2)=\mathdutchcal{g}(0;\frac{1}{2})=\mathcal{T}(\frac{1}{2})$.
        \item If $\varrho(\epsilon)+\frac{1}{2}<P_1<\varrho(\epsilon)+\frac{\mathcal{T}'(0)}{\mathcal{T}'(1)+\mathcal{T}'(0)}$, then $\overline{\psi}_2^*=\overline{\psi}_2^{\scriptscriptstyle\diamond}$ and $\gamma^*_{\psi_2^*}=\frac{\mathcal{T}(\overline{\psi}_2^{\scriptscriptstyle\diamond})-\mathcal{T}(1-\overline{\psi}_2^{\scriptscriptstyle\diamond})}{\kappa^p}$, where $\overline{\psi}_2^{\scriptscriptstyle\diamond}$ is the unique solution to $(P_0+\varrho(\epsilon))\mathcal{T}'(\overline{\psi}_2)-(P_1-\varrho(\epsilon))\mathcal{T}'(1-\overline{\psi}_2)=0$ on $[0,\frac{1}{2}]$. Hence,  ${\psi}_2^*={\psi}_2^{\scriptscriptstyle\diamond}$ and $\gamma^*_{\psi_2^*}=\frac{\mathcal{T}(1-{\psi}_2^{\scriptscriptstyle\diamond})-\mathcal{T}({\psi}_2^{\scriptscriptstyle\diamond})}{\kappa^p}$, where ${\psi}_2^{\scriptscriptstyle\diamond}=1-\overline{\psi}_2^{\scriptscriptstyle\diamond}$ is the unique solution to $-(P_0+\varrho(\epsilon))\mathcal{T}'(1-{\psi}_2)+(P_1-\varrho(\epsilon))\mathcal{T}'({\psi}_2)=0$ on $[\frac{1}{2},1]$. Then $\mathdutchcal{g}(\gamma^*_{\psi^*_2};\psi^*_2)=(P_0+\varrho(\epsilon))\mathcal{T}(1-\psi^{\scriptscriptstyle\diamond}_2)+(P_1-\varrho(\epsilon))\mathcal{T}(\psi^{\scriptscriptstyle\diamond}_2)$.
    \end{itemize}
\end{itemize}

In summary, we present the derived results in Tables \ref{table.robust_risk_concave} and \ref{table.robust_risk_convex} for the scenarios where $\mathcal{T}$ is concave and convex, respectively.

\begin{table}[h]
\scriptsize
    \centering
    \begin{tabular}{cc|ccc}
         \hline
         &  & $\psi^*_j$ & $\gamma^*_{\psi^*_j}$ & robust risk $\mathdutchcal{g}(\gamma^*_{\psi^*_j};\psi^*_j)$\\
         \hline
         Case 1 & $P_0\ge\varrho(\epsilon)+\frac{\mathcal{T}(0)-\mathcal{T}(1/2)}{\mathcal{T}(0)-\mathcal{T}(1)}$ & $\psi^*_1=0$ & $\gamma^*_{\psi^*_1}=\frac{\mathcal{T}(0)-\mathcal{T}(1)}{\kappa^p}$ & $(P_1+\varrho(\epsilon))\mathcal{T}(0)+(P_0-\varrho(\epsilon))\mathcal{T}(1)$  \\
         & $P_0<\varrho(\epsilon)+\frac{\mathcal{T}(0)-\mathcal{T}(1/2)}{\mathcal{T}(0)-\mathcal{T}(1)}$ & $\psi^*_1=\frac{1}{2}$ & $\gamma^*_{\psi^*_1}=0$ & $\mathcal{T}(\frac{1}{2})$ \\
         \hline
         Case 2 & $P_1\ge\varrho(\epsilon)+\frac{\mathcal{T}(0)-\mathcal{T}(1/2)}{\mathcal{T}(0)-\mathcal{T}(1)}$ & $\psi^*_2=1$ & $\gamma^*_{\psi^*_2}=\frac{\mathcal{T}(0)-\mathcal{T}(1)}{\kappa^p}$ & $(P_0+\varrho(\epsilon))\mathcal{T}(0)+(P_1-\varrho(\epsilon))\mathcal{T}(1)$  \\
         & $P_1<\varrho(\epsilon)+\frac{\mathcal{T}(0)-\mathcal{T}(1/2)}{\mathcal{T}(0)-\mathcal{T}(1)}$ & $\psi^*_2=\frac{1}{2}$ & $\gamma^*_{\psi^*_2}=0$ & $\mathcal{T}(\frac{1}{2})$ \\
         \hline
    \end{tabular}
    \caption{Summarized results in two cases when $\mathcal{T}$ is concave.}
    \label{table.robust_risk_concave}
\end{table}

\begin{table}[h]
\scriptsize
\setlength{\tabcolsep}{2pt}
    \centering
    \begin{tabular}{cc|ccc}
         \hline
         &  & $\psi^*_j$ & $\gamma^*_{\psi^*_j}$ & robust risk $\mathdutchcal{g}(\gamma^*_{\psi^*_j};\psi^*_j)$\\
         \hline
         Case 1 & $P_0\ge \varrho(\epsilon)+\frac{\mathcal{T}'(0)}{\mathcal{T}'(0)+\mathcal{T}'(1)}$ & $\psi^*_1=0$ & $\gamma^*_{\psi^*_1}=\frac{\mathcal{T}(0)-\mathcal{T}(1)}{\kappa^p}$ & $(P_1+\varrho(\epsilon))\mathcal{T}(0)+(P_0-\varrho(\epsilon))\mathcal{T}(1)$  \\
         & $P_0\le\varrho(\epsilon)+\frac{1}{2}$ & $\psi^*_1=\frac{1}{2}$ & $\gamma^*_{\psi^*_1}=0$ & $\mathcal{T}(\frac{1}{2})$ \\
         & $\varrho(\epsilon)+\frac{1}{2}<P_0<\varrho(\epsilon)+\frac{\mathcal{T}'(0)}{\mathcal{T}'(0)+\mathcal{T}'(1)}$ & $\psi^*_1=\psi^{\scriptscriptstyle\diamond}_1$ & $\gamma^*_{\psi^*_1}=\frac{\mathcal{T}(\psi^{\scriptscriptstyle\diamond}_1)-\mathcal{T}(1-\psi^{\scriptscriptstyle\diamond}_1)}{\kappa^p}$ & $(P_1+\varrho(\epsilon))\mathcal{T}(\psi^{\scriptscriptstyle\diamond}_1)+(P_0-\varrho(\epsilon))\mathcal{T}(1-\psi^{\scriptscriptstyle\diamond}_1)$ \\
         \hline
         Case 2 & $P_1\ge \varrho(\epsilon)+\frac{\mathcal{T}'(0)}{\mathcal{T}'(0)+\mathcal{T}'(1)}$ & $\psi^*_2=1$ & $\gamma^*_{\psi^*_2}=\frac{\mathcal{T}(0)-\mathcal{T}(1)}{\kappa^p}$ & $(P_0+\varrho(\epsilon))\mathcal{T}(0)+(P_1-\varrho(\epsilon))\mathcal{T}(1)$  \\
         & $P_1\le\varrho(\epsilon)+\frac{1}{2}$ & $\psi^*_2=\frac{1}{2}$ & $\gamma^*_{\psi^*_2}=0$ & $\mathcal{T}(\frac{1}{2})$ \\
         & $\varrho(\epsilon)+\frac{1}{2}<P_1<\varrho(\epsilon)+\frac{\mathcal{T}'(0)}{\mathcal{T}'(0)+\mathcal{T}'(1)}$ & $\psi^*_2=\psi^{\scriptscriptstyle\diamond}_2$ & $\gamma^*_{\psi^*_2}=\frac{\mathcal{T}(1-\psi^{\scriptscriptstyle\diamond}_2)-\mathcal{T}(\psi^{\scriptscriptstyle\diamond}_2)}{\kappa^p}$ & $(P_0+\varrho(\epsilon))\mathcal{T}(\psi^{\scriptscriptstyle\diamond}_2)+(P_1-\varrho(\epsilon))\mathcal{T}(1-\psi^{\scriptscriptstyle\diamond}_2)$ \\
         \hline
    \end{tabular}
    \caption{Summarized results in two cases when $\mathcal{T}$ is convex.}
    \label{table.robust_risk_convex}
\end{table}

Finally, for any given input $\mathbf{x}$, applying the preceding results to the optimal solution $(\gamma^*_{\psi^*},\psi^*)=\arg\inf_{\psi\in[0,1]}\inf_{\gamma\ge 0}\mathdutchcal{g}(\gamma,\psi)$ in (\ref{appdx.pf_binary_three_cases}), we obtain that the optimal action on the given instance $\mathbf{x}$ is given as below: for concave $\mathcal{T}$,
\begin{align*}
    \psi^{\star}(\mathbf{x})
    =& \left\{
    \begin{aligned}
        &\ 0,\ \text{if}\ P_0\ge\varrho(\epsilon)+\frac{\mathcal{T}(0)-\mathcal{T}(1/2)}{\mathcal{T}(0)-\mathcal{T}(1)};\\
        &\ 1,\ \text{if}\ P_1\ge\varrho(\epsilon)+\frac{\mathcal{T}(0)-\mathcal{T}(1/2)}{\mathcal{T}(0)-\mathcal{T}(1)};\\
        &\ 1/2, \ \text{otherwise};
    \end{aligned}
    \right.
\end{align*}
and for convex $\mathcal{T}$,
\begin{align*}
    \psi^{\star}(\mathbf{x})
    =& \left\{
    \begin{aligned}
        &\ 0,\ \text{if}\ P_0\ge \varrho(\epsilon)+\frac{\mathcal{T}'(0)}{\mathcal{T}'(0)+\mathcal{T}'(1)};\\
        &\ \mathdutchcal{t}^*_0,\ \text{if}\ \varrho(\epsilon)+1/2< P_0< \varrho(\epsilon)+\frac{\mathcal{T}'(0)}{\mathcal{T}'(0)+\mathcal{T}'(1)};\\
        &\ 1,\ \text{if}\ P_1\ge\ \varrho(\epsilon)+\frac{\mathcal{T}'(0)}{\mathcal{T}'(0)+\mathcal{T}'(1)};\\
        &\ \mathdutchcal{t}^*_1,\ \text{if}\ \varrho(\epsilon)+1/2< P_1< \varrho(\epsilon)+\frac{\mathcal{T}'(0)}{\mathcal{T}'(0)+\mathcal{T}'(1)};\\
        &\ 1/2, \ \text{otherwise},
    \end{aligned}\right.
\end{align*}
where $\mathdutchcal{t}^*_0$ is the unique solution of $(P_0-\varrho(\epsilon))\mathcal{T}'(1-\mathdutchcal{t})=(P_1+\varrho(\epsilon))\mathcal{T}'(\mathdutchcal{t})$ on $\mathdutchcal{t}\in(0,\frac{1}{2})$, and $\mathdutchcal{t}^*_1$ is the unique solution of $(P_0+\varrho(\epsilon))\mathcal{T}'(1-\mathdutchcal{t})=(P_1-\varrho(\epsilon))\mathcal{T}'(\mathdutchcal{t})$ on $\mathdutchcal{t}\in(\frac{1}{2}, 1)$. Hence, the proof is established.

\subsection{Proof of Theorem \ref{thm.optimal_action_multi}}\label{appdx.pf_thm_optimal_action_multi}

For ease of presentation, we omit the dependence on $\mathbf{x}$ and $\widetilde{\mathbf{y}}$ in the notation for now. Specifically, for $j\in[K]$, we let $P_j\triangleq P_j(\mathbf{x}, \widetilde{\mathbf{y}})\triangleq P(\mathrm{Y}=j|\mathbf{x}, \widetilde{\mathbf{y}})$ and $\psi_j\triangleq\psi(\mathbf{x})_j$. Let the objective function in (\ref{eq.multi_problem}) be denoted as
\begin{align}\label{appdx.pf_optimal_multi_g}
    \mathdutchcal{g}(\gamma;\psi)\triangleq \gamma\epsilon^p+\sum_{j=1}^{K}P_j&\max\{1-\psi_1-\gamma\kappa^p,\ldots,1-\psi_{j-1}-\gamma\kappa^p,\notag\\
    &1-\psi_j,1-\psi_{j+1}-\gamma\kappa^p,\ldots,1-\psi_K-\gamma\kappa^p\}.
\end{align}

We complete the proof in four steps. In Step 1, for each given $\psi$, we investigate the \emph{inner optimization problem} in (\ref{eq.multi_problem}) by finding the optimal value of $\gamma$, defined as $\gamma^{\star}_{\psi}\triangleq\arg\min_{\gamma\ge0}\mathdutchcal{g}(\gamma;\psi)$. Then, in Step 2, by substituting $\gamma^{\star}_{\psi}$ into $\mathdutchcal{g}(\gamma;\psi)$, we find that the \emph{outer optimization problem} in (\ref{eq.multi_problem}) can be written in a linear programming format under certain transformations. Next, in Step 3, we find the extreme points of the associated linear programming, and finally, in Step 4, we obtain the solution format of the optimal action $\psi^\star$.

\paragraph{Step 1: For any $\psi\in\Psi$, finding the optimal value of $\gamma$, defined as $\gamma^{\star}_{\psi}\triangleq\arg\min_{\gamma\ge0}\mathdutchcal{g}(\gamma;\psi)$.} \mbox{}\\
Given $\psi$ and $\mathbf{x}$, we sort $\{\psi_1,\ldots,\psi_K\}$ in an decreasing order, denoted $\psi^{(1)}\ge\ldots\ge\psi^{(K)}$, and hence, $1-\psi^{(1)}\le\ldots\le1-\psi^{(K)}$. Assume that $\{\psi^{(1)},\ldots,\psi^{(K)}\}$ corresponds to $\{\psi_1,\ldots,\psi_K\}$ via a  permutation $\chi$, that is, $\psi^{(j)}=\psi_{\chi(j)}$ for $j\in[K]$. Correspondingly, the $P_j$'s with the associated indexes are denoted $P^{(j)}\triangleq P_{\chi(j)}$ for $j\in[K]$. Then, for the $\chi(j)$-th element in the summation of (\ref{appdx.pf_optimal_multi_g}), the maximum is taken between $1-\psi^{(K)}-\gamma\kappa^p$ and $1-\psi^{(j)}$.

First, for given $\psi$, we examine the continuity of $\mathdutchcal{g}(\gamma;\psi)$ in $\gamma$ by eliminating the $\max$ operators in (\ref{appdx.pf_optimal_multi_g}), which is conducted by comparing $1-\psi^{(K)}-\gamma\kappa^p$ and $1-\psi^{(j)}$ for $j\in[K]$ as follows.

If $1-\psi^{(1)}\ge1-\psi^{(K)}-\gamma\kappa^p$, i.e., $\gamma\ge\frac{\psi^{(1)}-\psi^{(K)}}{\kappa^p}$, then $1-\psi_j\ge 1-\psi^{(1)}\ge1-\psi^{(K)}-\gamma\kappa^p\ge 1-\psi^{(j')}-\gamma\kappa^p$ for $j,j'\in[K]$, and hence, (\ref{appdx.pf_optimal_multi_g}) becomes
\begin{align}\label{appdx.pf_optimal_multi_g1}
    \mathdutchcal{g}(\gamma;\psi)=\gamma\epsilon^p+\sum_{j=1}^{K}P_j (1-\psi_j),
\end{align}
which is continuous in $\gamma$ for $\gamma\ge\frac{\psi^{(1)}-\psi^{(K)}}{\kappa^p}$. 

On the other hand, if $1-\psi^{(1)}<1-\psi^{(K)}-\gamma\kappa^p$, i.e., $0\le\gamma<\frac{\psi^{(1)}-\psi^{(K)}}{\kappa^p}$, then we express the range of $\gamma$ as: 
\begin{align*}
    \bigg[0, \frac{\psi^{(1)}-\psi^{(K)}}{\kappa^p}\bigg)=\cup_{s\in[K-1]}\bigg[\frac{\psi^{(s+1)}-\psi^{(K)}}{\kappa^p},\frac{\psi^{(s)}-\psi^{(K)}}{\kappa^p}\bigg).
\end{align*}
Then we consider $\gamma$ in each interval $\Big[\frac{\psi^{(s+1)}-\psi^{(K)}}{\kappa^p},\frac{\psi^{(s)}-\psi^{(K)}}{\kappa^p}\Big)$ for $s\in[K-1]$. In this case, $1-\psi^{(s)}<1-\psi^{(K)}-\gamma\kappa^p\le 1-\psi^{(s+1)}$, and (\ref{appdx.pf_optimal_multi_g}) becomes
\begin{align}\label{appdx.pf_optimal_multi_g2}
    \mathdutchcal{g}(\gamma;\psi)&=\gamma\epsilon^p+\sum_{j=1}^{s}P^{(j)} (1-\psi^{(K)}-\gamma\kappa^p)+\sum_{j=s+1}^{K}P^{(j)}(1-\psi^{(j)})\notag\\
    &=\sum_{j=1}^{s}P^{(j)} (1-\psi^{(K)})+\sum_{j=s+1}^{K}P^{(j)}(1-\psi^{(j)})+\gamma\kappa^p\Big\{\varrho(\epsilon)-\sum_{j=1}^{s}P^{(j)}\Big\},
\end{align}
where the last step holds by re-arranging the arguments and using the definition of $\varrho(\epsilon)$ given after (\ref{eq.binary_problem}). Consequently,
\begin{align*}
    &\lim_{\gamma\rightarrow((\psi^{(s)}-\psi^{(K)})/\kappa^p)-}\mathdutchcal{g}(\gamma;\psi)\\
    =&\sum_{j=1}^{s}P^{(j)} (1-\psi^{(K)})+\sum_{j=s+1}^{K}P^{(j)}(1-\psi^{(j)})+\big(\psi^{(s)}-\psi^{(K)}\big)\Big\{\varrho(\epsilon)-\sum_{j=1}^{s}P^{(j)}\Big\}\\
    =&\sum_{j=1}^{s}P^{(j)} (1-\psi^{(K)})+\sum_{j=s+1}^{K}P^{(j)}(1-\psi^{(j)})+\big(\psi^{(s)}-\psi^{(K)}\big)\Big\{\varrho(\epsilon)-\sum_{j=1}^{s-1}P^{(j)}\Big\}\\
    &-P^{(s)}\big\{(1-\psi^{(K)})-(1-\psi^{(s)})\big\}\\
    =&\sum_{j=1}^{s-1}P^{(j)} (1-\psi^{(K)})+\sum_{j=s}^{K}P^{(j)}(1-\psi^{(j)})+\big(\psi^{(s)}-\psi^{(K)}\big)\Big\{\varrho(\epsilon)-\sum_{j=1}^{s-1}P^{(j)}\Big\}\\
    =&\mathdutchcal{g}((\psi^{(s)}-\psi^{(K)})/\kappa^p;\psi),
\end{align*}
where the last step comes from the expression (\ref{appdx.pf_optimal_multi_g2}) for $\mathdutchcal{g}(\gamma;\psi)$ when $\frac{\psi^{(s)}-\psi^{(K)}}{\kappa^p}\le\gamma<\frac{\psi^{(s-1)}-\psi^{(K)}}{\kappa^p}$. Thus, $\mathdutchcal{g}(\gamma;\psi)$ is continuous in $\gamma$ for $\gamma\in\left[\frac{\psi^{(s+1)}-\psi^{(K)}}{\kappa^p},\frac{\psi^{(s)}-\psi^{(K)}}{\kappa^p}\right]$ with $s\in[K-1]$. Consequently, $\mathdutchcal{g}(\gamma;\psi)$ is continuous in $\gamma$ for  $0\le\gamma\le \frac{\psi^{(1)}-\psi^{(K)}}{\kappa^p}$.

Therefore, combining the discussion regrading (\ref{appdx.pf_optimal_multi_g1}) and(\ref{appdx.pf_optimal_multi_g2}), we obtain that given $\psi$, $\mathdutchcal{g}(\gamma;\psi)$ is continuous in $\gamma$ for $\gamma\ge 0$.

Next, for each given $\psi$, we examine the monotonicity of $\mathdutchcal{g}(\gamma;\psi)$ in $\gamma$ to find the $\gamma$ that minimizes $\mathdutchcal{g}(\gamma;\psi)$. To this end, we consider the following three cases by the values of $\varrho(\epsilon)$.

{{\textbf{Case 1:}}} If $P^{(1)}<\varrho(\epsilon)<\sum_{j=1}^{K}P^{(j)}$, 

then there exists an $s^*\in\{2,\ldots, K\}$ such that $\sum_{j=1}^{s^*-1}P^{(j)}\le\varrho(\epsilon)\le\sum_{j=1}^{s^*}P^{(j)}$. Then, by (\ref{appdx.pf_optimal_multi_g2}), $\mathdutchcal{g}(\gamma;\psi)$ is decreasing in $\gamma$ for $\gamma\in[0,\frac{\psi^{(s^*)}-\psi^{(K)}}{\kappa^p}]$ and increasing for $\gamma\in[\frac{\psi^{(s^*)}-\psi^{(K)}}{\kappa^p},\frac{\psi^{(1)}-\psi^{(K)}}{\kappa^p}]$; and by (\ref{appdx.pf_optimal_multi_g1}), $\mathdutchcal{g}(\gamma;\psi)$ is increasing in $\gamma$ for $\gamma\ge \frac{\psi^{(1)}-\psi^{(K)}}{\kappa^p}$. Therefore, $\gamma^\star_\psi=\frac{\psi^{(s^*)}-\psi^{(K)}}{\kappa^p}$.

{{\textbf{Case 2:}}} If $\varrho(\epsilon)\le P^{(1)}$, 

then by (\ref{appdx.pf_optimal_multi_g2}), $\mathdutchcal{g}(\gamma;\psi)$ is decreasing in $\gamma$ for $\gamma\in[0,\frac{\psi^{(1)}-\psi^{(K)}}{\kappa^p}]$; and by (\ref{appdx.pf_optimal_multi_g1}), increasing in $\gamma$ for $\gamma\ge \frac{\psi^{(1)}-\psi^{(K)}}{\kappa^p}$. Therefore, $\gamma^\star_\psi=\frac{\psi^{(1)}-\psi^{(K)}}{\kappa^p}$.

{{\textbf{Case 3:}}} If $\varrho(\epsilon)\ge\sum_{j=1}^{K}P^{(j)}$, 

then by (\ref{appdx.pf_optimal_multi_g2}), $\mathdutchcal{g}(\gamma;\psi)$ is increasing in $\gamma$ for $\gamma\in[0,\frac{\psi^{(1)}-\psi^{(K)}}{\kappa^p}]$; and by (\ref{appdx.pf_optimal_multi_g1}), increasing in $\gamma$ for $\gamma\ge \frac{\psi^{(1)}-\psi^{(K)}}{\kappa^p}$. Therefore, $\gamma^\star_\psi=0$.

Therefore, we conclude that 
\begin{align*}
    \gamma^\star_\psi=\left\{
    \begin{aligned}
        &\frac{\psi^{(1)}-\psi^{(K)}}{\kappa^p}\ \ \text{if}\ \varrho(\epsilon)\le P^{(1)}\\
        &\frac{\psi^{(s^*)}-\psi^{(K)}}{\kappa^p}\ \ \text{if}\ \sum_{j=1}^{s^*-1}P^{(j)}\le\varrho(\epsilon)\le\sum_{j=1}^{s^*}P^{(j)}\ \text{with}\ s^*\in\{2,\ldots,K\}\\
        &0\ \ \text{if}\ \varrho(\epsilon)\ge\sum_{j=1}^{K}P^{(j)}.
    \end{aligned}
    \right.
\end{align*}


\paragraph{Step 2: Linear programming format.} \mbox{}\\
For \emph{each fixed permutation $\chi$}, we now find the optimal $\psi$ that minimizes $\mathdutchcal{g}(\gamma^{\star}_{\psi};\psi)$ by examining the three cases in Step 1.

In Case 3 of Step 1, $\mathdutchcal{g}(\gamma^{\star}_{\psi};\psi)=\mathdutchcal{g}(0;\psi)=1-\psi^{(K)}\ge1-1/K$. Then the corresponding optimal action is $\psi^{(1)}=\ldots=\psi^{(K)}=1/K$.

In Case 1 of Step 1, for a single data point $(\mathbf{x},\widetilde{\mathbf{y}})$, by substituting $\gamma^\star_{\psi}=\frac{\psi^{(s^*)}-\psi^{(K)}}{\kappa^p}$ with $s^*\in\{2,\ldots,K\}$ into (\ref{appdx.pf_optimal_multi_g2}), we obtain that
\begin{align}\label{appdx.pf_optimal_multi_form2}
    \mathdutchcal{g}(\gamma^{\star}_{\psi};\psi)=&\sum_{j=1}^{s^*-1}P^{(j)} (1-\psi^{(K)})+\sum_{j=s^*}^{K}P^{(j)}(1-\psi^{(j)})+(\psi^{(s^*)}-\psi^{(K)})\Big\{\varrho(\epsilon)-\sum_{j=1}^{s^*-1}P^{(j)}\Big\}\notag\\
    =&(1-\psi^{(K)})\sum_{j=1}^{s^*-1}P^{(j)}+P^{(s^*)}(1-\psi^{(s^*)})+P^{(K)}(1-\psi^{(K)})+\sum_{j=s^*+1}^{K-1}P^{(j)}(1-\psi^{(j)})\mathbf{1}(s^*<K-1)\notag\\
    &+(1-\psi^{(K)})\Big\{\varrho(\epsilon)-\sum_{j=1}^{s^*-1}P^{(j)}\Big\}-(1-\psi^{(s^*)})\Big\{\varrho(\epsilon)-\sum_{j=1}^{s^*-1}P^{(j)}\Big\}\notag\\
    =&\Big\{\sum_{j=1}^{s^*}P^{(j)}-\varrho(\epsilon)\Big\}(1-\psi^{(s^*)})+\sum_{j=s^*+1}^{K-1}P^{(j)}(1-\psi^{(j)})\mathbf{1}(s^*<K-1)\notag\\
    &+\Big\{P^{(K)}+\varrho(\epsilon)\Big\}(1-\psi^{(K)}).
\end{align}


To find the optimal value that minimizes $\mathdutchcal{g}(\gamma^{\star}_{\psi};\psi)$ in (\ref{appdx.pf_optimal_multi_form2}), we link it with a linear programming problem. Specifically, for $j\in[K]$, let $z_j\triangleq 1-\psi^{(j)}$ and $\mathbf{z}\triangleq(z_1,\ldots,z_K)^\top$. Define 
\begin{align*}
    a_j=\left\{
    \begin{aligned}
        &\sum_{j=1}^{s^*}P^{(j)}-\varrho(\epsilon)\ \ \text{if}\ j=s^*\\
        &P^{(K)}+\varrho(\epsilon)\hspace{0.7cm} \text{if}\ j=K\\
        &P^{(j)}\hspace{1.86cm} \text{if}\ s^*<j<K\ \text{when}\ s^*\ne K-1.
    \end{aligned}
    \right.
\end{align*}
When $s^*=K-1$, only the entries for $j=s^*=K-1$ and $j=K$ need to be considered. Let $\mathsf{V}(\mathbf{z})=\sum_{j=s^*}^{K}a_j z_j$. Then, the optimal $\psi$ that minimizes $\mathdutchcal{g}(\gamma^*_\psi;\psi)$ in (\ref{appdx.pf_optimal_multi_form2}) can be derived by solving the linear programming problem:
\begin{equation}\label{appdx.pf_optimal_multi_LPP_1}
    \left\{
    \begin{aligned}
        \min_{z_1,\ldots,z_K}&\ \mathsf{V}(\mathbf{z}),\\
        s.t.\ \ &\sum_{j=1}^{K}(1-z_j)=1,\\
        &\ 0\le z_1\le\ldots\le z_K\le 1,
    \end{aligned}
    \right.
\end{equation}
where the constraint $\sum_{j=1}^{K}(1-z_j)=1$ is due to $\sum_{j=1}^{K}(1-z_j)=\sum_{j=1}^{K}\psi^{(j)}=1$ by the definitions of $\mathbf{z}$ and $\psi$, and the constraint $0\le z_1\le\ldots\le z_K\le 1$ reflects the definition of $\psi^{(j)}$ for $j\in[K]$.

Similarly, in Case 2 of Step 1, by substituting $\gamma^\star_{\psi}=\frac{\psi^{(1)}-\psi^{(K)}}{\kappa^p}$ into (\ref{appdx.pf_optimal_multi_g1}), we obtain that
\begin{align*}
    \mathdutchcal{g}(\gamma^\star_\psi;\psi)=&\varrho(\epsilon)\{\psi^{(1)}-\psi^{(K)}\}+\sum_{j=1}^{K}P_j (1-\psi_j)\notag\\
    =&\varrho(\epsilon)\left[\{1-\psi^{(K)}\}-\{1-\psi^{(1)}\}\right]+\sum_{j=1}^{K}P_j (1-\psi_j)\notag\\
    =&\{P^{(1)}-\varrho(\epsilon)\}(1-\psi^{(1)})+\sum_{j=2}^{K-1}P^{(j)}(1-\psi^{(j)})+\{P^{(K)}+\varrho(\epsilon)\}(1-\psi^{(K)}),
\end{align*}
which is a form similar to (\ref{appdx.pf_optimal_multi_form2}) if letting $s^*$ in (\ref{appdx.pf_optimal_multi_form2}) equal 1. Hence, its optimal minimizer can be found through a linear programming problem similar to (\ref{appdx.pf_optimal_multi_LPP_1}). Consequently, in the next step, our discussion focuses on (\ref{appdx.pf_optimal_multi_form2}) only.

\paragraph{Step 3: Extreme points.} \mbox{}\\
The feasible region of (\ref{appdx.pf_optimal_multi_LPP_1}), denoted ${\Xi}$, can be expressed as follows:
\begin{align}\label{appdx.pf_optimal_multi_feasible}
    {\Xi}\triangleq&\Big\{\mathbf{z}:\sum_{j=1}^{K}(1-z_j)=1,\  0\le z_{1}\le\ldots\le z_K\le 1\Big\}\notag\\
    =&\Big\{\mathbf{z}:\sum_{j=1}^{K}(1-z_j)=1,\  0\le z_{1}\le\ldots\le z_K\le 1, 1-z_j\le\frac{1}{j}\ \text{for}\ j\in[K]\Big\}\notag\\
    =&\Big\{\mathbf{z}: \sum_{j=1}^{K}z_j=K-1,\  0\le z_{1}\le\ldots\le z_K\le 1, z_j\ge1-\frac{1}{j}\ \text{for}\ j\in[K]\Big\},
\end{align}
where the second step holds since, for $j\in[K]$, $j(1-z_j)=\sum_{t=1}^{j}(1-z_j)\le\sum_{t=1}^{j}(1-z_t)\le1$ as $z_t\le z_j$ for $t\in[j]$, and the last step holds by rearranging the equality $\sum_{j=1}^{K}(1-z_j)=1$.

We next prove that the following $K$ feasible solutions are the only extreme points of (\ref{appdx.pf_optimal_multi_LPP_1}): 
\begin{align*}
    &\mathbf{z}_1\triangleq\left(0,1,1,\ldots,1,1\right)^\top, \\
    &\mathbf{z}_2\triangleq\left(1-\frac{1}{2}, 1-\frac{1}{2}, 1,\ldots,1,1\right)^\top,\\
    &\ldots,\\
    &\mathbf{z}_j\triangleq\bigg(\underbrace{1-\frac{1}{j},\ldots,1-\frac{1}{j}}_{j\ \text{elements}},\underbrace{1,\ldots,1}_{K-j\ \text{elements}}\bigg)^\top\\
    &\ldots,\\
    &\mathbf{z}_{K-1}\triangleq\left(1-\frac{1}{K-1}, 1-\frac{1}{K-1}, 1-\frac{1}{K-1},\ldots,1-\frac{1}{K-1},1\right)^\top,\\
    &\mathbf{z}_{K}\triangleq\left(1-\frac{1}{K}, 1-\frac{1}{K}, 1-\frac{1}{K},\ldots,1-\frac{1}{K},1-\frac{1}{K}\right)^\top.
\end{align*}
We denote $\Xi_0\triangleq\{\mathbf{z}_1,\ldots,\mathbf{z}_{K}\}$.  

Firstly, we prove that each data point in $\Xi_0$ is an extreme point of (\ref{appdx.pf_optimal_multi_LPP_1}). To this end, consider any $\mathbf{z}_j\in\Xi_0$. If there exist $\nu\in(0,1)$, $\mathbf{z}'=(z'_{1},\ldots,z'_K)^\top\in{\Xi}$, and $\mathbf{z}''=(z''_{1},\ldots,z''_K)^\top\in{\Xi}$, such that $\mathbf{z}_j=\nu\mathbf{z}'+(1-\nu)\mathbf{z}''$, then $\mathbf{z}'=\mathbf{z}''{=\mathbf{z}_j}$, as shown below. Let $z_{j,t}$, $z'_t$, and $z''_t$ represent the $t$th element of $\mathbf{z}_j$, $\mathbf{z}'$, and $\mathbf{z}''$, respectively.
\begin{itemize}
    \item If $t=j+1,\ldots,K$: then by $\nu z'_t+(1-\nu)z''_t=z_{j,t}$, $z_{j,t}=1$, and $z'_t,z''_t\le 1$, we have that $z'_t=z''_t=z_{j,t}=1$;
    \item If $t=j$: then $\nu z'_{j}+(1-\nu)z''_{j}=z_{j,j}=1-\frac{1}{j}$, and $z'_{j},z''_{j}\ge 1-\frac{1}{j}$ by (\ref{appdx.pf_optimal_multi_feasible}). Thus, we obtain that $z'_{j}=z''_{j}=z_{j,j}$.
    \item If $t=1,\ldots,j-1$: then $z'_t\le z'_{j}=1-\frac{1}{j}$, $z''_t\le z''_{j}=1-\frac{1}{j}$, and $\nu z'_t+(1-\nu)z''_t=z_{j,t}=1-\frac{1}{j}$. Thus, we can also obtain that $z'_t=z''_t=z_{j,t}$.
\end{itemize}
Therefore, $\mathbf{z}'=\mathbf{z}''=\mathbf{z}_j$, and hence, $\mathbf{z}_j$ is an extreme point of (\ref{appdx.pf_optimal_multi_LPP_1}) by Definition \ref{def.extreme}. 

Next, for any point $\widetilde{\mathbf{z}}\triangleq(\widetilde{z}_{1},\ldots,\widetilde{z}_K)^\top\in\Xi\backslash\Xi_0$, we prove that $\widetilde{\mathbf{z}}$ is not an extreme point of (\ref{appdx.pf_optimal_multi_LPP_1}) by construction. Specifically, we have the following claims for $\widetilde{\mathbf{z}}$.
\begin{itemize}
    \item \textit{Claim 1: $\widetilde{z}_{t}>1-\frac{1}{t}$ for $t\in[K]$}:\\
    This claim can be proved by contradiction. Assume there exists $t_0\in[K]$ such that $\widetilde{z}_{t_0}\le1-\frac{1}{t_0}$. As $\widetilde{\mathbf{z}}\in\Xi$, by (\ref{appdx.pf_optimal_multi_feasible}) and the assumption, we have $\widetilde{z}_{t_0}=1-\frac{1}{t_0}$. Since $\widetilde{\mathbf{z}}\notin\Xi_0$, one of the following statements must hold: (1) there exists $j<t_0$ such that $\widetilde{z}_j<1-\frac{1}{t_0}$; or (2) there exists $j>t_0$ such that $\widetilde{z}_j<1$ for some $j>t_0$. Therefore, $\sum_{j=1}^{K}\widetilde{z}_j{=\sum_{j=1}^{t_0}\widetilde{z}_j+\sum_{j=t_0+1}^{K}\widetilde{z}_j<\sum_{j=1}^{t_0}\widetilde{z}_{t_0}+\sum_{j=t_0+1}^{K}1=}t_0\cdot \widetilde{z}_{t_0}+(K-t_0)\cdot 1=K-1$ since $\widetilde{z}_1\le\ldots\le\widetilde{z}_{K}$ and $\widetilde{z}_t\le 1$ for $t\in[K]$ by (\ref{appdx.pf_optimal_multi_feasible}), where the strict inequality arises from the fact that either statement (1) or (2) holds. This conclusion contradicts the condition that $\widetilde{\mathbf{z}}\in\Xi$ by (\ref{appdx.pf_optimal_multi_feasible}).
    \item \textit{Claim 2: There exists $t_1\in[K]$ such that $\widetilde{z}_{t_1-1}<\widetilde{z}_{t_1}<1$:} \\
    This claim can be proved by contradiction:
    \begin{itemize}
        \item On one hand, if there exists $t'\in[K]$ such that $\widetilde{z}_{t'-1}<\widetilde{z}_{t'}$, then we must have $\widetilde{z}_t=\widetilde{z}_{t'-1}$ for $t\le t'-1$; otherwise, by letting $t''=\arg\max\{t:\widetilde{z}_t<\widetilde{z}_{t'-1},t<t'-1\}$, we obtain that $\widetilde{z}_{t''}<\widetilde{z}_{t''+1}=\widetilde{z}_{t'}<1$ and hence, $t_1$ can be set as $t''+1$, which contradicts the assumption. Additionally, we have $\widetilde{z}_t=1$ for $t\ge t'$; otherwise, $\widetilde{z}_{t'-1}<\widetilde{z}_{t'}<1$ and $t_1$ can be set as $t'$, which contradicts the assumption. Summarizing the discussion for $t\le t'-1$ and $t\ge t'$, we have $\widetilde{\mathbf{z}}\in\Xi_0$. 
        \item On the other hand, if $\widetilde{z}_{t-1}=\widetilde{z}_{t}$ for all $t\in[K]$, then $\widetilde{\mathbf{z}}=\mathbf{z}_{K}\in\Xi_0$. 
    \end{itemize}
    In both cases, $\widetilde{\mathbf{z}}\in\Xi_0$, contradicting the fact that $\widetilde{\mathbf{z}}\notin\Xi_0$. Hence, Claim 2 holds.
\end{itemize}
Let $t_2\triangleq\max\{t\in[K]:\widetilde{z}_{t}<1\}$. Then, $t_2\ge t_1$. Let 
\begin{align*}
    &\mathscr{c}_1\triangleq\min\{\frac{\widetilde{z}_{t_1}-\widetilde{z}_{t_1-1}}{2}, \widetilde{z}_t-(1-\frac{1}{t})\ \text{for}\ t\le t_1-1\}\ \text{and}\\
    &\mathscr{c}_2\triangleq\min\{\frac{\widetilde{z}_{t_1}-\widetilde{z}_{t_1-1}}{2}, \widetilde{z}_{t_1}-(1-\frac{1}{t_1}), 1-\widetilde{z}_t\ \text{for}\ t_1\le t\le t_2\}.
\end{align*}
By Claims 1 and 2, we have that $\mathscr{c}_1>0$ and $\mathscr{c}_2>0$. Let $\overline{\mathscr{c}}\triangleq\min\{(t_1-1)\mathscr{c}_1,(t_2-t_1+1)\mathscr{c}_2\}$, $\overline{\mathscr{c}}_1\triangleq\overline{\mathscr{c}}/(t_1-1)$, and $\overline{\mathscr{c}}_2\triangleq\overline{\mathscr{c}}/(t_2-t_1+1)$. Then we construct two points in $\Xi$: 
\begin{align*}
    &\mathbf{z}'\triangleq(\widetilde{z}_{1}+\overline{\mathscr{c}}_1,\ldots,\widetilde{z}_{t_1-1}+\overline{\mathscr{c}}_1,\widetilde{z}_{t_1}-\overline{\mathscr{c}}_2,\ldots,\widetilde{z}_{t_2}-\overline{\mathscr{c}}_2,\ldots,\widetilde{z}_{K})^\top\ \text{and}\\
    &\mathbf{z}''\triangleq(\widetilde{z}_{1}-\overline{\mathscr{c}}_1,\ldots,\widetilde{z}_{t_1-1}-\overline{\mathscr{c}}_1,\widetilde{z}_{t_1}+\overline{\mathscr{c}}_2,\ldots,\widetilde{z}_{t_2}+\overline{\mathscr{c}}_2,\ldots,\widetilde{z}_{K})^\top.
\end{align*}
Therefore, $\widetilde{\mathbf{z}}=\frac{1}{2}\mathbf{z}'+\frac{1}{2}\mathbf{z}''$, and hence, $\widetilde{\mathbf{z}}$ is not an extreme point of (\ref{appdx.pf_optimal_multi_LPP_1}).

\paragraph{Step 4: Solution format and optimal action.} \mbox{}\\
By Steps 2 and 3, we obtain that \emph{for each fixed $\chi$ and $s^*$}, the extreme points of the linear programming problem are given in $\Xi_0$. By Lemma \ref{lemma.extreme}, every linear program has an extreme point that is an optimal solution. Hence, {by the format of the $K$ extreme points in $\Xi_0$,} we obtain that at least one optimal action of $\psi$ can be found in the format: 
\begin{align}
    \psi^{(j)}=\frac{1}{k^*}\ \text{for}\ j\le k^*\ \text{and}\ \psi^{(j)}=0\ \text{for}\ j\ge k^*
\end{align}
for some $k^*\in[K]$. 

If $k^*=K$, by (\ref{appdx.pf_optimal_multi_g}), we have that $\mathdutchcal{g}(\gamma;\psi)=\gamma\epsilon^p+\sum_{j=1}^{K}P_j\cdot(1-\frac{1}{K})$, and hence, the robust risk is $\mathdutchcal{g}(\gamma^{\star}_{\psi};\psi)=1-\frac{1}{K}$ by taking $\gamma^{\star}_{\psi}=0$. 

If $k^*<K$, we obtain that 
\begin{align*}
    \mathdutchcal{g}(\gamma;\psi)=& \gamma\epsilon^p+\sum_{j=1}^{k^*}P^{(j)}\max\Big(1-\gamma\kappa^p,1-\frac{1}{k^*}\Big)+\sum_{j=k^*+1}^{K}P^{(j)}\cdot 1\\
    =&\left\{
    \begin{aligned}
    &1+\gamma\kappa^{p}\Big\{\varrho(\epsilon)-\sum_{j=1}^{k^*}P^{(j)}\Big\},\ \text{if}\ 0\le\gamma\le\frac{1}{k^*\kappa^p};\\
    &\gamma\epsilon^p+1-\frac{1}{k^*}\sum_{j=1}^{k^*}P^{(j)},\ \text{if}\ \gamma\ge\frac{1}{k^*\kappa^p}.\\
    \end{aligned}
    \right.
\end{align*}
Hence, for $k^*<K$, the robust risk is the minimum of $\mathdutchcal{g}(\gamma^{\star}_{\psi};\psi)=1$ by taking $\gamma^{\star}_{\psi}=0$ and $\mathdutchcal{g}(\gamma^{\star}_{\psi};\psi)=1+\frac{1}{k^*}\Big\{\varrho(\epsilon)-\sum_{j=1}^{k^*}P^{(j)}\Big\}$ by taking $\gamma^{\star}_{\psi}=\frac{1}{k^*\kappa^p}$. Additionally, we observe that we should take the highest $k^*$ values of $\{P_1,\ldots,P_K\}$ as $P^{(1)},\ldots,P^{(k^*)}$ to minimize {$\mathdutchcal{g}(\gamma^{\star}_{\psi};\psi)$}. Hence, we take the permutation $\chi$ such that $P^{(1)}\ge\ldots\ge P^{(K)}$.

In summary, the optimal action $\psi^\star$ that minimizes {$\mathdutchcal{g}(\gamma^{\star}_{\psi};\psi)$} is given as below.
\begin{itemize}
    \item If $\frac{1}{K}\ge \frac{1}{k^*}\sum_{j=1}^{k^*}P^{(j)}-\frac{1}{k^*}\varrho(\epsilon)$ for all $k^*\in[K-1]$, then $\psi^\star_j=\frac{1}{K}$ for $j\in[K]$.
    \item If there exists some $k_0\in[K-1]$, $\frac{1}{k_0}\sum_{j=1}^{k_0}P^{(j)}-\frac{1}{k_0}\varrho(\epsilon)>\frac{1}{K}$, and $\frac{1}{k_0}\sum_{j=1}^{k_0}P^{(j)}-\frac{1}{k_0}\varrho(\epsilon)\ge\frac{1}{k^*}\sum_{j=1}^{k^*}P^{(j)}-\frac{1}{k^*}\varrho(\epsilon)$ for all $k^*\in[K-1]$, then $\psi^{\star(j)}=\frac{1}{k_0}$ for $j\in[k_0]$ and $\psi^{\star(j)}=0$ for $j=k_0+1,\ldots,K$.
\end{itemize}
In particular, if $P^{(1)}\ge\max\{\frac{1}{K}+\varrho(\epsilon),P^{(2)}+\varrho(\epsilon)\}$, then the optimal action is given as: $\psi^{\star(1)}=1$ and $\psi^{\star(j)}=0$ for $j=2,\ldots, K$. Thus, the proof is complete.

\subsection{Proof of Theorem \ref{thm.robust_risk_form}}\label{appdx.pf_thm_robust_risk_form}

For ease of presentation, we omit the dependence on $\mathbf{x}_i$ and $\widetilde{\mathbf{y}}_i$ in the notation for now. Specifically, for $i\in[K]$ and $j\in[K]$, let $P_{i,j}\triangleq P_j(\mathbf{x}_i, \widetilde{\mathbf{y}}_i)\triangleq P(\mathrm{Y}=j|\mathbf{x}_i, \widetilde{\mathbf{y}}_i)$ and $\psi_{i,j}\triangleq\psi(\mathbf{x}_i)_j$. For given $\mathbf{x}_i$, we sort the $K$ elements of $\psi(\mathbf{x}_i)$, $\{\psi_{i,1},\ldots,\psi_{i,K}\}$, in a decreasing order, denoted $\psi^{(1)}_i\ge\ldots\ge\psi^{(K)}_i$.

We first consider the difference between the Wasserstein robust loss (\ref{eq.w_robust_loss}) and the nominal loss (\ref{eq.nominal_loss}):
\begin{align*}
    &\widehat{\mathfrak{R}}_\epsilon-\widehat{\mathfrak{R}}\\
    =&\inf_{\gamma\ge 0}\Big[\gamma\epsilon^p+\frac{1}{n}\sum_{i=1}^{n}\sum_{j=1}^{K}P_{i,j}\max\big\{\mathcal{T}(\psi_{i,1})-\mathcal{T}(\psi_{i,j})-\gamma\kappa^p,\ldots,\mathcal{T}(\psi_{i,j-1})-\mathcal{T}(\psi_{i,j})-\gamma\kappa^p,0,\\
    &\hspace{1.8cm}\mathcal{T}(\psi_{i,j+1})-\mathcal{T}(\psi_{i,j})-\gamma\kappa^p,\ldots,\mathcal{T}(\psi_{i,K})-\mathcal{T}(\psi_{i,j})-\gamma\kappa^p\big\}\Big]\\
    =&\inf_{\gamma\ge 0}\Big[\gamma\epsilon^p+\frac{1}{n}\sum_{i=1}^{n}\sum_{j=1}^{K}P_{i,j}\max\big\{\mathcal{T}(\psi_{i}^{(K)})-\mathcal{T}(\psi_{i,j})-\gamma\kappa^p,0\big\}\Big]\\
    \triangleq&\inf_{\gamma\ge 0}\mathdutchcal{h}(\gamma),
\end{align*}
where the second equality is due to the fact that $\psi^{(1)}_i\ge\ldots\ge\psi^{(K)}_i$ and that $\mathcal{T}$ is decreasing.

By definition, $\left\{\alpha_{i,j}\triangleq\mathcal{T}(\psi_{i}^{(K)})-\mathcal{T}(\psi_{i,j}):i\in[n],j\in[K]\right\}$ are ordered as $\alpha^{(1)}\ge\ldots\ge\alpha^{(nK)}$, and correspondingly, the $P_{i,j}$'s with the associated indexes are denoted $P^{(1)},\ldots,P^{(nK)}$. Consequently, we obtain that
\begin{align}\label{appdx.pf_robust_form_h_gamma_1_multi}
    \mathdutchcal{h}(\gamma)=\gamma\epsilon^p+\frac{1}{n}\sum_{t=1}^{nK}P^{(t)}(\alpha^{(t)}-\gamma\kappa^p)\mathbf{1}(\alpha^{(t)}>\gamma\kappa^p).
\end{align}

Define $\alpha^{(nK+1)}=0$. Then, $\mathdutchcal{h}(\gamma)$ in (\ref{appdx.pf_robust_form_h_gamma_1_multi}) can be expressed as 
\begin{align}\label{appdx.pf_robust_form_h_gamma_3_multi}
    \mathdutchcal{h}(\gamma)
    =& \left\{
    \begin{aligned}
        &\ \frac{1}{n}\sum_{t=1}^{s}P^{(t)}\alpha^{(t)}+\gamma\kappa^p\Big\{\varrho(\epsilon)-\frac{1}{n}\sum_{t=1}^{s}P^{(t)}\Big\},\ \text{if}\ \alpha^{(s+1)}/\kappa^p\le \gamma<\alpha^{(s)}/\kappa^p\ \text{for}\ s\in[nK];\\
        &\ \gamma\epsilon^p\ \text{if}\ \gamma\ge\alpha^{(1)}/\kappa^p.
    \end{aligned}\right.
\end{align}
Since 
\begin{align*}
    \lim_{\gamma\rightarrow(|\alpha^{(s)}|/\kappa^p)-}\mathdutchcal{h}(\gamma)=&\frac{1}{n}\sum_{t=1}^{s}P^{(t)}\alpha^{(t)}+|\alpha^{(s)}|\big\{\varrho(\epsilon)-\frac{1}{n}\sum_{t=1}^{s}P^{(t)}\big\}\\
    =&|\alpha^{(s)}|\varrho(\epsilon)+\frac{1}{n}\sum_{t=1}^{s}P^{(t)}(|\alpha^{(t)}|-|\alpha^{(s)}|)\\
    =&|\alpha^{(s)}|\varrho(\epsilon)+\frac{1}{n}\sum_{t=1}^{s-1}P^{(t)}(|\alpha^{(t)}|-|\alpha^{(s)}|)\\
    =&\frac{1}{n}\sum_{t=1}^{s-1}P^{(t)}|\alpha^{(t)}|+|\alpha^{(s)}|\Big\{\varrho(\epsilon)-\frac{1}{n}\sum_{t=1}^{s-1}P^{(t)}\Big\}\\
    =&\mathdutchcal{h}(|\alpha^{(s)}|/\kappa^p),
\end{align*}
and $\mathdutchcal{h}(\gamma)$ is right-continuous at $\gamma=|\alpha^{(s)}|/\kappa^p$ by definition (\ref{appdx.pf_robust_form_h_gamma_3_multi}), so we conclude that $\mathdutchcal{h}(\gamma)$ is continuous at $\gamma=|\alpha^{(s)}|/\kappa^p$ for $s\in[nK]$. Hence, by (\ref{appdx.pf_robust_form_h_gamma_3_multi}), $\mathdutchcal{h}(\gamma)$ is continuous for $\gamma\ge 0$.

By (\ref{appdx.pf_robust_form_h_gamma_3_multi}), $\mathdutchcal{h}(\gamma)$ is increasing in $\gamma$ on $\gamma\in[\alpha^{(s+1)}/\kappa^p,\alpha^{(s)}/\kappa^p)$ if $\varrho(\epsilon)-\frac{1}{n}\sum_{t=1}^{s}P^{(t)}>0$, and decreasing if $\varrho(\epsilon)-\frac{1}{n}\sum_{t=1}^{s}P^{(t)}<0$. Hence, to examine the monotonicity of $\mathdutchcal{h}(\gamma)$ and find the $\gamma$ that minimizes $\mathdutchcal{h}(\gamma)$, we consider the following three cases by the values of $\varrho(\epsilon)$.

\textbf{Case 1.} If $\frac{1}{n}P^{(1)}<\varrho(\epsilon)<\frac{1}{n}\sum_{t=1}^{nK}P^{(t)}$: then there exists $s^*\in\{2,\ldots,nK\}$ such that {$\varrho(\epsilon)>\frac{1}{n}\sum_{t=1}^{s}P^{(t)}$} for $s< s^*$, and $\varrho(\epsilon)\le \frac{1}{n}\sum_{t=1}^{s}P^{(t)}$ for $s\ge s^*$. Hence, by (\ref{appdx.pf_robust_form_h_gamma_3_multi}), $\mathdutchcal{h}(\gamma)$ is decreasing in $\gamma$ for $\gamma\in[0,\alpha^{(s^*)}/\kappa^p]$ and increasing for $\gamma\ge\alpha^{(s^*)}/\kappa^p$. Consequently, $\gamma^{\star}_{\psi}=\alpha^{(s^*)}/\kappa^p$, and
\begin{align*}
    \inf_{\gamma\ge 0}\mathdutchcal{h}(\gamma)=\mathdutchcal{h}(\alpha^{(s^*)}/\kappa^p)=&\frac{1}{n}\sum_{t=1}^{s^*-1}P^{(t)}|\alpha^{(t)}|+|\alpha^{(s^*)}|\Big\{\varrho(\epsilon)-\frac{1}{n}\sum_{t=1}^{s^*-1}P^{(t)}\Big\}\\
    =& \frac{1}{n}\sum_{t=1}^{s^*-1}P^{(t)}|\alpha^{(t)}|+O\left(\frac{1}{n}\right)|\alpha^{(s^*)}|.
\end{align*}
Here the last step holds because by the definition of $s^*$, $0<\varrho(\epsilon)-\frac{1}{n}\sum_{t=1}^{s^*-1}P^{(t)}\le\frac{1}{n}\sum_{t=1}^{s^*}P^{(t)}-\frac{1}{n}\sum_{t=1}^{s^*-1}P^{(t)}=\frac{1}{n}P^{(s^*)}\le\frac{1}{n}$, where the last inequality holds since $P^{(s^*)}\in[0,1]$.

\textbf{Case 2.} If $\varrho(\epsilon)\le\frac{1}{n}P^{(1)}$: then, by (\ref{appdx.pf_robust_form_h_gamma_3_multi}), $\mathdutchcal{h}(\gamma)$ is decreasing in $\gamma$ for $\gamma\in[0,\alpha^{(1)}/\kappa^p]$ and increasing for $\gamma\ge \alpha^{(1)}/\kappa^p$. Therefore, $\gamma^{\star}_{\psi}=\alpha^{(1)}/\kappa^p$, and
\begin{align*}
    \inf_{\gamma\ge 0}\mathdutchcal{h}(\gamma)=\mathdutchcal{h}(\alpha^{(1)}/\kappa^p)=\varrho(\epsilon)|\alpha^{(1)}|.
\end{align*}

\textbf{Case 3.} If $\varrho(\epsilon)\ge\frac{1}{n}\sum_{t=1}^{nK}P^{(t)}$: then, by (\ref{appdx.pf_robust_form_h_gamma_3_multi}), $\mathdutchcal{h}(\gamma)$ is increasing in $\gamma$ for $\gamma\ge0$. Therefore, $\gamma^{\star}_{\psi}=0=\alpha^{(nK+1)}/\kappa^p$, and
\begin{align*}
    \inf_{\gamma\ge 0}\mathdutchcal{h}(\gamma)=\mathdutchcal{h}(0)=\frac{1}{n}\sum_{t=1}^{nK}P^{(t)}|\alpha^{(t)}|.
\end{align*}

Hence, summarizing the discussion in the three cases above, we have that $\gamma^{\star}_{\psi}=\alpha^{(s^*)}/\kappa^p$, and the robust risk (\ref{eq.w_robust_loss}) is expressed as 
\begin{align*}
    \widehat{\mathfrak{R}}_\epsilon
    =&\widehat{\mathfrak{R}}+\frac{1}{n}\sum_{t=1}^{s^*-1}P^{(t)}\alpha^{(t)}\mathbf{1}(s^*>1)+O\left(\frac{1}{n}\right)\alpha^{(s^*)},
\end{align*}
where, to provide a unified expression, we define $s^*\triangleq 1$ and $s^*\triangleq nK+1$ in Case 2 and Case 3, respectively. Thus, the proof is completed.



\section{Experimental Details}\label{appdx.experiment}

\subsection{Implementation Details}\label{appdx.exp_setup}

\paragraph{Datasets.} We evaluate the effectiveness of the proposed AdaptCDRP on CIFAR-10 and CIFAR-100 \cite{krizhevsky2009learning} with synthetic annotations, and on four real-world datasets with human annotations: CIFAR-10N,  CIFAR-100N \cite[]{wei2022learning}, LabelMe \cite[]{rodrigues2017learning, torralba2010labelme}, and  Animal-10N \cite{song2019selfie}. CIFAR-10 has 10 classes of $32\times 32\times 3$ color images, with 50,000 training images and 10,000 test images;  CIFAR-10N provides three independent human annotated noisy labels per instance, with a majority vote yielding a 9.03\% noise rate. CIFAR-100, with the same number and size of training and test images as CIFAR-10, features 100 fine-grained classes; for each instance in CIFAR-100, CIFAR-100N provides one human annotated noisy label, with a noise rate of 40.20\%. 
LabelMe is an image classification dataset comprising 10,000 training images, 500 validation images, and 1,188 test images. The training set includes noisy and incomplete labels provided by 59 annotators, with each image being labeled an average of 2.547 times.
The Animal-10N dataset contains 10 classes of $64\times 64\times 3$ color animal images; it includes 5 pairs of similar-looking animals, where each pair consists of two animals that are visually alike. The training dataset contains 50,000 images and the test dataset contains 5,000 images. The noise rate (mislabeling ratio) of the dataset is about 8\%. For all the datasets except LabelMe, we allocate 10\% of the training data as validation data used for model selection, where we choose the model with the lowest validation accuracy during training. The test data is reserved for final evaluation of the model's performance on unseen data.

\paragraph{Noise generation.} We generate synthetic instance-dependent label noise on the CIFAR-10 and CIFAR-100 datasets using Algorihtm 2 in \cite{xia2020part}. Each annotator is classified as an IDN-$\tau$ annotator if their mislabeling ratio is upper bounded by $\tau$. We simulate $R$ annotators independently, with $R$ taking values from the set $\{5, 10, 30, 50, 100\}$. For each instance, one annotation is randomly selected from those provided by the $R$ annotators' contributions, which evaluates methods under incomplete annotator labeling conditions. Additionally, for each $R$, we consider three groups of annotators with varying expertise levels, characterized by average mislabeling ratios of approximately 20\%, 35\%, and 50\%. These groups are referred to as IDN-LOW, IDN-MID, and IDN-HIGH, indicating low, medium, and high error rates, respectively. We manually corrupt the datasets according to the following annotator groups:
\begin{align*}
    &\textbf{R=5:}\\
    &\textbf{IDN-LOW. }\textit{2 IDN-10\% annotators, 2 IDN-20\% annotators, 1 IDN-30\% annotator;}\\
    &\textbf{IDN-MID. }\textit{2 IDN-30\% annotators, 2 IDN-40\% annotators, 1 IDN-50\% annotator;}\\
    &\textbf{IDN-HIGH. }\textit{2 IDN-50\% annotators, 2 IDN-60\% annotators, 1 IDN-70\% annotator;}\\
    &\textbf{R=10:}\\
    &\textbf{IDN-LOW. }\textit{4 IDN-10\% annotators, 4 IDN-20\% annotators, 2 IDN-30\% annotators;}\\
    &\textbf{IDN-MID. }\textit{4 IDN-30\% annotators, 4 IDN-40\% annotators, 2 IDN-50\% annotators;}\\
    &\textbf{IDN-HIGH. }\textit{4 IDN-50\% annotators, 4 IDN-60\% annotators, 2 IDN-70\% annotators;}\\
    &\textbf{R=30:}\\
    &\textbf{IDN-LOW. }\textit{11 IDN-10\% annotators, 11 IDN-20\% annotators, 8 IDN-30\% annotators;}\\
    &\textbf{IDN-MID. }\textit{11 IDN-30\% annotators, 11 IDN-40\% annotators, 8 IDN-50\% annotators;}\\
    &\textbf{IDN-HIGH. }\textit{11 IDN-50\% annotators, 11 IDN-60\% annotators, 8 IDN-70\% annotators;}\\
    &\textbf{R=50:}\\
    &\textbf{IDN-LOW. }\textit{18 IDN-10\% annotators, 18 IDN-20\% annotators, 14 IDN-30\% annotators;}\\
    &\textbf{IDN-MID. }\textit{18 IDN-30\% annotators, 18 IDN-40\% annotators, 14 IDN-50\% annotators;}\\
    &\textbf{IDN-HIGH. }\textit{18 IDN-50\% annotators, 18 IDN-60\% annotators, 14 IDN-70\% annotators;}\\
    &\textbf{R=100:}\\
    &\textbf{IDN-LOW. }\textit{35 IDN-10\% annotators, 35 IDN-20\% annotators, 30 IDN-30\% annotators;}\\
    &\textbf{IDN-MID. }\textit{35 IDN-30\% annotators, 35 IDN-40\% annotators, 30 IDN-50\% annotators;}\\
    &\textbf{IDN-HIGH. }\textit{35 IDN-50\% annotators, 35 IDN-60\% annotators, 30 IDN-70\% annotators.}
\end{align*}

\paragraph{Experiment setup.}  We employ the ResNet-18 architecture for CIFAR-10 and CIFAR-10N, and the ResNet-34 architecture for CIFAR-100 and CIFAR-100N datasets. Following \cite{rodrigues2018deep}, we use a pretrained VGG-16 model with a 50\% dropout rate as the backbone for the LabelMe dataset. For the Animal-10N dataset, in line with \cite{song2019selfie}, we use the VGG19-BN architecture \cite{simonyan2014very} as the backbone. A batch size of 128 is maintained across all datasets. We use the Adam optimizer \cite{kingma2014adam} with a weight decay of $5\times 10^{-4}$ for CIFAR-10, CIFAR-100, CIFAR-10N, CIFAR-100N, and LabelMe datasets. The initial learning rate for CIFAR-10, CIFAR-100, CIFAR-10N, and CIFAR-100N is set to $10^{-3}$, with the networks trained for 120, 150, 120, and 150 epochs respectively. The first 30 epochs serve as a warm-up. For the LabelMe dataset, the model is trained for 100 epochs with an initial learning rate of $10^{-2}$ and a 20-epoch warm-up. For the Animal-10N dataset, the network is trained for 100 epochs with an initial learning rate of $10^{-1}$ and a weight decay of $10^{-3}$. The learning rate is reduced by a factor of 0.1 at the 50th and 75th epochs, with the first 40 epochs designated as the warm-up stage. Training times are approximately 3 hours on CIFAR-10 and 5.5 hours on CIFAR-100 using an NVIDIA V100 GPU.

\paragraph{Baselines.} Our method addresses learning from noisy annotations, particularly when estimated true label posteriors may be misspecified. Thus, we select baselines that either use estimated transition matrices or true label posteriors (MBEM \cite{khetan2017learning}, CrowdLayer \cite{rodrigues2018deep}, TraceReg \cite{tanno2019learning}, Max-MIG \cite{cao2019max}, CoNAL \cite{chu2021learning}). We also include baselines that aggregate labels differently (CE (MV), CE (EM) \cite{dawid1979maximum}, DoctorNet \cite{guan2018said}, CCC \cite{zhang2024coupled}). Since our theoretical framework applies to both single-annotator and multiple-annotator scenarios, we also include baselines designed for single noisy labels (LogitClip \cite{wei2023mitigating}), particularly those employing two networks (Co-teaching \cite{han2018co}, Co-teaching+ \cite{yu2019does}, CoDis \cite{xia2023combating}), as our method similarly uses two networks that act as priors for each other. Details of the baselines are given as follows. \vspace{-0.2cm}
\begin{itemize}
    \item[(1)] CE (Clean): Trains the network using the standard cross-entropy loss on clean datasets;
    \item[(2)] CE (MV): Trains the network using majority voting labels;
    \item[(3)] CE (EM) \cite{dawid1979maximum}: Aggregate labels using the EM algorithm;
    \item[(4)] Co-teaching \cite{han2018co}: Trains two networks and cross-trains on instances with small loss values;
    \item[(5)] Co-teaching+ \cite{yu2019does}: Combines the "Update by Disagreement" with the Co-teaching method;
    \item[(6)] CoDis \cite{xia2023combating}: Selects possibly clean data that have high-discrepancy prediction probabilities between two networks;
    \item[(7)] LogitClip \cite{wei2023mitigating}: Clamps the norm of the logit vector to ensure it is upper bounded by a constant;
    \item[(8)] DoctorNet \cite{guan2018said}: Models individual annotators and learns averaging weights by combining them; 
    \item[(9)] MBEM \cite{khetan2017learning}: Alternates between estimating annotator quality from disagreements with the current model and updating the model by optimizing a loss function that accounts for the current estimate of worker quality;
    \item[(10)] CrowdLayer \cite{rodrigues2018deep}: Concatenates the classifier with multiple annotator-specific layers and learns the parameters simultaneously;
    \item[(11)] TraceReg \cite{tanno2019learning}: Uses a loss function similar to CrowdLayer but adds regularization to establish identifiability of the confusion matrices and the classifier; 
    \item[(12)] Max-MIG \cite{cao2019max}: Jointly aggregates noisy crowdsourced labels and trains the classifier;
    \item[(13)] CoNAL \cite{chu2021learning}: Decomposes the annotation noise into common and individual confusions;
    \item[(14)] CCC \cite{zhang2024coupled}: Simultaneously trains two models to correct the confusion matrices learned by each other via bi-level optimization. 
\end{itemize}\vspace{-0.2cm}
Among these methods, Co-teaching, Co-teaching+, CoDis, and LogitClip are strong baselines for handling single noisy labels. We adapt them to the multiple annotations setting by using majority vote labels for loss computation. The results demonstrate the effectiveness of the proposed pseudo-label generation method across various scenarios. Results for CE (Clean), CE (MV), CE (EM), DoctorNet, MBEM, CrowdLayer, Max-MIG, and CoNAL in Table \ref{table_cifar} are sourced from \cite{guo2023label}. Baselines (1)-(3) are implemented according to their respective algorithms, while for the remaining baseline methods, we adapted the code from the GitHub repositories provided in their original papers, with further modifications to fit our setup.

\paragraph{Pseudo code for the algorithm.} The training process described in Section \ref{sec.algorithm} is presented in Algorithm \ref{algorithm_AdatpGamma}.

{\scriptsize
\newcommand\mycommfont[1]{\small\ttfamily\textcolor{blue}{#1}}
\SetCommentSty{mycommfont}

\begin{algorithm}[t]
    \SetKwInput{Input}{Input~}
    \SetKwInOut{KwOut}{Output}

    \caption{Learning from Noisy Labels via \textbf{C}onditional \textbf{D}istributionally \textbf{R}obust True Label \textbf{P}osterior with an \textbf{Adapt}ive  Lagrange multiplier  (AdaptCDRP)}\label{algorithm_AdatpGamma}

    \KwIn{$\mathcal{D}=\{\mathbf{x}_i,\Tilde{\mathbf{y}}_i\}_{i=1}^{n}$, $\epsilon\in(0, \frac{1}{K})$, $\kappa>0$, $\mathdutchcal{C}>1$, $\lambda>0$}

    Warm up classifiers $\psi^{(1)}$ and $\psi^{(2)}$; Approximate noise transition probabilities $\widehat{\tau}_j(\widetilde{\mathbf{y}})$ for $j\in[K]$ using small-loss data\;

    
    \For{$epoch\ t=1,...,T$}{
        \tcp{Update the classifiers with pseudo-empirical distribution (Theorem \ref{thm.optimal_action})}
        Update approximated true label posteriors: $\widehat{P}^{(\iota)}_j(\mathbf{x},\widetilde{\mathbf{y}})\propto \psi^{(\iota)}_j(\mathbf{x})\cdot \widehat{\tau}_j(\widetilde{\mathbf{y}})$ for $j\in[K]$\;

        \For{each instance $\mathbf{x}_i$}{
        if $\widehat{P}^{(\iota)}_{k^\star}(\mathbf{x}_i,\widetilde{\mathbf{y}}_i)/\max_{j\ne k^\star}\widehat{P}^{(\iota)}_j(\mathbf{x},\widetilde{\mathbf{y}})\ge\mathdutchcal{C}$, let $\mathrm{y}^\star_i=k^\star$ and collect $(\mathbf{x}_i,\widetilde{\mathbf{y}}_i, \mathrm{y}^\star_i)$ into $\mathcal{D}^{\star}_{t,\iota}$\;
        } 
        Update the pseudo-empirical distribution $P^\star_{t,\iota}$ based on $\mathcal{D}^{\star}_{t,\iota}$\;
        Update $\psi^{(\iota)}$ by minimizing the empirical robust risk (\ref{eq.re_rob_risk_emp}) with the reference distribution $P^\star_{t,\backslash\iota}$ and the Lagrange multiplier $\gamma^{(\iota)}_{t-1}$\;
        \tcp{Update the Lagrange multiplier (Theorem \ref{thm.robust_risk_form})}
        Compute $\alpha_i's$ for $i\in[nK]$ and $s^*$ by Theorem \ref{thm.robust_risk_form}\;
        Compute the reference value for the Lagrange multiplier:  $\gamma_{0,t}=|\alpha^{(s^*)}|/\kappa^p$\;
        Update the Lagrange multiplier: $\gamma^{(\iota)}_{t}=\gamma_{0,t}-\frac{1}{\lambda}\{\epsilon^p -\mathbb{E}_{P^\star_{t,\backslash\iota}}c^p(y',\mathrm{Y})\}$
    }
    \KwOut{\ $\psi_1$ and $\psi_2$.}
\end{algorithm}

}

\subsection{Additional Experimental Results}\label{appdx.exp_result}

\paragraph{Performance on the CIFAR-100 dataset with varying numbers of annotators.} We conduct additional experiments on the CIFAR-100 dataset, varying the number of annotators from 5 to 100, with each instance labeled only once. Figure \ref{Fig.ACC_cifar100} presents the average accuracy across different annotator counts, highlighting the advantages of the proposed method across various settings. As the total number of annotators increases, labeling sparsity becomes more pronounced, which may lead to a performance collapse in methods that do not account for this sparsity, especially in datasets with a large number of classes, such as CIFAR-100.

\begin{figure}[H]
\centering 
\subfigure[IDN-LOW]{
\label{Fig.ACC_cifar100_sub.1}
\includegraphics[width=0.323\textwidth]{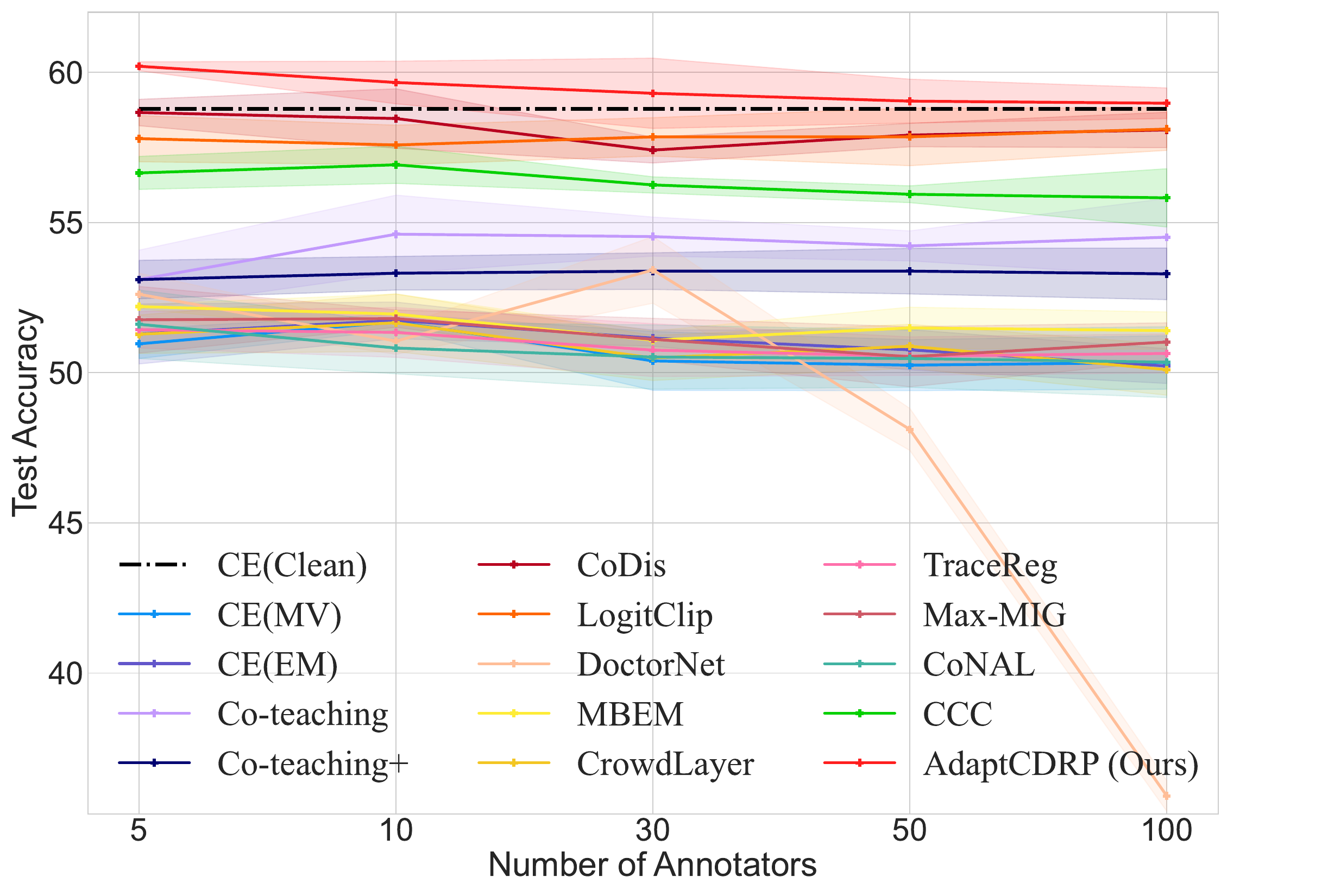}}
\subfigure[IDN-MID]{
\label{Fig.ACC_cifar100_sub.2}
\includegraphics[width=0.323\textwidth]{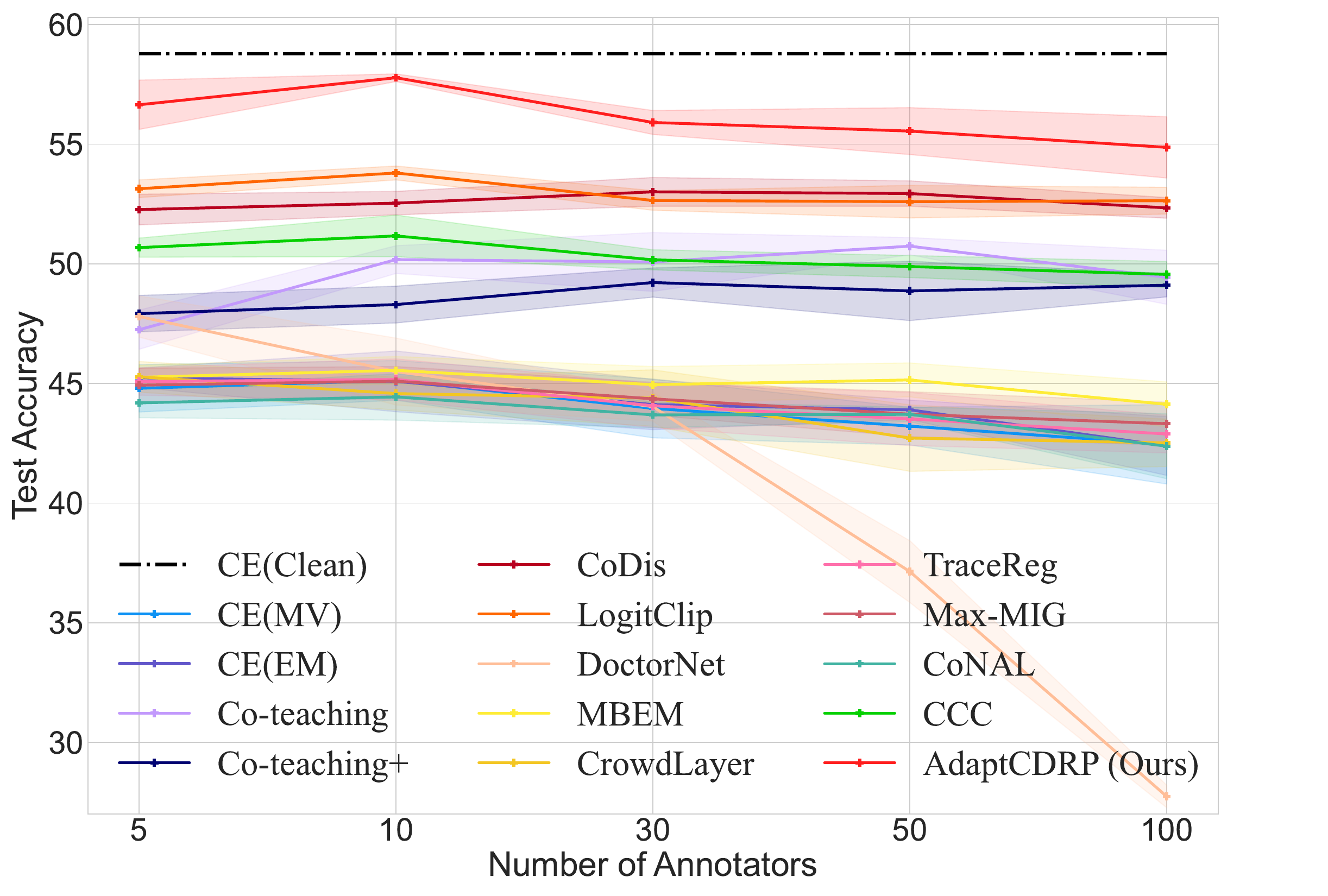}}
\subfigure[IDN-HIGH]{
\label{Fig.ACC_cifar100_sub.3}
\includegraphics[width=0.323\textwidth]{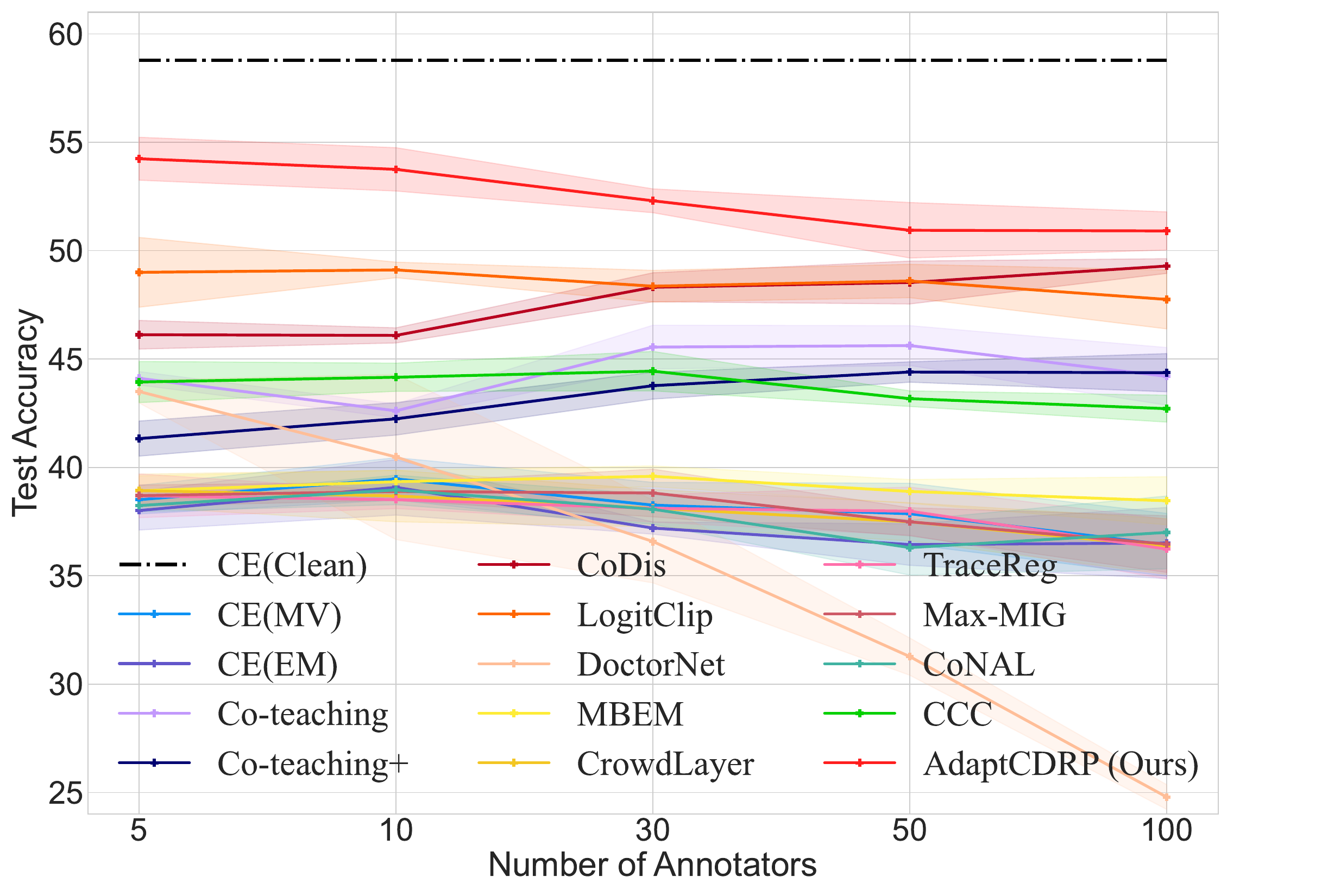}}
\caption{Average test accuracy on the CIFAR-100 dataset with varying numbers of annotators. The error bars representing standard deviations are shaded.}
\label{Fig.ACC_cifar100}
\end{figure}


\paragraph{Performance with varying numbers of annotations per instance.} To further evaluate model performance with varying numbers of annotations per instance, we use $R = 30$ annotators and randomly select $l = 1, 3, 5, 7, 9$ labels from these $R$ annotators for each instance. The test accuracies of the proposed method and other annotation aggregation methods are shown in Table \ref{table_increase_R}.

\begin{table}[H] \footnotesize
  \renewcommand\arraystretch{1.0}
  \setlength\tabcolsep{3.0pt}
  \caption{Average test accuracies (with associated standard errors expressed after the $\pm$ signs) for learning the CIFAR-10 dataset with varying numbers of annotations (denoted $l$) from $R=30$ annotators.}
  \label{table_increase_R}
  \centering
  {
  \begin{tabular}{cc|ccccc}
    \toprule 
    & & {Ours (AdaptCDRP)} & CE(MV) & CE(EM) \cite{dawid1979maximum} & IWMV \cite{li2014error} & IAA \cite{bucarelli2023leveraging} \\
    \midrule
    \multirow{5}{*}{\makecell{IDN-\\LOW}}
    & $l=1$\ \ & $88.09_{\pm 0.37}$ & $80.90_{\pm 0.88}$ & $81.15_{\pm 0.74}$ & $-$ & $-$ \\
    & $l=3$\ \ & $88.76_{\pm 0.30}$ & $83.00_{\pm 0.40}$ & $83.22_{\pm 0.43}$ & $83.09_{\pm 0.35}$ & $82.99_{\pm 0.40}$ \\
    & $l=5$\ \ & $89.05_{\pm 0.36}$ & $85.78_{\pm 0.53}$ & $85.59_{\pm 0.46}$ & $86.63_{\pm 0.29}$ & $84.53_{\pm 0.52}$ \\
    & $l=7$\ \ & $89.06_{\pm 0.30}$ & $87.46_{\pm 0.32}$ & $87.93_{\pm 0.53}$ & $87.83_{\pm 0.24}$ & $85.05_{\pm 0.55}$ \\
    & $l=9$\ \ & $89.24_{\pm 0.51}$ & $88.30_{\pm 0.25}$ & $88.38_{\pm 0.32}$ & $88.24_{\pm 0.12}$ & $85.55_{\pm 0.46}$ \\
    \cmidrule(lr){1-7}
    \multirow{5}{*}{\makecell{IDN-\\MID}}
    & $l=1$\ \ & $87.37_{\pm 0.29}$ & $76.05_{\pm 0.70}$ & $75.84_{\pm 0.97}$ & $-$ & $-$ \\
    & $l=3$\ \ & $88.47_{\pm 0.19}$ & $79.12_{\pm 0.66}$ & $79.11_{\pm 0.71}$ & $79.72_{\pm 0.44}$ & $79.43_{\pm 0.54}$ \\
    & $l=5$\ \ & $88.68_{\pm 0.17}$ & $81.58_{\pm 0.20}$ & $81.90_{\pm 0.67}$ & $82.05_{\pm 0.59}$ & $81.56_{\pm 0.31}$ \\
    & $l=7$\ \ & $88.71_{\pm 0.24}$ & $83.24_{\pm 0.34}$ & $82.84_{\pm 0.43}$ & $83.06_{\pm 0.47}$ & $83.28_{\pm 0.22}$ \\
    & $l=9$\ \ & $88.89_{\pm 0.29}$ & $84.04_{\pm 0.13}$ & $83.95_{\pm 0.43}$ & $84.28_{\pm 0.22}$ & $83.94_{\pm 0.23}$ \\
    \cmidrule(lr){1-7}
    \multirow{5}{*}{\makecell{IDN-\\HIGH}}
    & $l=1$\ \ & $86.62_{\pm 0.45}$ & $69.65_{\pm 1.73}$ & $69.85_{\pm 1.43}$ & $-$ & $-$ \\
    & $l=3$\ \ & $88.37_{\pm 0.19}$ & $74.32_{\pm 0.40}$ & $74.02_{\pm 1.01}$ & $75.00_{\pm 0.68}$ & $74.80_{\pm 0.38}$ \\
    & $l=5$\ \ & $88.63_{\pm 0.48}$ & $77.91_{\pm 1.30}$ & $78.42_{\pm 0.42}$ & $77.90_{\pm 0.74}$ & $78.04_{\pm 0.50}$ \\
    & $l=7$\ \ & $88.64_{\pm 0.37}$ & $79.75_{\pm 0.81}$ & $80.00_{\pm 0.55}$ & $79.36_{\pm 0.61}$ & $79.72_{\pm 0.69}$ \\
    & $l=9$\ \ & $88.78_{\pm 0.26}$ & $80.90_{\pm 0.65}$ & $81.11_{\pm 0.44}$ & $80.53_{\pm 0.88}$ & $80.74_{\pm 0.69}$ \\
    \bottomrule
  \end{tabular}}
\end{table}

\paragraph{Accuracy of robust pseudo-labels.} To enhance the assessment of the effectiveness of the proposed robust pseudo-label generation method, we present the average accuracy of the robust pseudo-labels on the CIFAR-10 and CIFAR-100 datasets during the training process over 5 random trials, as shown in Figure \ref{Fig.pseudo_ACC}. Additionally, the average accuracy of the robust pseudo-labels with varying numbers of annotators on the CIFAR-10 dataset is shown in Figure \ref{Fig.pseudo_ACC_sparse}.

\begin{figure}[H]
\centering 
\subfigure[CIFAR-10]{
\label{Fig.pseudo_ACC_cifar10}
\includegraphics[width=0.49\textwidth]{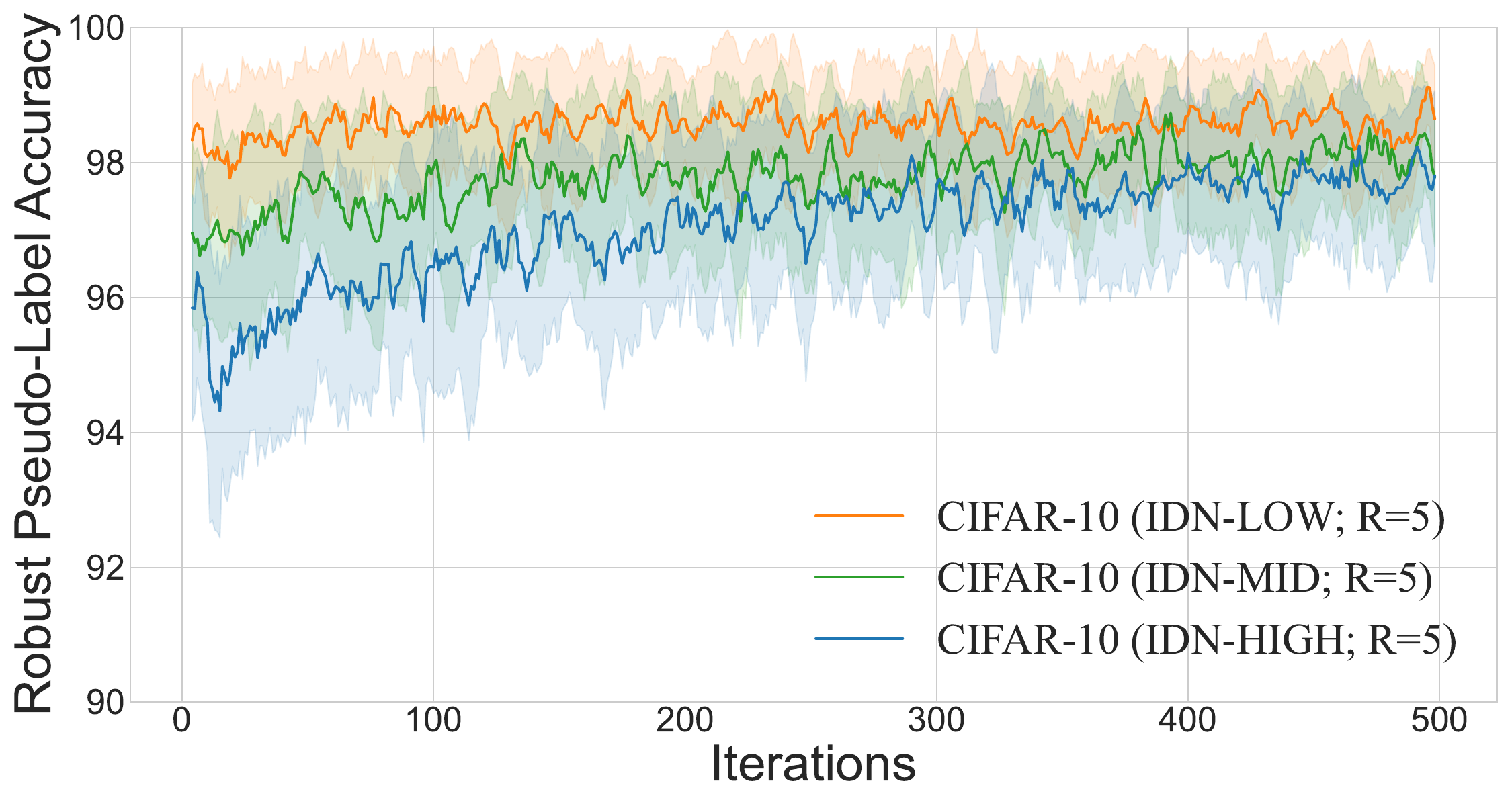}}
\subfigure[CIFAR-100]{
\label{Fig.pseudo_ACC_cifar100}
\includegraphics[width=0.49\textwidth]{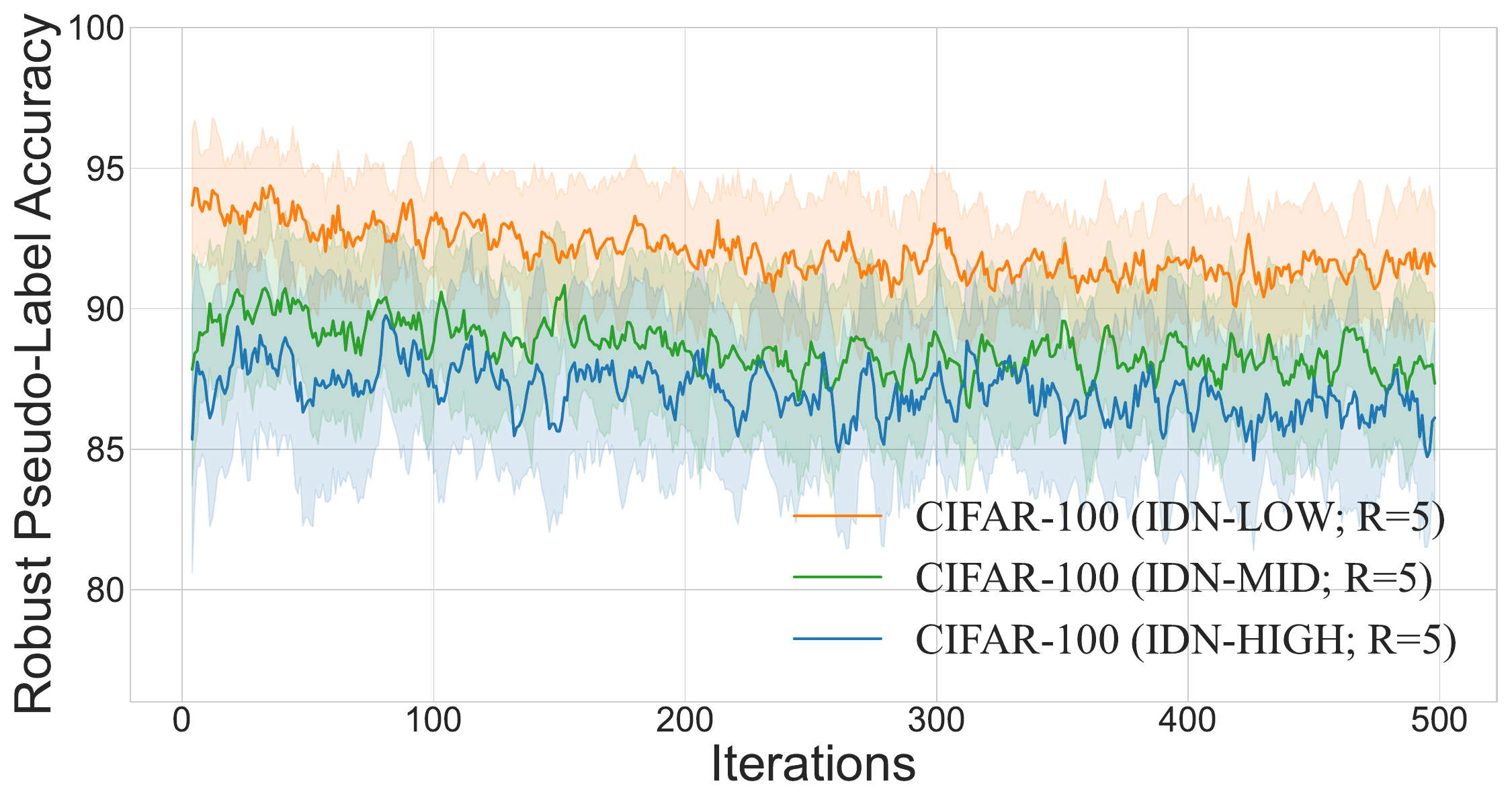}}
\caption{Average accuracy of robust pseudo-labels on the CIFAR-10 and CIFAR-100 datasets {($R=5$)} during the training process.}
\label{Fig.pseudo_ACC}
\end{figure}

\begin{figure}[H]
\centering 
\subfigure[CIFAR-10 ($R=5$)]{
\label{Fig.pseudo_ACC_cifar10_5}
\includegraphics[width=0.49\textwidth]{Pseudo_ACC_cifar10.pdf}}
\subfigure[CIFAR-10 ($R=10$)]{
\label{Fig.pseudo_ACC_cifar10_10}
\includegraphics[width=0.49\textwidth]{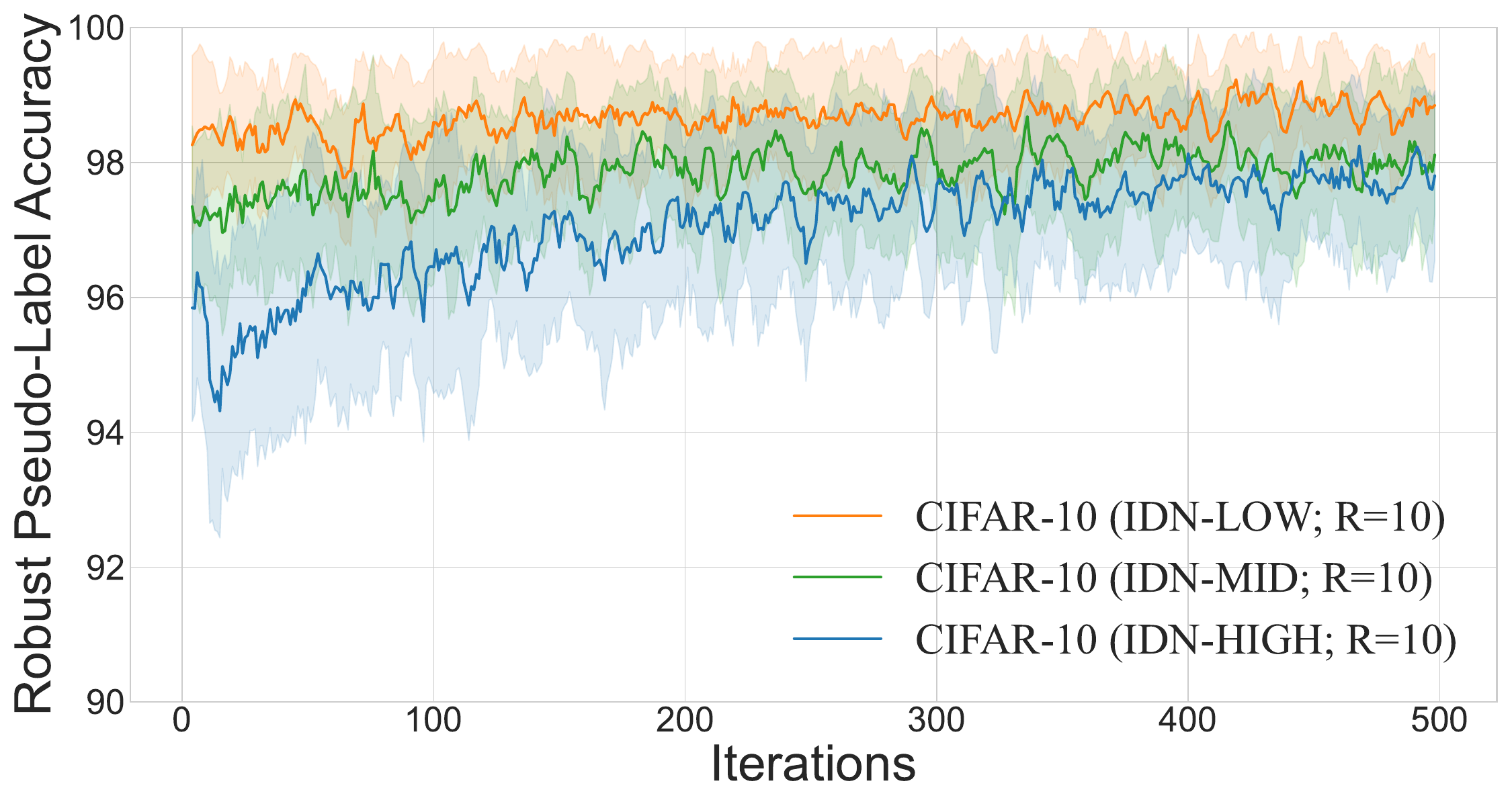}}
\subfigure[CIFAR-10 ($R=30$)]{
\label{Fig.pseudo_ACC_cifar10_30}
\includegraphics[width=0.49\textwidth]{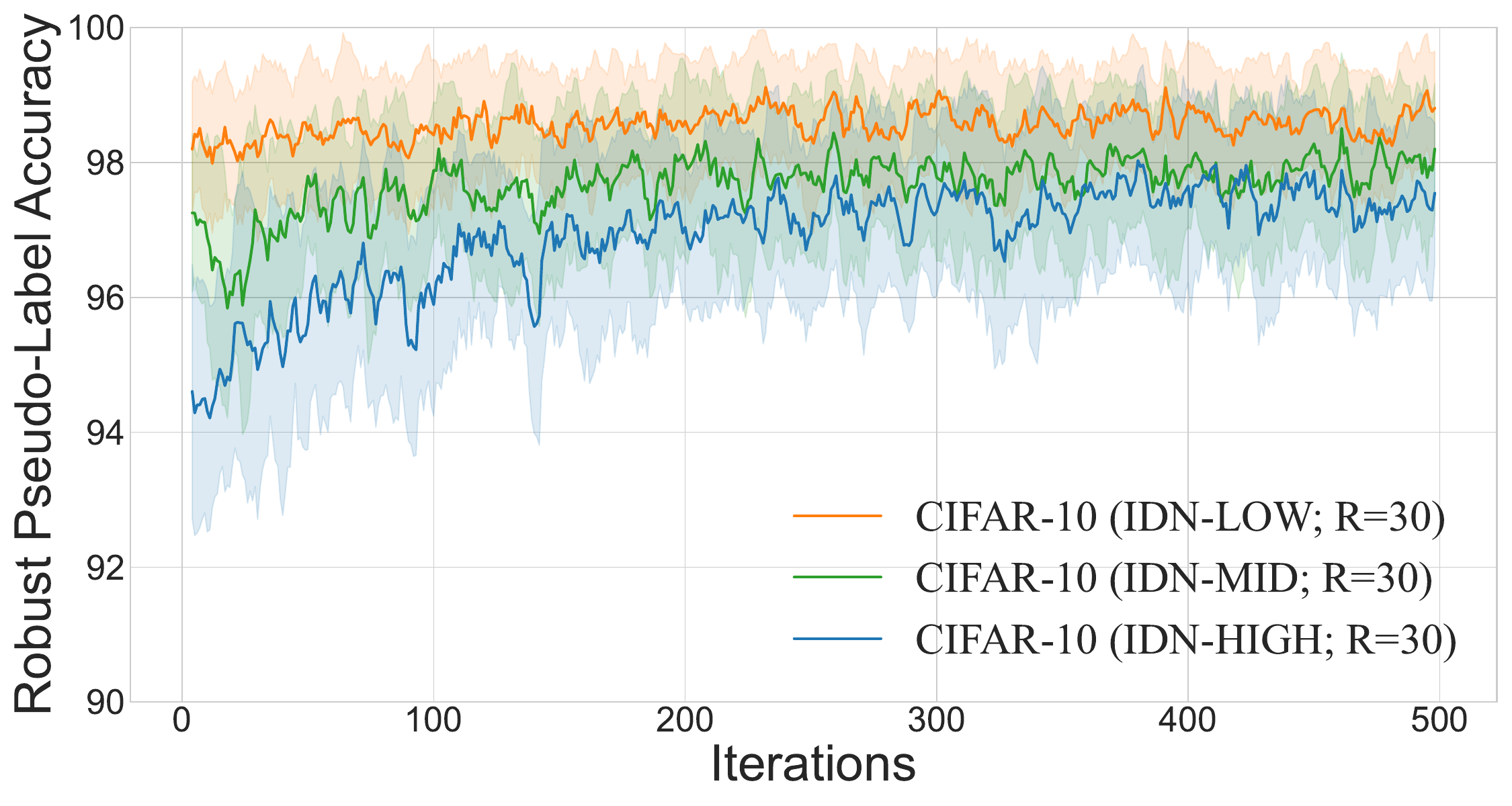}}
\subfigure[CIFAR-10 ($R=50$)]{
\label{Fig.pseudo_ACC_cifar10_50}
\includegraphics[width=0.49\textwidth]{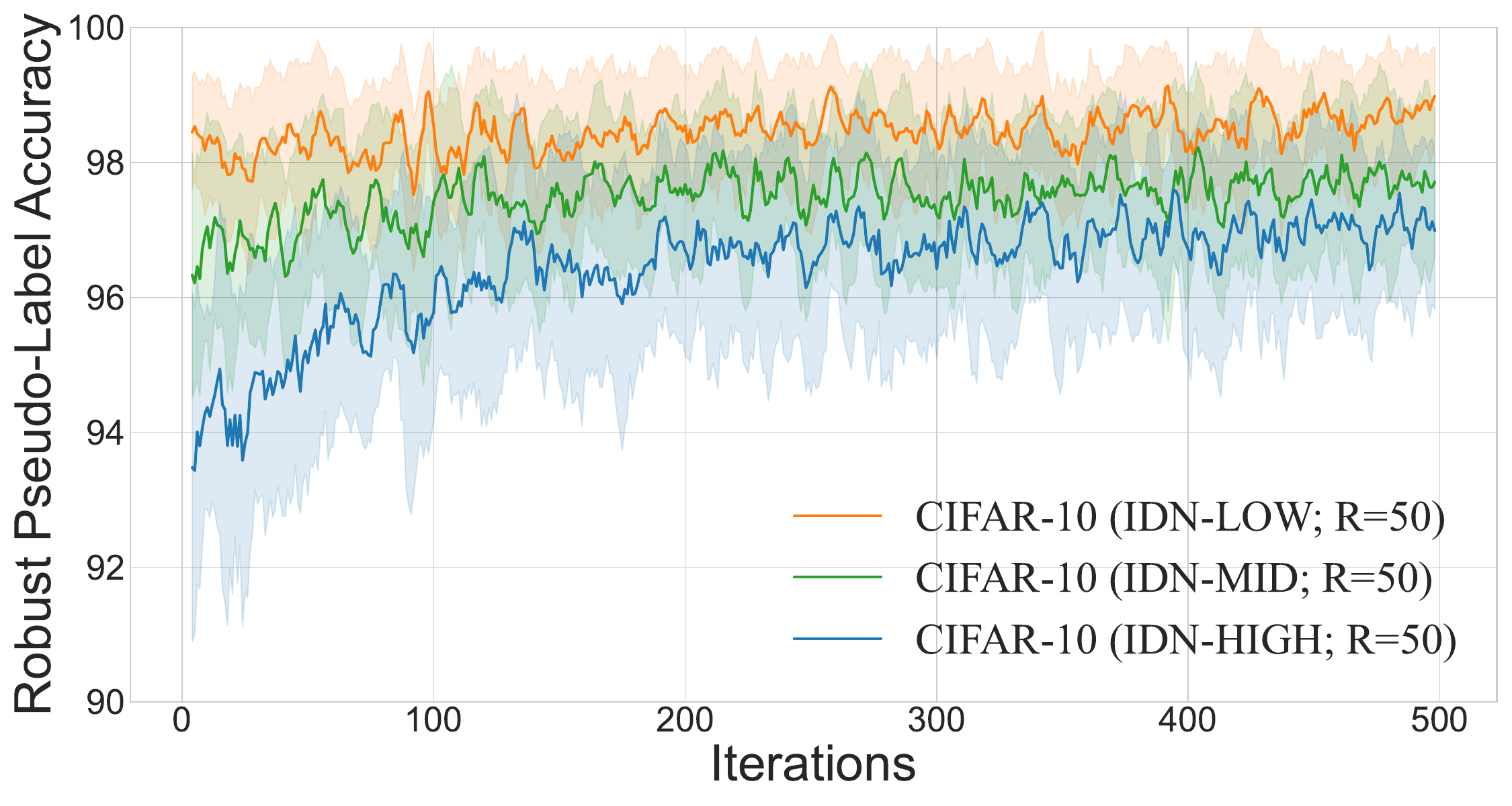}}
\subfigure[CIFAR-10 ($R=100$)]{
\label{Fig.pseudo_ACC_cifar10_100}
\includegraphics[width=0.49\textwidth]{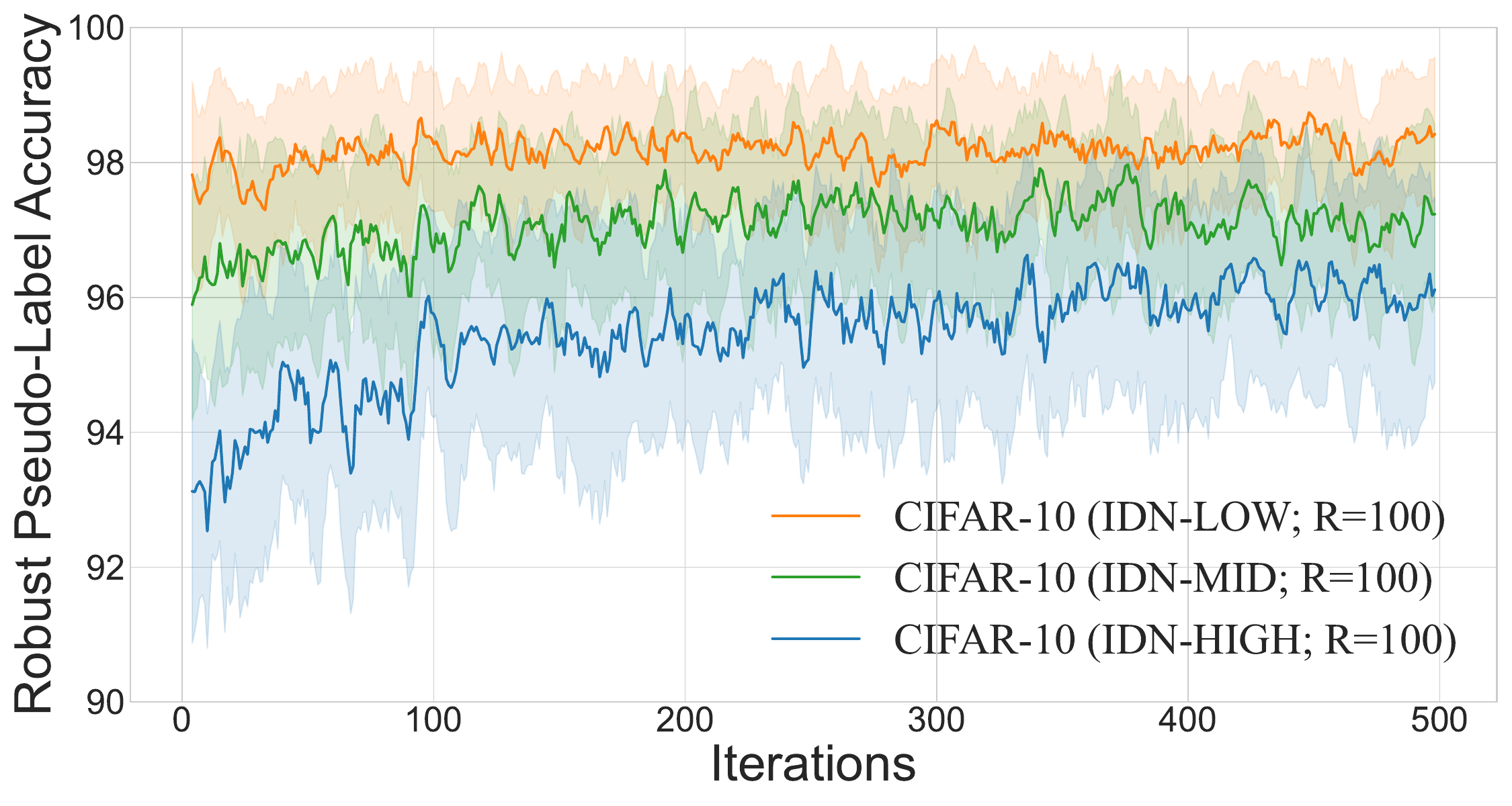}}
\caption{Average accuracy of robust pseudo-labels on the CIFAR-10 dataset with varying number of annotators in the training process.}
\label{Fig.pseudo_ACC_sparse}
\end{figure}

\paragraph{Test accuracy during the training process.} To further assess the effectiveness of the proposed method, we present the average test accuracy for the CIFAR-10 and CIFAR-100 datasets during the training process, as shown in Figures \ref{Fig.epoch_ACC_cifar_10} and \ref{Fig.epoch_ACC_cifar_100}, respectively. The results indicate that the model tends to overfit during the warm-up stage, particularly under higher noise rates. This suggests that the results in Table \ref{table_cifar} are not obtained with the optimal number of warm-up epochs. However, following the warm-up phase, the test accuracy of our method steadily improves, outperforming baseline methods across various scenarios.

\begin{figure}[H]
\centering 
\subfigure[IDN-LOW]{
\label{Fig.epoch_ACC_cifar10_low}
\includegraphics[width=0.32\textwidth]{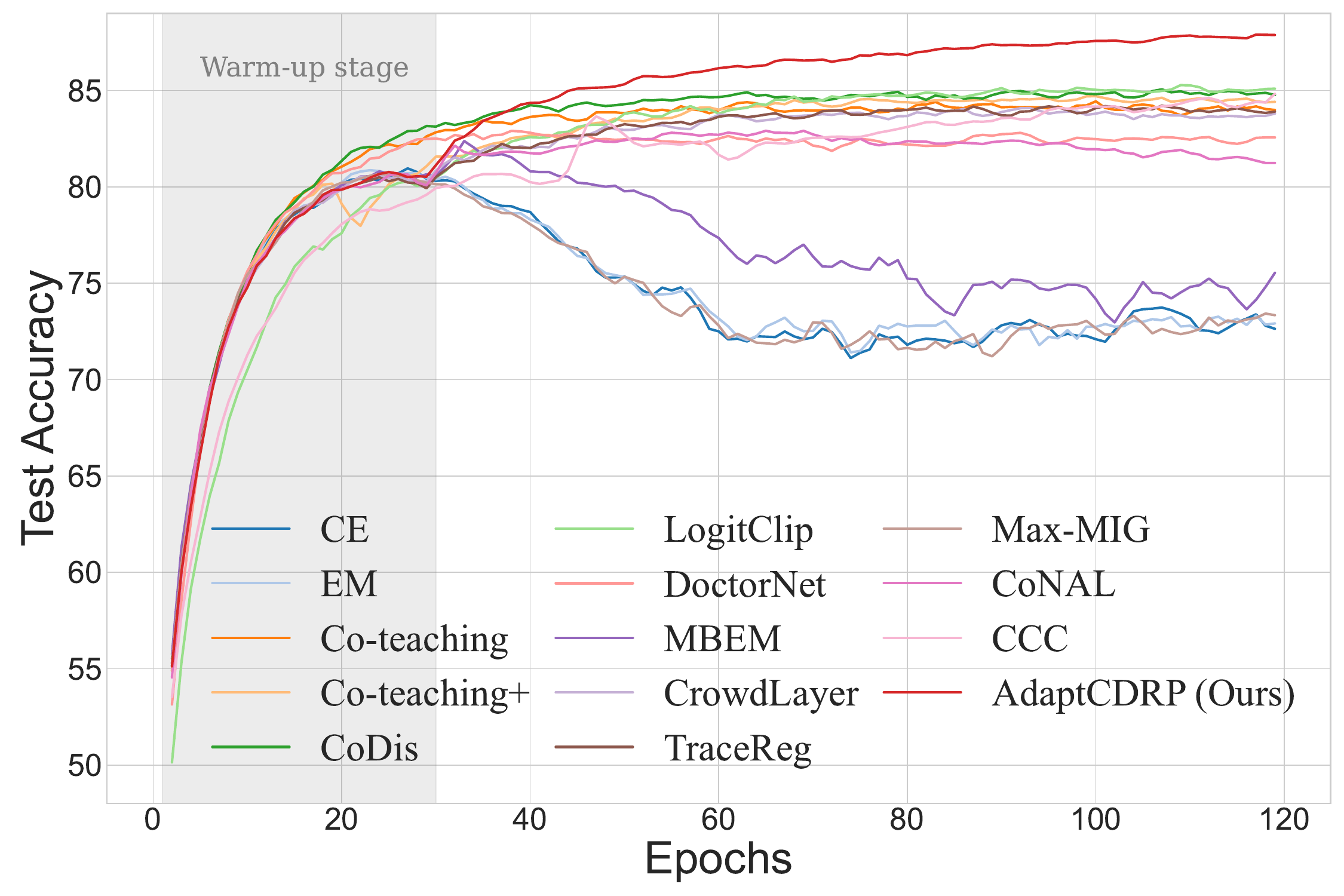}}
\subfigure[IDN-MID]{
\label{Fig.epoch_ACC_cifar10_mid}
\includegraphics[width=0.32\textwidth]{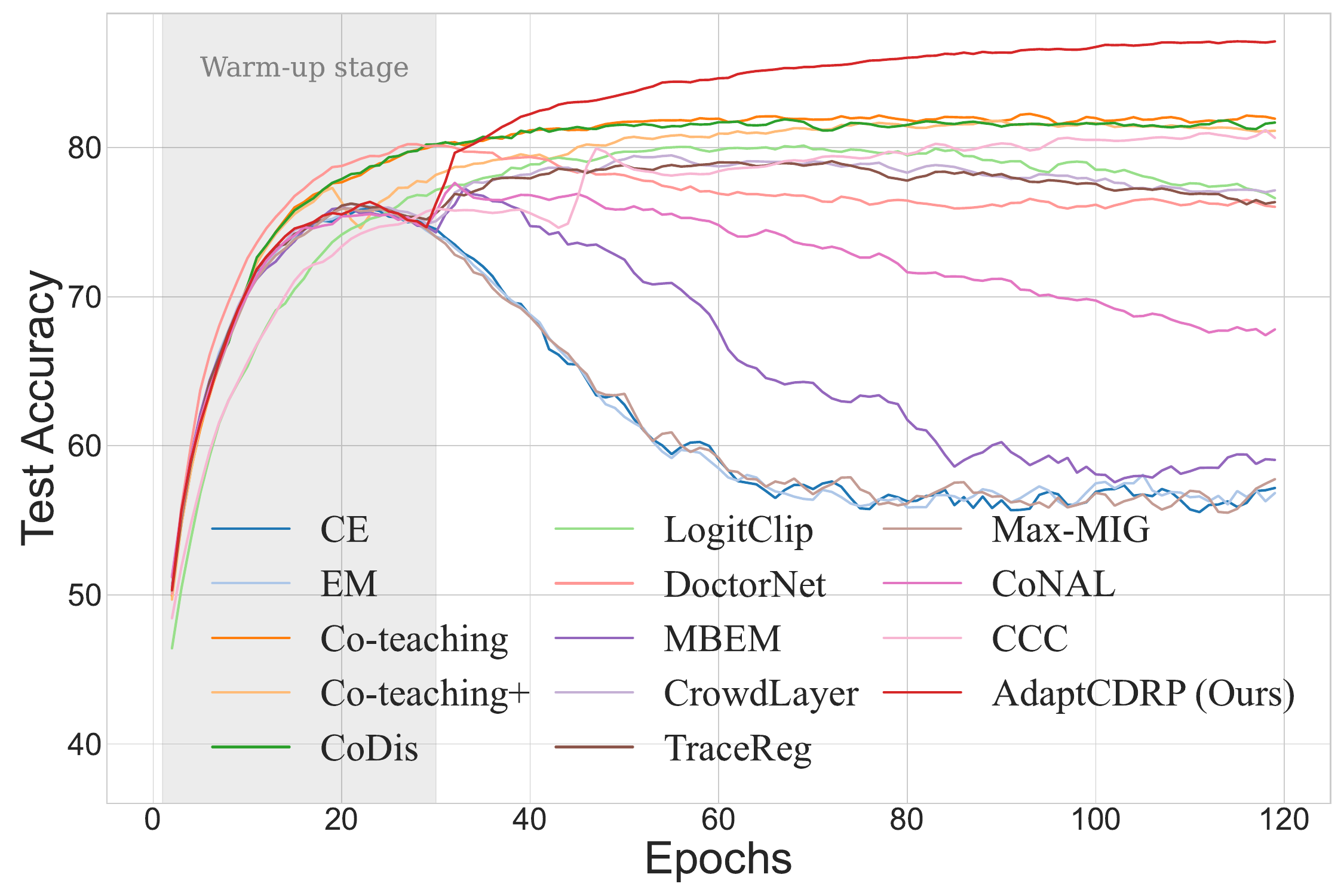}}
\subfigure[IDN-HIGH]{
\label{Fig.epoch_ACC_cifar10_high}
\includegraphics[width=0.32\textwidth]{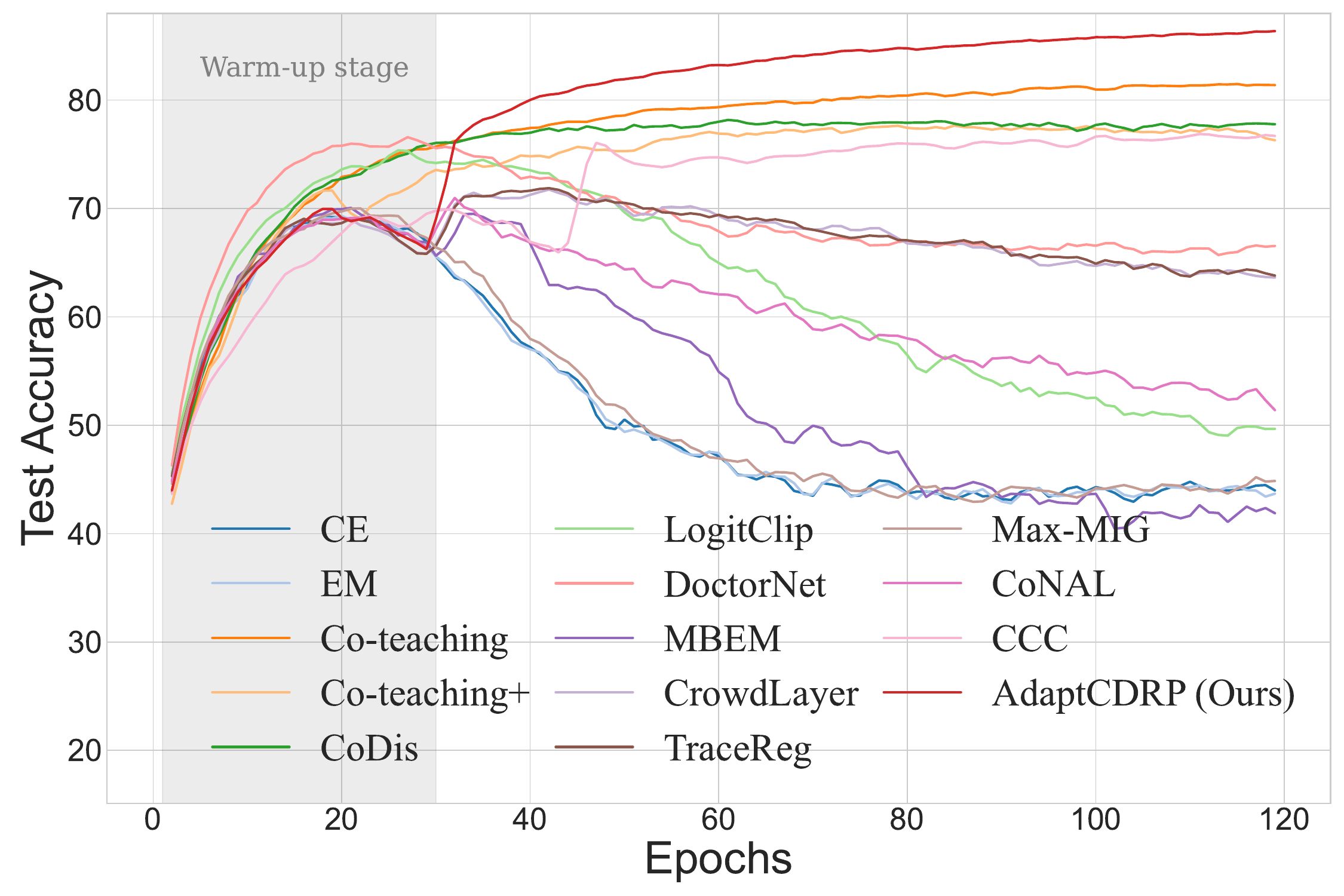}}
\caption{Average test accuracy on learning the CIFAR-10 dataset {($R=5$)} during the training process.}
\label{Fig.epoch_ACC_cifar_10}
\end{figure}

\begin{figure}[H]
\centering 
\subfigure[IDN-LOW]{
\label{Fig.epoch_ACC_cifar100_low}
\includegraphics[width=0.32\textwidth]{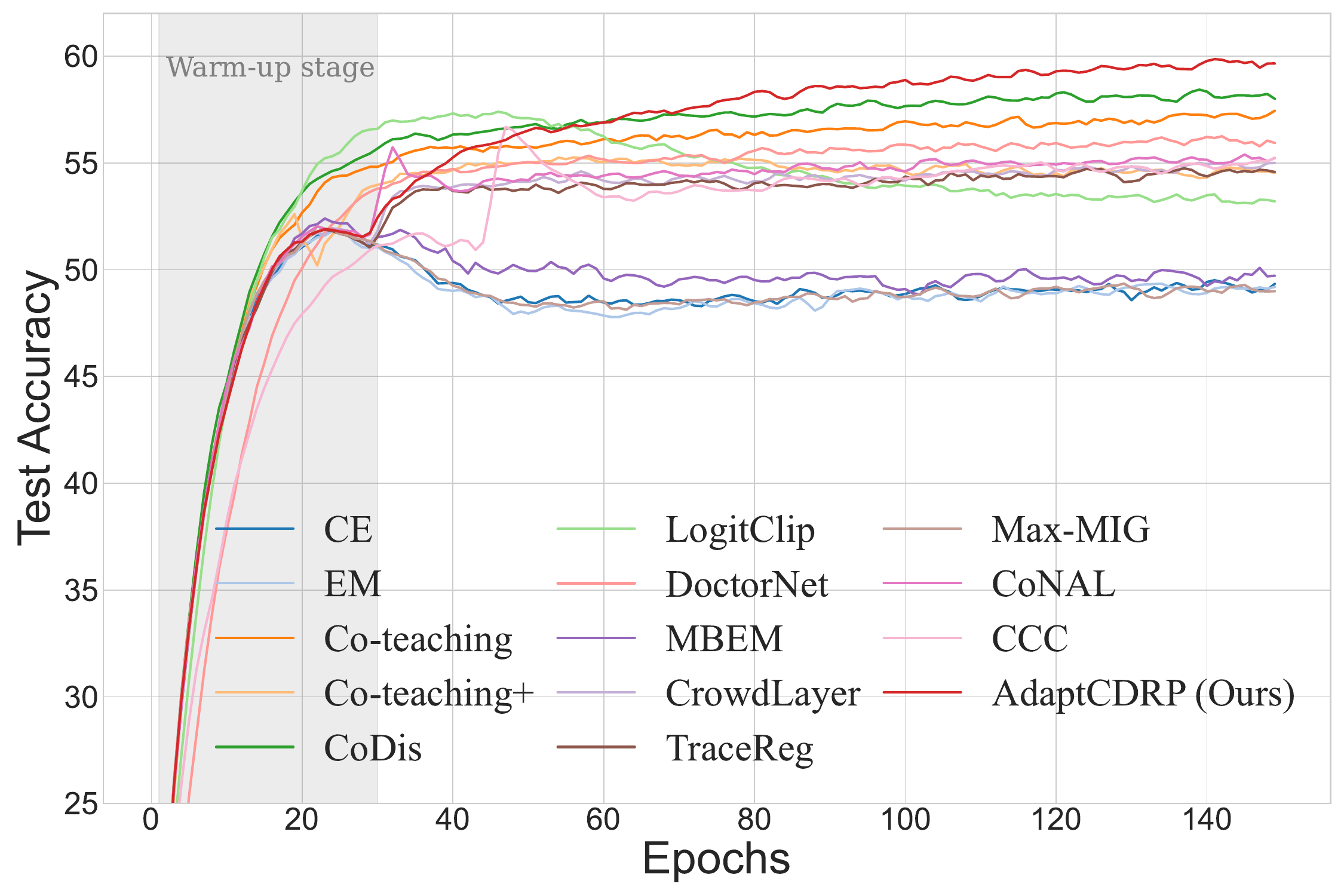}}
\subfigure[IDN-MID]{
\label{Fig.epoch_ACC_cifar100_mid}
\includegraphics[width=0.32\textwidth]{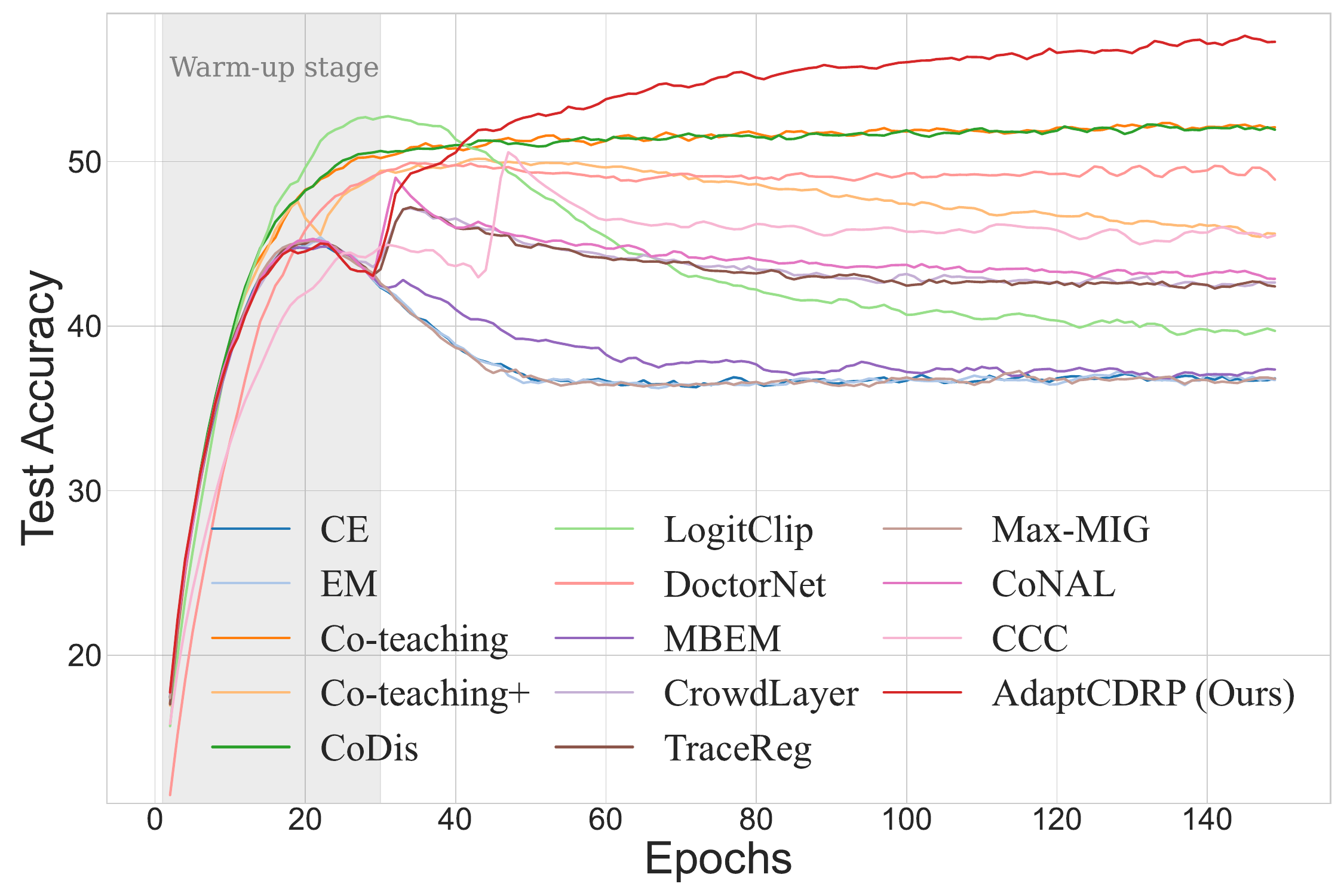}}
\subfigure[IDN-HIGH]{
\label{Fig.epoch_ACC_cifar100_high}
\includegraphics[width=0.32\textwidth]{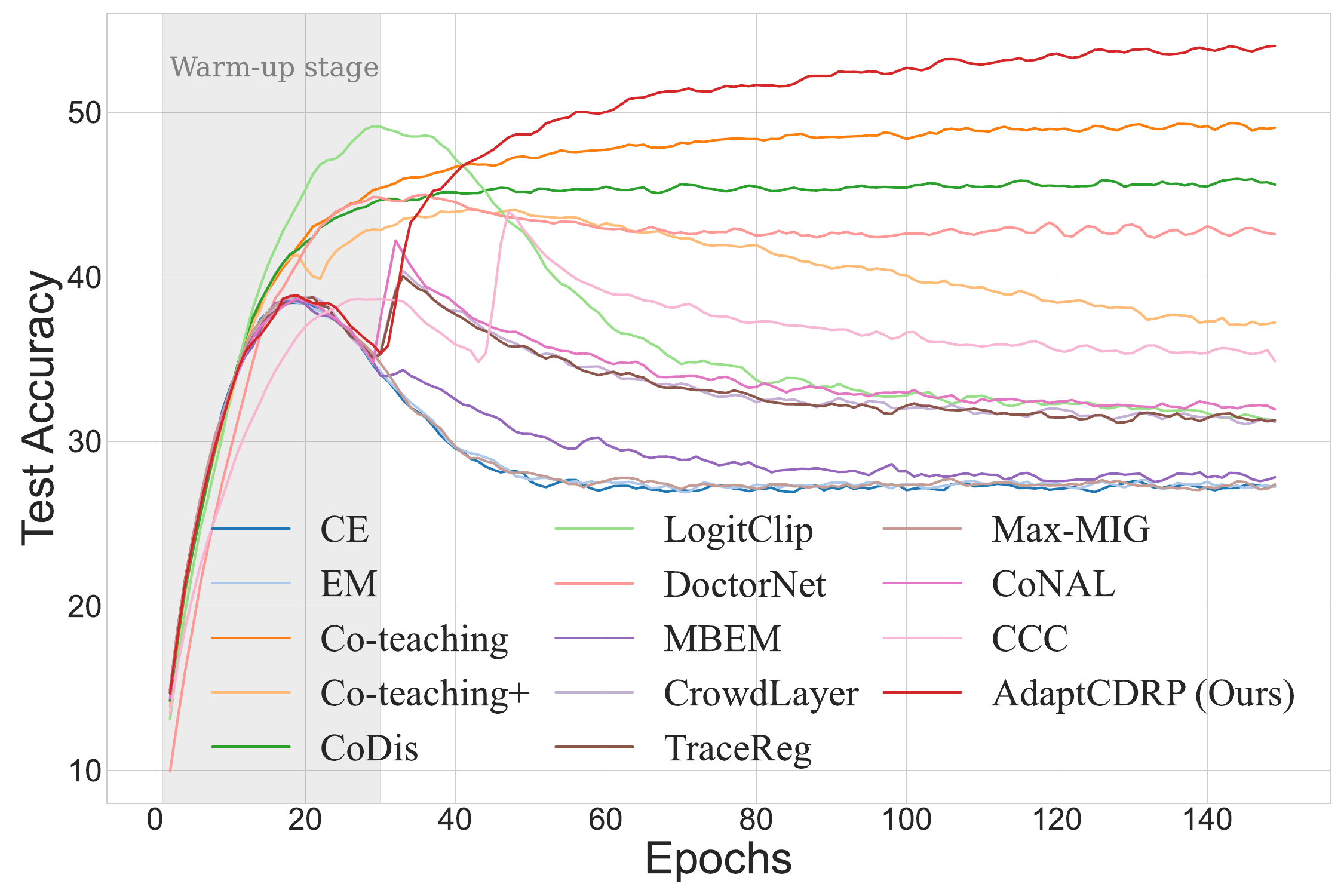}}
\caption{Average test accuracy on learning the CIFAR-100 dataset {($R=5$)} during the training process.}
\label{Fig.epoch_ACC_cifar_100}
\end{figure}

\paragraph{Impact of the number of warm-up epochs.} Following \cite{guo2023label, zheng2020error}, we use 30 warm-up epochs for the CIFAR-10 and CIFAR-100 datasets in our experiments. To rigorously assess the impact of the warm-up stage, we conduct additional experiments with varying numbers of warm-up epochs $(10, 20, 30, 40)$ on both our method and baseline approaches that also incorporate warm-up. The results, presented in Figure \ref{Fig.ACC_warmup}, illustrate how different warm-up durations affect performance.

\begin{figure}[H]
\centering  
\subfigure[CIFAR-10]{
\includegraphics[width=0.8\textwidth]{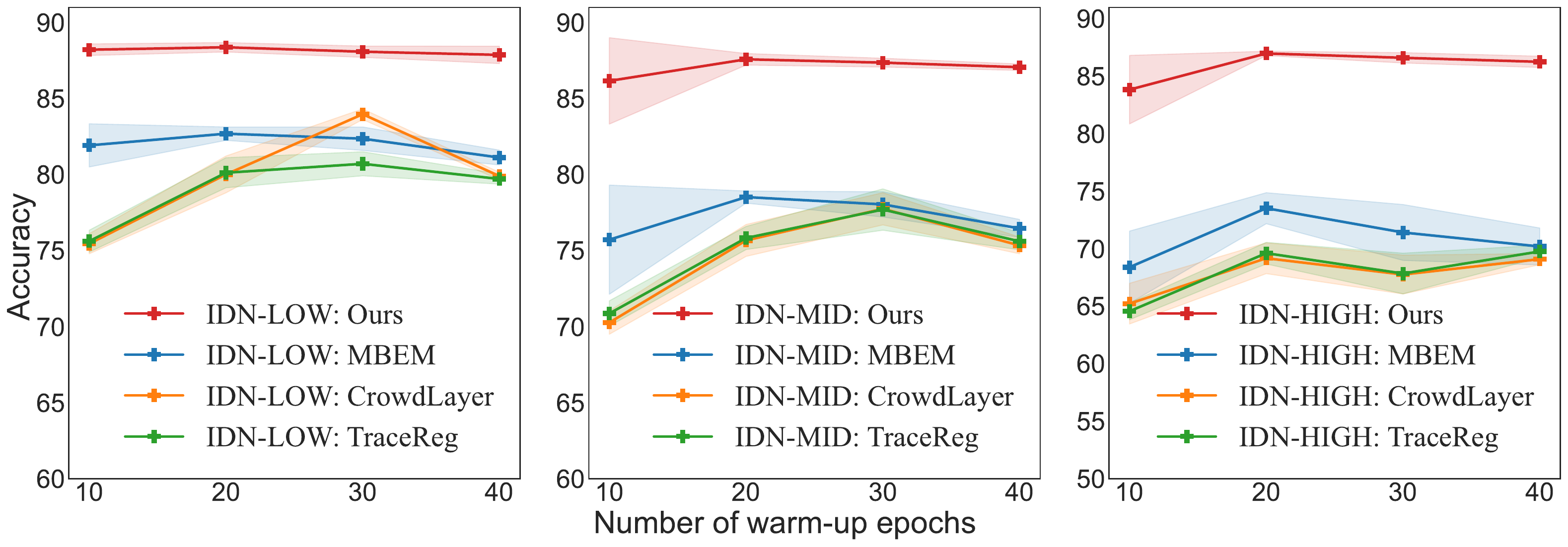}}
\subfigure[CIFAR-100]{
\includegraphics[width=0.8\textwidth]{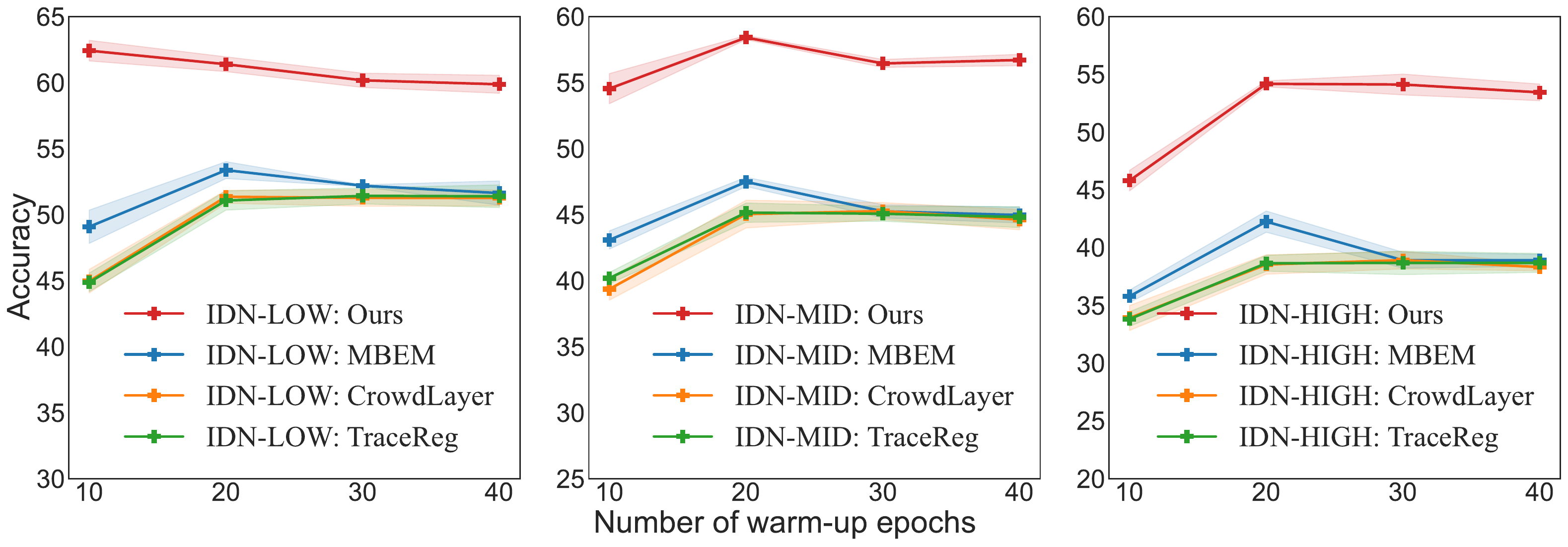}}
\caption{Average test accuracies for learning the CIFAR-10 and CIFAR-100 datasets with varying numbers of warm-up epochs. The error bars representing standard deviation are shaded.}
\label{Fig.ACC_warmup}
\end{figure}

\paragraph{Different transition matrix estimation methods.} Our work does not focus on precise estimation of the noise transition matrix; instead, we use a simple frequency-counting method for noise transition estimation in our experiments. Nevertheless, our approach is versatile and can be integrated with various methods for estimating the noise transition matrix or the true label posterior. Additional experiments using advanced transition matrix estimation methods are presented in Table \ref{table_transition}. As demonstrated, integrating these methods with AdaptCDRP significantly improves test accuracies compared to directly using the estimated noise transition matrices. Furthermore, applying advanced noise transition estimation methods enhances the performance of our method on real datasets. These results highlight the robustness and adaptability of our method.

\begin{table}[H] \footnotesize
  \renewcommand\arraystretch{1.0}
  \setlength\tabcolsep{1.0pt}
  \caption{Average test accuracies (with associated standard errors expressed after the $\pm$ signs) of learning the CIFAR-10 and real datasets with different transition matrix estimation methods.}
  \label{table_transition}
  \centering
  {
  \begin{tabular}{cccccccc}
    \toprule 
    \multirow{2}{*}{Method} & \multicolumn{3}{c}{CIFAR-10} & \multicolumn{4}{c}{Real datasets}\\
    \cmidrule(lr){2-4} \cmidrule(lr){5-8}
    & IDN-LOW & IDN-MID & IDN-HIGH & CIFAR-10N & CIFAR-100N & Animal10N & LabelMe\\
    \midrule
    TraceReg \cite{tanno2019learning} & $80.72_{\pm 0.79}$ & $77.71_{\pm 1.36}$ & $67.86_{\pm 1.77}$ & $82.94_{\pm 0.27}$ & $47.71_{\pm 0.70}$ & $80.34_{\pm 0.66}$ & $83.10_{\pm 0.15}$\\
    TraceReg+Ours & $87.74_{\pm 0.26}$ & $86.76_{\pm 0.07}$ & $85.83_{\pm 0.37}$ & $88.38_{\pm 0.35}$ & $52.16_{\pm 1.02}$ & $83.05_{\pm 0.26}$ & $83.80_{\pm 0.44}$ \\
    \cmidrule(lr){1-8}
    GeoCrowdNet (F) \cite{ibrahimdeep} & $84.73_{\pm 0.39}$ & $81.44_{\pm 1.00 }$ & $77.29_{\pm 1.23}$ & $87.49_{\pm 0.45}$ & $47.74_{\pm 1.17}$ & $81.07_{\pm 0.45}$ & $84.59_{\pm 0.19}$ \\
    GeoCrowdNet (F) + Ours & $88.06_{\pm 0.33}$ & $87.43_{\pm 0.29}$ & $86.69_{\pm 0.13}$ & $88.30_{\pm 0.13}$ & $51.07_{\pm 0.57}$ & $83.12_{\pm 0.42}$ & $86.20_{\pm 0.48}$ \\
    \cmidrule(lr){1-8}
    GeoCrowdNet (W) \cite{ibrahimdeep} & $83.82_{\pm 0.53}$ & $75.72_{\pm 1.10}$ & $64.64_{\pm 2.23}$ & $87.36_{\pm 0.24}$ & $47.49_{\pm 0.91}$ & $80.19_{\pm 0.33}$ &  $81.63_{\pm 1.49}$\\
    GeoCrowdNet (W) + Ours & $87.94_{\pm 0.35}$ & $87.21_{\pm 0.33}$ & $83.48_{\pm 5.69}$ & $87.81_{\pm 0.12}$ & $52.03_{\pm 0.45}$ & $82.41_{\pm 0.04}$ & $83.32_{\pm 0.51}$\\
    \cmidrule(lr){1-8}
    BayesianIDNT \cite{guo2023label} & $86.46_{\pm 1.07}$ & $85.14_{\pm 0.96}$ & $82.49_{\pm 2.86}$  & $87.83_{\pm 0.53}$ & $51.06_{\pm 0.72}$ & $81.22_{\pm 0.59}$ &  $83.01_{\pm 0.32}$ \\
    BayesianIDNT + Ours & $87.66_{\pm 0.85}$ & $86.44_{\pm 0.57}$ & $84.38_{\pm 0.10}$ & $88.31_{\pm 0.20}$ & $53.13_{\pm 0.74}$ & $83.80_{\pm 0.44}$ & $84.09_{\pm 0.53}$\\
    \bottomrule
  \end{tabular}}
\end{table}

\paragraph{Impact of sparse annotation.} To further address the issue of annotation sparsity, we increase the total number of annotators, $R$, to 200, and manually corrupt the datasets according to the
following annotator groups:
\begin{align*}
    &\textbf{R=200:}\\
    &\textbf{IDN-LOW. }\textit{70 IDN-10\% annotators, 70 IDN-20\% annotators, 60 IDN-30\% annotators;}\\
    &\textbf{IDN-MID. }\textit{70 IDN-30\% annotators, 70 IDN-40\% annotators, 60 IDN-50\% annotators;}\\
    &\textbf{IDN-HIGH. }\textit{70 IDN-50\% annotators, 70 IDN-60\% annotators, 60 IDN-70\% annotators.}
\end{align*}
The three groups of annotators, labeled as IDN-LOW, IDN-MID, and IDN-HIGH, have average labeling error rates of approximately 26\%, 34\%, and 42\%, respectively. In this setup, we incorporate regularization techniques - specifically, GeoCrowdNet (F) and GeoCrowdNet (W) penalties \cite{ibrahimdeep} - into our method. We then compare the results against those obtained using the traditional frequency-counting approach for estimating the noise transition matrices. Table \ref{table_R200} presents the performance of our proposed method on the CIFAR10 ($R=200$) dataset, where different
approaches are used to estimate the noise transition matrices. In addition, Figure \ref{Fig.PACC_cifar10_200} displays the average accuracies of the robust pseudo-labels generated by our method during the training process. These pseudo-labels play a crucial role in constructing the pseudo-empirical distribution.

\begin{table}[H] \footnotesize
  \renewcommand\arraystretch{1.0}
  \setlength\tabcolsep{1.0pt}
  \caption{Average test accuracies (with associated standard errors expressed after the $\pm$ signs) of learning the CIFAR-10 dataset ($R=200$) with different transition matrix estimation methods.}
  \label{table_R200}
  \centering
  {
  \begin{tabular}{cccc}
    \toprule 
    Mthod & IDN-LOW & IDN-MID & IDN-HIGH \\
    \midrule
    Ours + frequency-counting & $86.01_{\pm 0.67}$ & $85.48_{\pm 0.58}$ & $85.07_{\pm 0.59}$ \\
    Ours + GeoCrowdNet (F) penalty \cite{ibrahimdeep} & $90.89_{\pm 0.21}$ & $90.27_{\pm 0.46}$ & $89.25_{\pm 0.63}$ \\
    Ours + GeoCrowdNet (W) penalty \cite{ibrahimdeep} & $90.99_{\pm 0.42}$ & $90.23_{\pm 0.27}$ & $89.42_{\pm 0.29}$  \\
    \bottomrule
  \end{tabular}}
\end{table}

\begin{figure}[H]
\centering 
\subfigure[Ours + Frequency Counting]{
\label{Fig.PACC_cifar10_200_sub.1}
\includegraphics[width=0.323\textwidth]{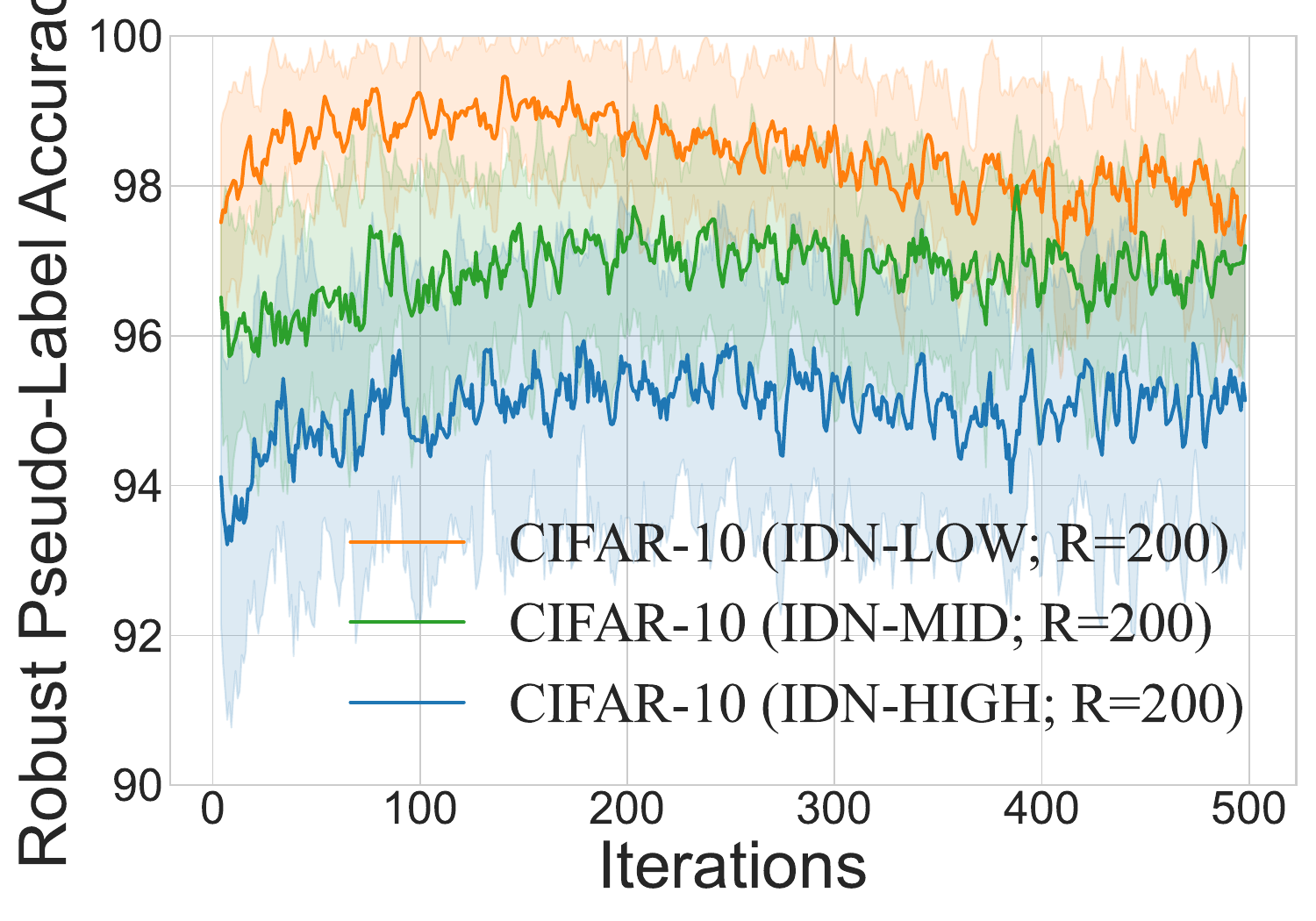}}
\subfigure[Ours + GeoCrowdNet (F)]{
\label{Fig.PACC_cifar10_200_sub.2}
\includegraphics[width=0.323\textwidth]{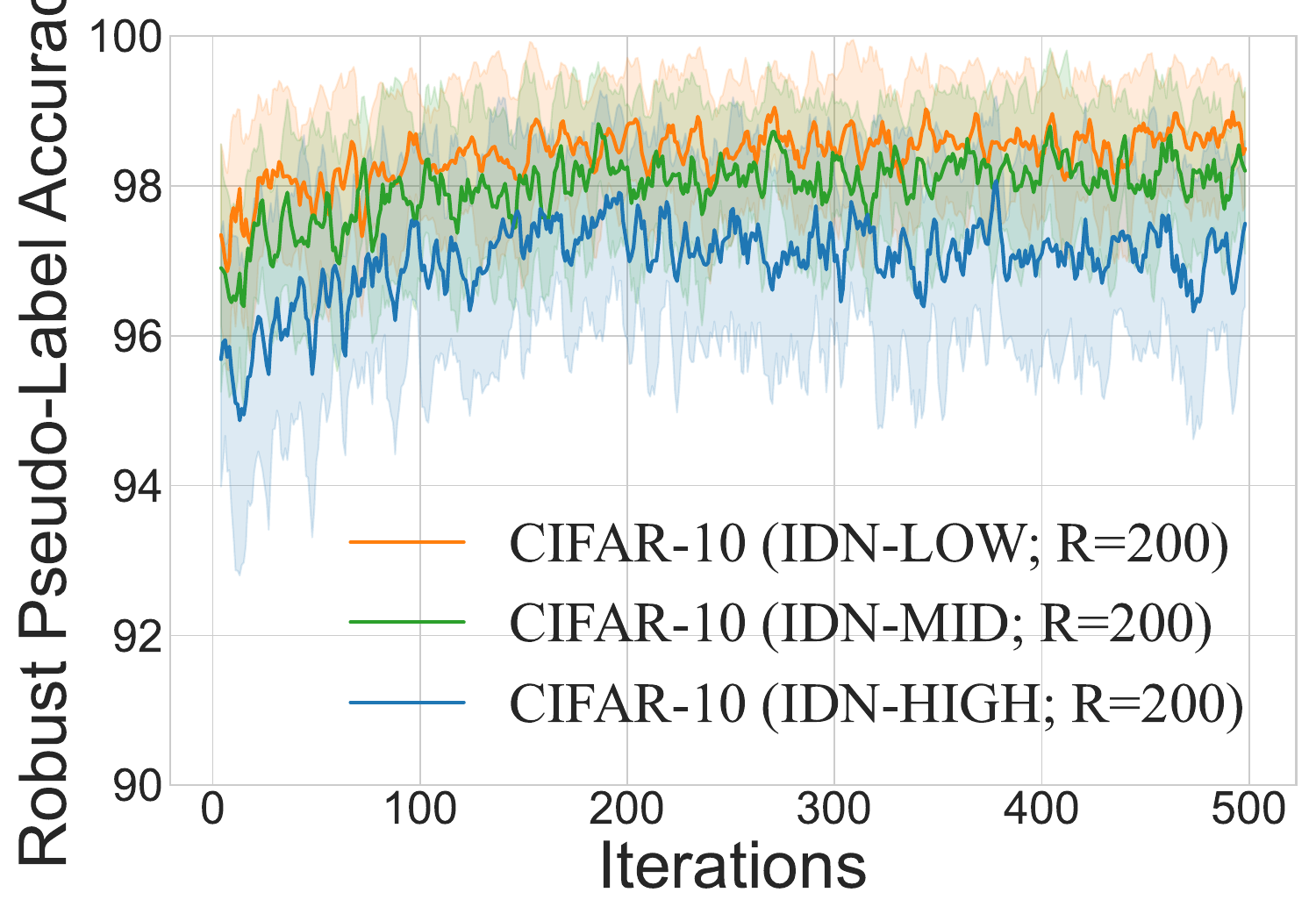}}
\subfigure[Ours + GeoCrowdNet (W)]{
\label{Fig.PACC_cifar10_200_sub.3}
\includegraphics[width=0.323\textwidth]{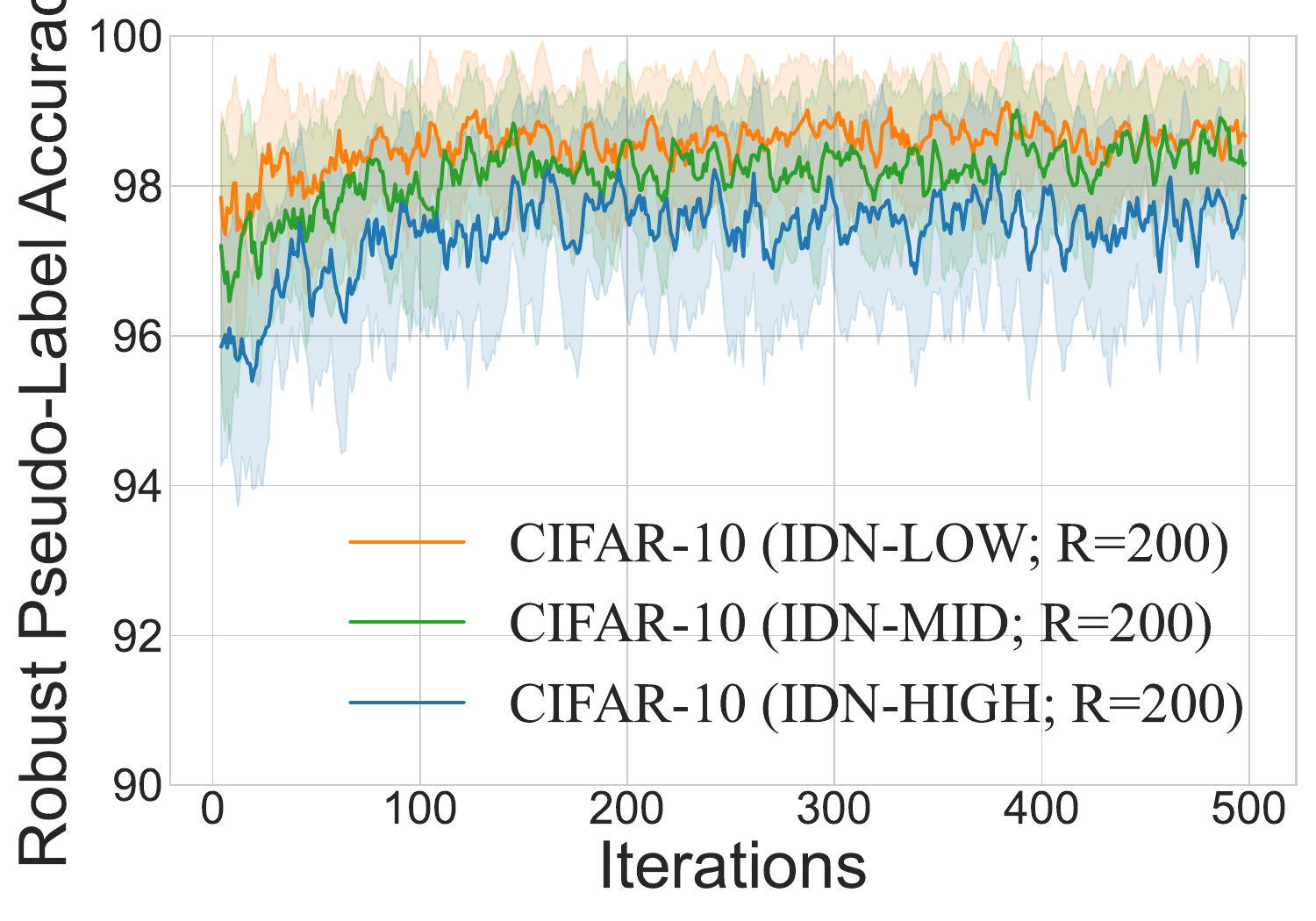}}
\caption{Average accuracy of robust pseudo-labels on the CIFAR-10 dataset ($R=200$) using different transition matrix estimation methods during the training process.}
\label{Fig.PACC_cifar10_200}
\end{figure}

}



\newpage
\section*{NeurIPS Paper Checklist}

\begin{enumerate}

\item {\bf Claims}
    \item[] Question: Do the main claims made in the abstract and introduction accurately reflect the paper's contributions and scope?
    \item[] Answer: \answerYes{} 
    \item[] Justification:  The main claims made in the abstract and introduction accurately reflect our contributions and scope. 

\item {\bf Limitations}
    \item[] Question: Does the paper discuss the limitations of the work performed by the authors?
    \item[] Answer: \answerYes{} 
    \item[] Justification: We discuss the limitations of the work performed by the authors in the conclusion.

\item {\bf Theory Assumptions and Proofs}
    \item[] Question: For each theoretical result, does the paper provide the full set of assumptions and a complete (and correct) proof?
    \item[] Answer: \answerYes{} 
    \item[] Justification: We provide the assumptions in the main text, and the proofs can be found in the appendix.

    \item {\bf Experimental Result Reproducibility}
    \item[] Question: Does the paper fully disclose all the information needed to reproduce the main experimental results of the paper to the extent that it affects the main claims and/or conclusions of the paper (regardless of whether the code and data are provided or not)?
    \item[] Answer: \answerYes{} 
    \item[] Justification: We provide all the experimental details.

\item {\bf Open access to data and code}
    \item[] Question: Does the paper provide open access to the data and code, with sufficient instructions to faithfully reproduce the main experimental results, as described in supplemental material?
    \item[] Answer: \answerYes{} 
    \item[] Justification: We upload our code and use public datasets.

\item {\bf Experimental Setting/Details}
    \item[] Question: Does the paper specify all the training and test details (e.g., data splits, hyperparameters, how they were chosen, type of optimizer, etc.) necessary to understand the results?
    \item[] Answer: \answerYes{} 
    \item[] Justification: We provide all the details and the code.

\item {\bf Experiment Statistical Significance}
    \item[] Question: Does the paper report error bars suitably and correctly defined or other appropriate information about the statistical significance of the experiments?
    \item[] Answer: \answerYes{} 
    \item[] Justification: We provide the error bars.

\item {\bf Experiments Compute Resources}
    \item[] Question: For each experiment, does the paper provide sufficient information on the computer resources (type of compute workers, memory, time of execution) needed to reproduce the experiments?
    \item[] Answer: \answerYes{} 
    \item[] Justification: We provide the training time on different datasets.
    
\item {\bf Code Of Ethics}
    \item[] Question: Does the research conducted in the paper conform, in every respect, with the NeurIPS Code of Ethics \url{https://neurips.cc/public/EthicsGuidelines}?
    \item[] Answer: \answerYes{} 
    \item[] Justification: The research conducted in the paper conform, in every respect, with the NeurIPS Code of Ethics.

\item {\bf Broader Impacts}
    \item[] Question: Does the paper discuss both potential positive societal impacts and negative societal impacts of the work performed?
    \item[] Answer: \answerNA{} 
    \item[] Justification: It's not appropriate for the scope and focus of our paper, and we don't see any direct negative social impacts of our paper.
    
\item {\bf Safeguards}
    \item[] Question: Does the paper describe safeguards that have been put in place for responsible release of data or models that have a high risk for misuse (e.g., pretrained language models, image generators, or scraped datasets)?
    \item[] Answer: \answerNA{} 
    \item[] Justification: Our model doesn't have a high risk for misuse or dual-use.

\item {\bf Licenses for existing assets}
    \item[] Question: Are the creators or original owners of assets (e.g., code, data, models), used in the paper, properly credited and are the license and terms of use explicitly mentioned and properly respected?
    \item[] Answer: \answerYes{} 
    \item[] Justification: We cite the original papers.

\item {\bf New Assets}
    \item[] Question: Are new assets introduced in the paper well documented and is the documentation provided alongside the assets?
    \item[] Answer: \answerYes{} 
    \item[] Justification: We provide the data generation details.

\item {\bf Crowdsourcing and Research with Human Subjects}
    \item[] Question: For crowdsourcing experiments and research with human subjects, does the paper include the full text of instructions given to participants and screenshots, if applicable, as well as details about compensation (if any)? 
    \item[] Answer: \answerNA{} 
    \item[] Justification: We do not involve human subjects.

\item {\bf Institutional Review Board (IRB) Approvals or Equivalent for Research with Human Subjects}
    \item[] Question: Does the paper describe potential risks incurred by study participants, whether such risks were disclosed to the subjects, and whether Institutional Review Board (IRB) approvals (or an equivalent approval/review based on the requirements of your country or institution) were obtained?
    \item[] Answer: \answerNA{} 
    \item[] Justification: We do not involve human subjects.

\end{enumerate}

\end{document}